\documentclass[10pt,journal,compsoc]{IEEEtran}

\usepackage{graphicx}
\usepackage{booktabs} 
\usepackage{array}

\usepackage{booktabs} 
\usepackage{bbold}
\graphicspath{{./images/}}

\usepackage{bm}
\usepackage{tikz}
\usetikzlibrary{calc,shapes,positioning}
\usepackage{enumitem}

\usepackage{hyperref}

\usepackage{mystyle}

\usepackage{amsthm}

\newtheorem{myprop}{Proposition}
\ifCLASSOPTIONcompsoc
\usepackage[nocompress]{cite}
\else
\usepackage{cite}
\fi

\allowdisplaybreaks

\newcommand{\ubar}[1]{\mkern2mu\underline{\mkern-2mu #1\mkern-2mu}\mkern2mu}

\begin{document}

\title{Neural Network based Explicit Mixture Models and Expectation-maximization based Learning}

\author{Dong~Liu,
  Minh~Th{\`a}nh~Vu,
  Saikat~Chatterjee,
  and~Lars~K.~Rasmussen
  \IEEEcompsocitemizethanks{\IEEEcompsocthanksitem The authors are with School of Electrical Engineering and Computer Science, KTH Royal Institute of Technology, Stockholm, Sweden. 
    E-mail: \{doli, mtvu, sach, lkra\}@kth.se}
}

\IEEEtitleabstractindextext{%

  \begin{abstract}
  We propose two neural network based mixture models in this article. The proposed mixture models are explicit in nature. The explicit models have analytical forms with the advantages of computing likelihood and efficiency of generating samples. Computation of likelihood is an important aspect of our models. Expectation-maximization based algorithms are developed for learning parameters of the proposed models. We provide sufficient conditions to realize the expectation-maximization based learning. The main requirements are invertibility of neural networks that are used as generators and Jacobian computation of functional form of the neural networks. The requirements are practically realized using a flow-based neural network. In our first mixture model, we use multiple flow-based neural networks as generators. Naturally the model is complex. A single latent variable is used as the common input to all the neural networks. The second mixture model uses a single flow-based neural network as a generator to reduce complexity. The single generator has a latent variable input that follows a Gaussian mixture distribution. We demonstrate efficiency of proposed mixture models through extensive experiments for generating samples and maximum likelihood based classification.
  
       
  \end{abstract}
}

\maketitle

\section{Introduction}

The paradigm of neural network based implicit distribution
modeling has received a significant attention. In this paradigm, a neural network being a powerful non-linear function acts as an efficient generator. Prominent examples of neural network based implicit distributions are generative adversarial networks (GANs) \cite{NIPS2014_5423} and its
variants \cite{NIPS2016_6125,
  2018arXiv180508318Z, salimans2018improving}. GANs are efficient for generating samples and successful in several applications \cite{ledig2017photo}, \cite{NIPS2016_6125}. For a GAN, a latent variable is used as
an input to the generator neural network of the GAN and the output of the
neural network is considered to be a data sample from the implicit
distribution.  In implicit distribution modeling by GANs, neither
analytical form of the distribution nor likelihood for a data sample
is available. Naturally the use of neural network based implicit distribution models like GANs is restricted to many applications where it is important to compute likelihood, for example, maximum likelihood based classification.


In this article, we focus on neural network based explicit distribution modeling. An explicit distribution model has an analytical functional form and we are able to compute likelihood. While use of neural network based generators in GANs for distribution modeling is powerful, we look for further improvements. In this regard, a standard practice is to use mixture models assuming that the underlying distribution is multi-modal, or data are spread over multiple manifolds and subspaces. Therefore we propose to design neural network based explicit mixture models such that they
\begin{enumerate}[noitemsep, nolistsep]
    \item[(a)] have analytical forms, 
    \item[(b)] offer the advantage of computing likelihood, and 
    \item[(c)] retain the advantage of generating samples efficiently.
\end{enumerate}
With the advantage of computing likelihood, our proposed neural network based mixture models are suitable for maximum likelihood based classification. 

An important question is how to design practical algorithms to learn parameters of the proposed mixture models. In literature, expectation-maximization (EM) \cite{dempster1977maximum} is a standard approach for learning parameters of an explicit mixture model in a maximum likelihood framework, such as learning parameters of a Gaussian mixture model (GMM) \cite{Bishop:2006:PRM:1162264}. For realizing EM, computation of the
posterior distribution of a hidden variable (related to identity of a mixture component) given the observation
(visible signal/data) is required in the expectation step
(E-step). In addition, it is required to compute the joint log-likelihood of the observation signal and the hidden variable in the maximization step (M-step). For example, EM for GMM can be realized due to fullfilment of the above two requirements. Typically it is challenging to fulfill these two requirements for many other mixture distribution models. We also face the challenge to realize EM for learning parameters of neural network based explicit mixture models. This is due to the fact that use of neural networks in design of a system/algorithm/method often leads to loss of required level of analytical tractability. 



In pursuit of neural network based explicit mixture models, our contributions in this article are as follows.
\begin{enumerate}[noitemsep, nolistsep]
    \item[(a)] Proposing two mixture models - a high-complexity model and a low-complexity model. The low-complexity model uses shared parameters.
    \item[(b)] Finding theoretical conditions for the models such that EM can be applied for their parameter learning. The theoretical conditions help to find explicit posterior and computation of expected likelihood. 
    \item[(c)] Designing practical algorithms for realization of EM where gradient search based optimization is efficiently embedded into M-step.
    \item[(d)] Demonstrating efficiency of proposed mixture models through extensive experiments for generating samples and maximum likelihood based classification. 
\end{enumerate}
{At this point we mention the conditions of realizing EM to learn neural network based explicit mixture models.} The sufficient conditions are invertibility of associated neural networks and Jacobian computation of functional form of the neural networks. This helps to compute likelihood using change of variables. In practice we address the sufficient conditions using a flow-based neural network \cite{2018arXiv180703039K}.

\subsection{Related Work and Background}

While GANs have high success in many applications, they are known to
suffer in a mode dropping problem where a generator of a GAN is unable
to capture all modes of an underlying probability distribution of data
\cite{2018arXiv180600880K}. To address diversity in data and model
multiple modes in a distribution, variants of generative models have been developed and
usage of multiple generators
has been considered. For instance, methods of minibatch discrimination \cite{NIPS2016_6125} and feature representation \cite{bang2018icml} are used to construct new discriminators of GANs which encourage the GANs to generate samples with diversity. Multiple Wasserstein GANs \cite{2017arXiv170107875A} are used in \cite{2018arXiv180600880K} with appropriate mutual information based regularization to encourage the diversity of samples generated by different GANs.
A mixture GAN approach is proposed in \cite{hoang2018mgan} using multiple generators and multi-classification solution to
encourage diversity of samples. Multi-agent diverse GAN \cite{DBLP:journals/corr/GhoshKNTD17} similarly employs $k$ generators, but uses a $(k+1)$-class discriminator instead of a typical binary discriminator to increase the diversity of generated samples. These works are implicit probability distribution modeling and thus prior distribution of generators can not be inferred when multiple generators are used.

Typically, for a GAN, the latent variable is assumed to follow a known and fixed distribution, e.g., Gaussian. The latent signal for a given data sample can not be obtained since generators which are usually based on neural networks are non-invertible. The mapping from a data sample to its corresponding latent signal is approximately estimated by neural networks in different ways. \cite{donahue2017adversarial} and \cite{dumoulin2017adversarially} propose to train a
generative model and an inverse mapping (also neural network) from the data sample to the latent signal
simultaneously, using the adversarial training method of
GAN. Alternatively, \cite{dustin2017hierarchical}
proposes to approximately minimize a Kullback-Leibler divergence to
estimate the mapping from the data sample to the latent variable, which leads to a nontrivial probability density ratio estimation problem.

Another track of
mixture modeling is based on ensembling method that combines weaker
learners together to boost the overall performance \cite{grover2017aaai_boost,
  2017arXiv170102386T}. In this approach mixture models are obtained as follows. Based on how well the current-step mixture model captures the underlying distribution, a new generative model is
trained to compensate the miss-captured part. However, measuring the difference between current-step mixture model and underlying distribution
of dataset quantitatively is a nontrivial task. In addition, since incremental building components are used in the mixture modeling, parallel training of model components is not feasible.

\begin{figure}
  \begin{tikzpicture}
    \tikzstyle{enode} = [thick, draw=blue, ellipse, inner sep = 2pt,  align=center]
    \tikzstyle{nnode} = [thick, rectangle, rounded corners = 2pt,minimum size = 0.8cm,draw,inner sep = 2pt]
    \node[enode] (z) at (0,0) {$\bm{z}\sim p(\bm{z})$};
    \node[enode] (x) at (5.5,0){$\bm{x}\sim p(\bm{x}; \bm{\Phi})$};
    \node[nnode] (g1) at (2.6,1.8) {$\bm{g}_1$};
    \node[nnode] (g2) at (2.6,0.5) {$\bm{g}_2$};
    \node[nnode] (gk) at (2.6,-1.8) {$\bm{g}_K$};
    \draw[dotted,line width=2pt] (2.6,-0.3) -- (2.6,-1.2);
    \draw[->] (z) [in= 180, out =0] to (g1);
    \draw[->] (z) [in= 180, out =0] to (g2);
    \draw[->] (z) [in= 180, out =0] to (gk);
    \filldraw[->] (3.7, 0.5)circle (2pt) -- node[above=0.2]{$\bm{s}\sim \bm{\pi}$} (x) ;
    \draw[->] (g1) -- (3.5,1.8);
    \draw[->] (g2) -- (3.5, 0.5);
    \draw[->] (gk) -- (3.5, -1.8);
  \end{tikzpicture}
  \caption{Diagram of Generator Mixture Model (GenMM).}\label{dia-emgm-nm}
  \vspace{0.1cm}
\end{figure}
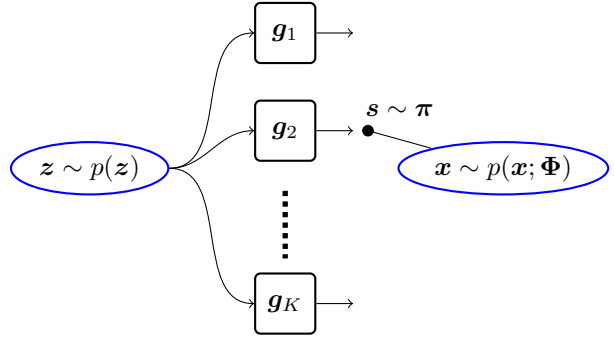

\section{Generator mixture model and EM}

In this proposed generative model, we have $K$ separate neural networks. All the $K$ neural networks have a common input latent variable $\bm{z} \in \mathbb{R}^M$. Here, a neural network $\bm{g}_k(\bm{z}): \mathbb{R}^M
\rightarrow \mathbb{R}^N$ acts as the $k$-th generator and depends on a set of parameters $\bm{\theta}_k$ as
$\bm{g}_k(\bm{z})=\bm{g}(\bm{z};\boldsymbol{\theta}_k)$. For
simplicity, we assume that all $K$ neural networks have the same
signal-flow structure. Furthermore, the distribution of $\bm{z}$ is fixed as Gaussian $\mathcal{N}(\bm{0},\bm{I})$.
The induced probability density function (pdf) of $\bm{x} \in \mathbb{R}^N$ of the proposed mixture model with $K$ mixture components is given as:
\begin{align}\label{eq:FirstMixtureModel}
  p(\bm{x};\bm{\Phi})  &= \textstyle\sum_{k=1}^K \pi_k  p_k(\bm{x}) \nonumber\\
                       &= \textstyle \sum_{k=1}^K \pi_k  p(\bm{g}_k(\bm{z}))\nonumber\\
                       &= \textstyle \sum_{k=1}^K \pi_k  p(\bm{g}(\bm{z};\boldsymbol{\theta}_k)).
\end{align}
Their parameters $\boldsymbol{\theta}_k$, however,
are different. We use $\bm{\Phi}$ to denote the set of all parameters $ \{\bm{\pi},\bm{\theta}_1, \dots, \bm{\theta}_K \}$, where $\bm{\pi} = \left[\pi_1, \hdots, \pi_K\right]^{T}$is the prior distribution of the generators. Note that $\pi_k \geq 0$ and $\sum_{k=1}^K \pi_k =1$. The mixture model
in \eqref{eq:FirstMixtureModel} is called a generator mixture model (GenMM). The
diagram of GenMM is illustrated in \autoref{dia-emgm-nm}. The GenMM can be considered as a high-complexity model because each mixture component $p_k(\bm{x})$ has its own parameter set $\bm{\theta}_k$. 

The maximum likelihood estimation problem is
\begin{equation}\label{eq:max-genmm}
  \hat{\bm{\Phi}} =    \arg \max_{\bm{\Phi}} \log \textstyle\prod_{i} p(\bm{x}^{(i)};\bm{\Phi}),
\end{equation}
where the superscript $(i)$ corresponds to the $i$'th data sample in a given dataset.
We address the above maximum likelihood estimation problem using EM. 
Let us use a categorical variable $\bm{s} = [s_1, s_2, \cdots, s_K]$
for $1$-of-$K$ representation to be a hidden variable that indicates which generator is the actual one. Elements of $\bm{s}$ follow $s_k \in \{0,1\}$, $\sum_{k=1}^K s_k =1$, and $\PP\{s_k=1\}=\pi_k$. The variable $\bm{s}$ is the hidden variable in EM. We will use $\gamma_k$ to denote the posterior probability $\PP\{s_k =1|\bm{x}\}$ calculated as
\begin{align}\label{eq-genmm-gamma}
  \gamma_k = \PP\{s_k =1|\bm{x};\bm{\Phi}\}  
  = \frac{\pi_k p(\bm{g}(\bm{z};\bm{\theta}_k))}{\sum_{l=1}^K\; \pi_l p(\bm{g}(\bm{z};\bm{\theta}_l))}.
\end{align}
The posterior probability $\gamma_k$ is also known as responsibility in the EM literature.
Assume that a value $\bm{\Phi}^{\mathrm{old}}$ of the parameter set $\bm{\Phi}$ is given, the iterative steps in EM algorithm update $\bm{\Phi}$ as follows.
\begin{enumerate}
\item E-step: Evaluation of $\gamma_{k}^{(i)}$ is 
  \begin{equation}\label{eq-genmm-e-step}
    \gamma_{k}^{(i)}(\bm{\Phi}^{\mathrm{old}}) = \frac{\pi_k^\mathrm{{old}} p(\bm{g}(\bm{z}^{(i)};\bm{\theta}_k^{\mathrm{old}}))}{\sum_{l=1}^K\; \pi_l^\mathrm{{old}} p(\bm{g}(\bm{z}^{(i)};\bm{\theta}_l^{\mathrm{old}}))}.   
  \end{equation}
\item M-step: Evaluation of $\bm{\Phi}^{\mathrm{new}}$ given by
  \begin{equation}\label{eq-genmm-opt}
    \bm{\Phi}^{\mathrm{new}} =   \arg \max_{\bm{\Phi}} \mathcal{Q} (\bm{\Phi},\bm{\Phi}^{\mathrm{old}}),  
  \end{equation}
  where the expected likelihood is
  \begin{equation}
    \hspace{-8pt}\mathcal{Q} (\bm{\Phi},\bm{\Phi}^{\mathrm{old}}) = \sum_{i}\sum_{k} \gamma_{k}^{(i)}(\bm{\Phi}^{\mathrm{old}}) \log \pi_k p_k(\bm{x}^{(i)}).
  \end{equation}
\end{enumerate}

For the GenMM in \eqref{eq:FirstMixtureModel}, the main technical challenges in realizing EM are computation of $\gamma_k$ in the E-step and computation of the joint likelihood $\log{\pi_k p_k(\bm{x})}$ in the M-step. They require explicit computation of the conditional density $p_k(\bm{x}) =  p(\bm{g}_k(\bm{z})) =p(\bm{g}(\bm{z};\bm{\theta}_k))$. Thus, the problem statement is how to design the neural network $\bm{g}_k(\cdot) = \bm{g}(\cdot;\bm{\theta}_k)$ such that $p_k(\bm{x})=p(\bm{g}(\bm{z};\bm{\theta}_k))$ can be computed.

\subsection{On Theoretical Requirement}

We provide sufficient conditions for realizing EM algorithm associated with learning parameters of GenMM.  

\begin{myprop}[Sufficient conditions]
  EM algorithm for GenMM distribution~\eqref{eq:FirstMixtureModel} is realizable if every generator neural network is an one-to-one mapping function and $M=N$. That means $\bm{g}_k(\bm{z}): \mathbb{R}^N \rightarrow \mathbb{R}^N, \forall k$ are invertible.
\end{myprop}

\noindent \emph{\textbf{Proof:}} We use the multivariate transformation method to prove the proposition. To realize EM for GenMM, it is required to compute $p_k(\bm{x}),\forall k$. Without loss of generality we consider computation of $\gamma_k  = \PP(s_k =1|\bm{x};\bm{\Phi})$ and $p_k(\bm{x}) =  p(\bm{g}_k(\bm{z})) =p(\bm{g}(\bm{z};\bm{\theta}_k))$. 
Let us denote the output of the $k$-th generator by $\tilde{\bm{x}}$ 
and $\tilde{\bm{g}}(\bm{z})
= \bm{g}(\bm{z};\bm{\theta}_k)$. We have $\tilde{\bm{x}} = \tilde{\bm{g}}(\bm{z})$. Under the conditions $M=N$ and invertible $\bm{g}_k(\bm{z}) = \tilde{\bm{g}}(\bm{z})$, there exists an inverse function $\tilde{\bm{g}}^{-1}(\tilde{\bm{x}})=[\tilde{g}_1^{-1}(\tilde{\bm{x}}), \, \tilde{g}_2^{-1}(\tilde{\bm{x}}), \hdots, \tilde{g}_N^{-1}(\tilde{\bm{x}})]^{\top}=\bm{z}$. Then, the Jacobian of this multivariate transformation is
\begin{equation}
  \bm{J} = \left[
    \begin{array}{ccc}
      \frac{\partial \tilde{g}_1^{-1}}{\partial \tilde{x}_1} & \hdots & \frac{\partial \tilde{g}_1^{-1}}{\partial \tilde{x}_N} \\
      \vdots & \hdots & \vdots \\
      \frac{\partial \tilde{g}_N^{-1}}{\partial \tilde{x}_1} & \hdots & \frac{\partial \tilde{g}_N^{-1}}{\partial \tilde{x}_N} 
    \end{array}
  \right].
\end{equation}
Let $\mathrm{det}(\bm{J})$ denotes the determinant of $\bm{J}$. Then, the pdf is
\begin{equation}\label{eq:sufficient-change-var}
  p_k(\bm{x}) = p(\tilde{\bm{x}}) =  p(\bm{z}) \bigg| \mathrm{det}(\bm{J}) \big|_{\bm{z}=\tilde{\bm{g}}^{-1}(\tilde{\bm{x}})}\bigg|.
\end{equation}
Similarly, we can compute the pdf of other mixture components.

\subsection{Algorithm for Learning}
In this section we first discuss about a suitable neural network model for $\bm{g}_k(\bm{z})$ in GenMM and then design the EM algorithm for GenMM. 

\subsubsection{Use of flow-based neural network}\label{subsec-flow-intro}

We use feed-forward neural network to implement every generator. With some notational abuse, assume that $\tilde{\bm{g}}$ is a feed-forward neural network:
$\tilde{\bm{x}}=\tilde{\bm{g}}(\bm{z})$ that has multiple hidden layers
$\tilde{\bm{g}}=\tilde{\bm{g}}^{[L]}\circ \tilde{\bm{g}}^{[L-1]}\circ \cdots
\circ \tilde{\bm{g}}^{[1]}$ and is invertible $\tilde{\bm{f}}=\tilde{\bm{g}}^{-1}$. Then the signal flow can be depicted as
\begin{equation*}
  \vspace{-8pt}
  \centering
  \begin{tikzpicture}
    \node (z) at (0,0) {};
    \node at ($(z)-(0.5,0)$){$\bm{z}=\bm{h}_0$};
    \node (xi1) at (1.5,0) {$\bm{h}_1$};
    \node (xi2) at (3,0) {};
    \node (xi3) at (4.5,0){};
    \node (x) at (6,0) {};
    \node at ($(x)+(0.5,0)$){$\tilde{\bm{x}} = \bm{h}_L$};
    \draw[->] ($(z) + (0.3,0.1)$) -- node[above]{$\tilde{\bm{g}}^{[1]}$} ($(xi1)+(-0.3,0.1)$); 
    \draw[->] ($(xi1)-(0.3,0.1)$) -- node[below]{$\tilde{\bm{f}}^{[1]}$}($(z) - (-0.3,0.1)$);
    \draw[->] ($(xi1) + (0.3,0.1)$) -- node[above]{$\tilde{\bm{g}}^{[2]}$} ($(xi2)+(-0.3,0.1)$); 
    \draw[->] ($(xi2)-(0.3,0.1)$) -- node[below]{$\tilde{\bm{f}}^{[2]}$}($(xi1) - (-0.3,0.1)$);
    \draw[->] ($(xi3) + (0.3,0.1)$) -- node[above]{$\tilde{\bm{g}}^{[L]}$} ($(x)+(-0.3,0.1)$); 
    \draw[->] ($(x)-(0.3,0.1)$) -- node[below]{$\tilde{\bm{f}}^{[L]}$}($(xi3) - (-0.3,0.1)$);
    \draw[dotted,line width = 0.3 mm] (xi2) -- (xi3);
  \end{tikzpicture},
\end{equation*}
where $\tilde{\bm{g}}^{[l]}$ and $\tilde{\bm{f}}^{[l]}$ are the $l$-th layer of $\tilde{\bm{g}}$ and $\tilde{\bm{f}}$, respectively. In a feed-forward neural network, if every layer is invertible,
the full feed-forward neural network is invertible. The inverse
function is given by $\bm{z}=\tilde{\bm{f}}(\tilde{\bm{x}})$. Flow-based
network, proposed in \cite{DBLP:journals/corr/DinhKB14}, is such a
feed-forward neural network, which is further improved in subsequent works \cite{2016arXiv160508803D, 2018arXiv180703039K}. It also has additional advantages as efficient Jacobian computation and low computational complexity. 

For a flow-based neural network architecture, let us assume that
the feature $\bm{h}_l$ at the $l$'th layer has two subparts as
$\bm{h}_l = [\bm{h}_{l,a}^{T} \, , \, \bm{h}_{l,b}^{T}]^{T}$ where
$(\cdot)^{T}$ denotes transpose operation. Then considering $\bm{h}_0 = \bm{z}$, we have the following forward and inverse relations between $(l-1)$'th and $l$'th layers:
\begin{equation}\label{eq-gl}
  \begin{array}{l}
    \bm{h}_{l-1} =
    \begin{bmatrix}
      \bm{h}_{l-1,a}\\
      \bm{h}_{l-1,b}
    \end{bmatrix}
    =
    \begin{bmatrix}
      \bm{h}_{l,a}\\
      \bm{m}_a(\bm{h}_{l,a})\odot \bm{h}_{l,b} + \bm{m}_b(\bm{h}_{l,a})
    \end{bmatrix},\vspace{10pt}\\
    \bm{h}_{l} =
    \begin{bmatrix}
      \bm{h}_{l,a}\\
      \bm{h}_{l,b}
    \end{bmatrix}
    =
    \begin{bmatrix}
      \bm{h}_{l-1,a}\\
      \left(  \bm{h}_{l-1,b} - \bm{m}_b(\bm{h}_{l-1,a}) \right)\oslash \bm{m}_a(\bm{h}_{l-1,a}) 
    \end{bmatrix}, \\
  \end{array}  
\end{equation}
where $\odot$ denotes element-wise product, $\oslash$ denotes
element-wise division, and $\bm{m}_a(\cdot), \bm{m}_b(\cdot)$ can be
complex non-linear mappings (implemented by neural networks).
For the flow-based neural network, the determinant of Jacobian matrix is
\begin{equation}
  \begin{array}{rl}
    \mathrm{det}(\bm{J}) |_{\bm{z}=\tilde{\bm{f}}(\tilde{\bm{x}})} & = \prod_{l=1}^L \det (\bm{J}_l) |_{\bm{h}_{l}},
  \end{array}
\end{equation}
where $\bm{J}_l$ is the Jacobian of the transformation from the $l$-th layer to the $(l-1)$-th layer, i.e., the inverse transformation. We compute the determinate of the Jacobian matrix as
\begin{align}\label{eq-hl-determinate}
  \det (\bm{J}_l)|_{\bm{h}_{l}}& = \det \left[  \pd{\bm{h}_{l-1}}{\bm{h}_l} \right] \nonumber\\
                               & = \det
                                 \begin{bmatrix}
                                   \bm{I}_a & \mathbf{0} \nonumber\\
                                   \pd{\bm{h}_{l-1,b}}{\bm{h}_{l,a}} & \mathrm{diag}(\bm{m}_a(\bm{h}_{l,a}))
                                 \end{bmatrix}\nonumber\\
                               &= \det \left( \mathrm{diag}(\bm{m}_a(\bm{h}_{l,a})) \right),
\end{align}
where $\bm{I}_a$ is identity matrix and $\mathrm{diag}(\cdot)$ returns a square matrix with the elements of $(\cdot)$ on the main diagnal. Then the pdf is
\begin{align}
  p(\tilde{\bm{x}}) & =  p(\bm{z}) \big| \mathrm{det}(\bm{J}) |_{\bm{z}=\tilde{\bm{f}}(\tilde{\bm{x}})}\big| \nonumber\\
            &  = p(\bm{z}) \prod_{l=1}^L \abs{\det\left( \mathrm{diag}(\bm{m}_a(\bm{h}_{l,a}))]  \right)}.
\end{align}
\autoref{eq-gl} describes a \textit{coupling} mapping between layers. Since the coupling has a partial identity mapping, direct concatenation of multiple such coupling mappings would result in a partial identity mapping of the whole neural network $\tilde{\bm{g}}$. Alternating the positions of identity mapping \cite{2016arXiv160508803D} or using $1\times1$ convolution operations \cite{2018arXiv180703039K} before each coupling mapping is used to treat the issue.
Furthermore, \cite{2016arXiv160508803D}\cite{2018arXiv180703039K} split some hidden
layer signal $\bm{h}$ and model a part of it directly as standard Gaussian to reduce computation and memory burden.

\subsubsection{EM Algorithm for GenMM}
\label{sec-algo-genmm}
The mixture model GenMM is illustrated in \autoref{dia-emgm-nm},
where $K$ generators with a certain prior distribution share the same
latent distribution $p(\bm{z})$. With a flow-based neural network used as the generator $\bm{g}_k$ for the $k$'th mixture component in GenMM, the pdf $p_k(\bm{x})$ for any
$\bm{x}$ can be computed exactly. Recall that $p_k(\bm{x}) =  p(\bm{g}_k(\bm{z})) =p(\bm{g}(\bm{z};\bm{\theta}_k))$. Let $\bm{f}_k$ be the inverse of $\bm{g}_k$. Then, %
the posterior probability can be computed further from \autoref{eq-genmm-e-step} as
\begin{align}\label{eq-genmm-gamma}
  \gamma_k({\bm{\Phi}}^{\mathrm{old}}) 
  = & \frac{\pi_k^{\mathrm{old}} p(\bm{g}(\bm{z};\bm{\theta}_k))}{\sum_{j=1}^K\;
      \pi_j^{\mathrm{old}} p(\bm{g}(\bm{z};\bm{\theta}_j))} \nonumber\\
  = &\frac{\pi_k^{\mathrm{old}} p(\bm{f}_k(\bm{x})) \big|\det\left( \pd{\bm{f}_k(\bm{x})}{\bm{x}} \right)\big|}{\sum_{j=1}^K\; \pi_j^{\mathrm{old}} p(\bm{f}_j(\bm{x})) \big|\det\left( \pd{\bm{f}_j(\bm{x})}{\bm{x}} \right)\big|},
\end{align}
and the objective function in the M-step can be written as
\begin{align}\label{eq-genmm-obj}
  &\Qq\left(\bm{\Phi},\bm{\Phi}^{\mathrm{old}}\right)  =\sum_{i=1}^n \sum_{k=1}^{K}\gamma_k^{(i)}(\bm{\Phi}^{\mathrm{old}})\bigg[ \log\;\pi_k 
    \nonumber\\ &+\log\;p(\bm{f}_k(\bm{x}^{(i)})) 
                  + \log\;\bigg|\det\left(\pd{\bm{f}_k(\bm{x}^{(i)})}{\bm{x}^{(i)}} \right)\bigg|\bigg],
\end{align}
where $n$ denotes the number of data samples. We usually deal with a large dataset for model learning, i.e. $n$ is large. In that case we implement the EM algorithm in batch fashion. 
 Recall that $ \bm{\Phi}= \{\bm{\pi},\bm{\theta}_1, \dots, \bm{\theta}_K \}$ and hence the M-step optimization problem $\arg \max_{\bm{\Phi}} \mathcal{Q} (\bm{\Phi},\bm{\Phi}^{\mathrm{old}})$ is addressed in two steps: (a) optimization of $\{ \bm{\theta}_k \}_{k=1}^{K}$, and (b) optimization of $\bm{\pi}$.


Finding a closed-form solution for the problem $\arg \max_{\{\bm{\theta}_k\}_{k=1}^{K}} \mathcal{Q} (\bm{\Phi},\bm{\Phi}^{\mathrm{old}})$ is challenging. Instead, we do the batch-size gradient decent to optimize w.r.t. $\{\bm{\theta}_k\}_{k=1}^{K}$.
Further, optimization in the batch fashion leads to a practical problem as follows. 
Since $\bm{\theta}_k$ is the parameter set of neural networks $\bm{g}_k$, one update step of gradient decent would update the generator $\bm{g}_k$ and we would lose the old mixture model parameter set $\bm{\Phi}^{\mathrm{old}}$ that is needed to compute the posteriors $\gamma_k(\bm{\Phi}^{\mathrm{old}})$ and to update $\bm{\pi}$. Thus, in learning GenMM, we maintain two such models with parameter sets $\bm{\Phi}$ and $\bm{\Phi}^{\mathrm{old}}$, respectively. At the beginning of an EM step, $\bm{\Phi} = \bm{\Phi}^{\mathrm{old}}$. While we optimize $\{\bm{\theta}_k\}_{k=1}^{K}$ of $\bm{\Phi}$ with batch-size gradient decent, we use the model with old parameter set $\bm{\Phi}^{\mathrm{old}}$ to do posterior computation and update of $\bm{\pi}$. At the end of the EM step, the old parameter set is replaced by the updated one: $\bm{\Phi}^{\mathrm{old}}\gets \bm{\Phi}$.

\begin{algorithm}[t]
  \caption{EM for learning GenMM}\label{flow-algo-em}
  \begin{algorithmic}[1]
    \State {\bfseries Input:}
      Latent distribution: $p(\bm{z})$. Empirical distribution $P_d(\bm{x})$ of the input dataset;
      \State Set a total number of epochs $T$ for training, a prior distribution $\bm{\pi}$, EM update gap $t_{\mathrm{EM}}$;
    \Statex  Set a learning rate $\eta$. 
    \State Build two models with parameter sets:
    \Statex ${\bm{\Phi}}^{\mathrm{old}}=\{ \bm{\pi}^{\mathrm{old}},
    \bm{\th}_1^{\mathrm{old}},\hdots,
    \bm{\th}_K^{\mathrm{old}} \}$,
    \Statex ${\bm{\Phi}} = \{ \bm{\pi},
    \bm{\th}_1,\hdots, \bm{\th}_K \}$.
    \State Initialize the generator prior distribution $\pi_k = 1/K$;\\ Initialize $\bm{\theta}_k$ of $\bm{g}_k$, for all $k=1,\dots,K$ randomly. 
    \State $\bm{\Phi}^{\mathrm{old}} \gets \bm{\Phi}$.
    \For { {epoch} $t < T$}
    \For{the iteration in epoch $t$} 
    \State Sample a batch of data $\left\{ \bm{x}^{(i)}
    \right\}_{i=1}^{n_b}$ from the dataset $P_d(\bm{x})$
    \State Compute $\gamma_k^{(i)}(\bm{\Phi}^{\mathrm{old}})$ as in \autoref{eq-genmm-gamma},
    for all $\bm{x}^{(i)}$ and $k=1, \dots, K$
    \State Compute
    $\Qq\left(\bm{\Phi},\bm{\Phi}^{\mathrm{old}}\right)$ as in \autoref{eq-genmm-obj}
    
    \State $\partial{g_k} \gets \nabla_{\bm{\th}_k} \frac{1}{n^{b}}\Qq\left(
      \bm{\Phi}, \bm{\Phi}^{\mathrm{old}}\right)$,
    $\forall \bm{\th}_k \in \bm{\Phi}$
    \State $\bm{\th}_k \gets \bm{\th}_k + \eta \cdot \partial{g_k}$, $\forall \bm{\th}_k \in \bm{\Phi}$
    \EndFor
    \If{$(t \mod t_{\mathrm{EM}} ) = 0$}
    \State $\pi_k \gets \EE_{P_d}\left[ \gamma_k \right]$ 
    \State $\bm{\Phi}^{\mathrm{old}} \gets \bm{\Phi}$.
    \EndIf
    \EndFor
  \end{algorithmic}
\end{algorithm}

Then we discuss the optimization of the prior distribution $\bm{\pi}$. The optimization problem is 
\begin{equation}\label{eq-pi-update1}
  \bm{\pi}^{\mathrm{new}} = \displaystyle \arg \max_{\bm{\pi}} \mathcal{Q}
  (\bm{\Phi},\bm{\Phi}^{\mathrm{old}}), \,\,
  \mathrm{s.t.} \sum_{k=1}^{K}\pi_k = 1.
\end{equation}
The update of prior follows the solution
\begin{equation}\label{eq-pi-solution}
  \pi_k^{\mathrm{new}} = \frac{1}{n}\sum_{i=1}^{n}\gamma_k^{(i)}(\bm{\Phi}^{\mathrm{old}}).
\end{equation}
The detail to get the solution is derived in the subsection~\ref{subsubsec:Proof_for_update}. For a given dataset with empirical distribution $P_d(\bm{x})$,
$\gamma_k$ is evaluated with batch data in order to calculate the cost
$\Qq\left(\bm{\Phi},\bm{\Phi}^{\mathrm{old}}\right)$ and to update the
parameter $\bm{\theta}_k$ of $\bm{g}_k$. We accumulate the values of $\gamma_k$ of batches and
average out for one epoch to update $\bm{\pi}$, {i.e.}, $\pi_k \gets \EE_{P_d}\left[ \gamma_k \right]$.

We summarize the EM algorithm for GenMM in Algorithm~\autoref{flow-algo-em}.
In implementation, to avoid numerical computation problem, $\log{p(\bm{g}(\bm{z}; \bm{\theta}_k))}$ is
scaled by the dimension of signal $\bm{x}$ in order to compute $\gamma_k$. 


\subsubsection{Proof for update of $\pi$}
\label{subsubsec:Proof_for_update}


The optimization of $\bm{\pi}$ is addressed in the following Lagrangian form
\begin{equation}
  \Ff(\bm{\Phi}) = \mathcal{Q}(\bm{\Phi},\bm{\Phi}^{\mathrm{old}}) + \lambda
  \left( 1 - \sum_{k=1}^{K}\pi_k \right),
\end{equation}
where $\lambda$ is the Lagrange multiplier. Then
\begin{align}
  \bm{\pi}^{\mathrm{new}} =& \arg \max_{\bm{\pi}} \Ff(\bm{\Phi}) \nonumber\\
                   =&  \arg \max_{\bm{\pi}}\sum_{i=1}^n
                     \sum_{k=1}^{K}\gamma_k^{(i)}(\bm{\Phi}^{\mathrm{old}})\bigg[
                      \log\;\pi_k +   \log\;p(\bm{f}_k(\bm{x}^{(i)})) \nonumber \\
  &+ \log\;\bigg|\det\left(
      \pd{\bm{f}_k(\bm{x}^{(i)})}{\bm{x}^{(i)}}
                      \right)\bigg|\bigg] + \lambda  \left( 1 - \sum_{k=1}^{K}\pi_k \right) \nonumber\\
    =& \arg \max_{\bm{\pi}} \sum_{i=1}^n
                     \sum_{k=1}^{K}\gamma_k^{(i)}(\bm{\Phi}^{\mathrm{old}})
      \log\;\pi_k+ \lambda  \left( 1 - \sum_{k=1}^{K}\pi_k \right),
\end{align}
where $\bm{f}_k = \bm{g}_k^{-1}$. Then solving
\begin{equation}
   \pd{\Ff}{\pi_k} = 0, k=1, 2, \cdots, K,
 \end{equation}
 we get 
 $
   \pi_k = \frac{1}{\lambda}
   \sum_{i=1}^{n}\gamma_k^{(i)}(\bm{\Phi}^{\mathrm{old}}), \forall k.
 $
With condition $\sum_{k=1}^{K}\pi_k =1$, we have
 $
\lambda = \sum_{k=1}^{K}\sum_{i=1}^{n}\gamma_k^{(i)}(\bm{\Phi}^{\mathrm{old}}) =n.
$
Therefore, the solution is
 $
\pi_k = \frac{1}{n}
   \sum_{i=1}^{n}\gamma_k^{(i)}(\bm{\Phi}^{\mathrm{old}}), \forall k.
$
Note that the updated prior parameter $\pi_k$
is non-negative due to the non-negativity of the posterior $\gamma_k{(i)}$.

\subsection{On Convergence of GenMM}
{In general the convergence of EM is guaranteed only in some cases, cf. \cite{wu1983convergence}. However, under some conditions our GenMM converges. In what follows we present the convergence arguments. 
}

\begin{myprop}
  Assume that for all $k$, the parameters $\bm{\theta}_k$ are in a compact set such that the corresponding mapping $\bm{g}_k$ is invertible. Assume further that all generator mappings fulfill that $\bm{f}_k$ and $\frac{\partial \bm{f}_k}{\partial \bm{x}}$ are continuous functions of $\bm{\theta}_k$. Then GenMM converges. 
\end{myprop}
\begin{proof}
  Assume that the assumption holds. Then the determinant term $\det(\bm{J})$ in \autoref{eq:sufficient-change-var} is a continuous function of $\bm{\theta}_k$. Due to \autoref{eq:sufficient-change-var} and the continuity of Gaussian density $p(\bm{z})$, the pdf $p_k(\bm{x})$ is a continuous function of $\bm{\theta}_k$. Therefore, $p(\bm{x})$ given in \autoref{eq:FirstMixtureModel} is a continuous function of $\bm{\Phi}$. Denote the likelihood in \autoref{eq:max-genmm} as $\Ll(\bm{\Phi})=\log \textstyle\prod_{i} p(\bm{x}^{(i)};\bm{\Phi})$. The maximum value of $\Ll(\bm{\Phi})$ is bounded due to continuity of $p(\bm{x})$ w.r.t. $\bm{\Phi}$.
  Define $\Bb(\bm{\Phi}) = \Qq\left(\bm{\Phi},\bm{\Phi}^{\mathrm{old}}\right) - \sum_{i=1}^n\sum_{k=1}^{K}\gamma_k^{(i)}(\bm{\Phi}^{\mathrm{old}}) \log\gamma_k^{(i)}(\bm{\Phi}^{\mathrm{old}})$. It is well known that $\Bb(\bm{\Phi})$ is a lower bound on the likelihood function $\Ll(\bm{\Phi})$, i.e. $\Ll(\bm{\Phi}) \geq \Bb(\bm{\Phi})$. Note that the essence of EM algorithm is that the likelihood function value is elevated by increasing the value of its lower bound $\Bb(\bm{\Phi})$. Since the maximum value of the log-likelihood $\Ll(\bm{\Phi})$ is finite, $\Bb(\bm{\Phi})$ can not grow unbounded.

\end{proof}

\section{A low-complexity model}
There are $K$ neural networks in GenMM, which makes GenMM a high-complexity model.
We now propose a low-complexity model where parameters are shared. This is motivated by many machine learning setups where model parameters are shared across model components. For example, this techniques is applied as use of shared covariance matrices in a tied Gaussian mixture model, in linear discriminant analysis \cite{bellegarda1990tiedmixture, Kimball:1993:UTD:1075671.1075694, Bishop:2006:PRM:1162264}, and use of common subspace in non-negative matrix factorization \cite{Gupta2013}. Based on the idea of sharing parameters, we propose a low-complexity model which we refer to as latent mixture model as follow.

\subsection{Latent mixture model}\label{subsec-latmm}

In this generative model, we use a latent variable $\bm{z}$ that has the following Gaussian mixture distribution
\begin{equation}
  p(\bm{z}) = \sum_{k=1}^K \pi_k p_k(\bm{z}), 
\end{equation}
where $p_k(\bm{z})$ is pdf of Gaussian distribution $\mathcal{N}(\bm{z};\bm{\mu}_k,\bm{C}_k)$ with mean $\bm{\mu}_k$ and covariance $\bm{C}_k$.
The data $\bm{x}$ is
assumed to be generated in the model using a single neural network
$\bm{g}(\bm{z}): \mathbb{R}^M \rightarrow \mathbb{R}^N$ as
$\bm{x}=\bm{g}(\bm{z};\bm{\theta})$, where $\bm{\theta}$
is the set of parameters of the neural network. The diagram of this mixture model
is shown in \autoref{dia-emgm-sm}. Similarly, we use $\ubar{\bm{\Phi}}$ to denote the set of all parameters $\{ \bm{\pi},
\bm{\mu}_1,\hdots, \bm{\mu}_K, \bm{C}_1,\hdots,
\bm{C}_K, \bm{\theta} \}$. Furthermore, we also have a categorical variable $\bm{s}$ to indicate which underlying source is chosen. The density function of the proposed latent mixture model (LatMM) is given as
\begin{align}\label{eq:SecondMixtureModel}
  p(\bm{x};\ubar{\bm{\Phi}}) & = \textstyle\sum_{k=1}^K \pi_k p_k(\bm{x}) \nonumber\\
                             &= \textstyle\sum_{k=1}^K \pi_k  p(\bm{g}(\bm{z}; \bm{\theta})| s_k=1) \nonumber \\
                             & = \textstyle\sum_{k=1}^K \pi_k  p(\bm{g}(\bm{z};\bm{\theta}); \bm{\mu}_k, \bm{C}_k).
\end{align}
The LatMM is illustrated in Figure~\ref{dia-emgm-sm} where the neural network $\bm{g}$ is shared. Learning of LatMM requires solving the maximum likelihood estimation problem
\begin{equation}
  \ubar{\hat{\bm{\Phi}}} =    \arg \max_{\ubar{\bm{\Phi}}} \log \prod_{i} p(\bm{x}^{(i)};\ubar{\bm{\Phi}}),
\end{equation}
which we address using EM.
We have
\begin{align}\label{eq-latMM-gamma}
  \hspace{-9pt}\ubar{\gamma}_k = \PP(s_k =1|\bm{x};\ubar{\bm{\Phi}})  
  = \frac{\pi_k p(\bm{g}(\bm{z};\bm{\theta});\bm{\mu}_k, \bm{C}_k)}{\sum_{l=1}^K\; \pi_l p(\bm{g}(\bm{z};\bm{\theta});\bm{\mu}_l, \bm{C}_l)}.
\end{align}

Similar to the case of GenMM, realization of the corresponding EM
algorithm associated with LatMM in \autoref{eq:SecondMixtureModel}
also has technical challenges on computing the posterior distribution $\ubar{\gamma}_k$ and the joint likelihood $\log{\pi_k p_k(\bm{x})}$. They require explicit computation of the conditional density function $p_k(\bm{x}) = p(\bm{g}(\bm{z};\bm{\theta})| s_k=1) = p(\bm{g}(\bm{z};\bm{\theta});\bm{\mu}_k, \bm{C}_k) $. In LatMM, $\bm{g}(\bm{z}): \mathbb{R}^N \rightarrow \mathbb{R}^N$ is also required to be invertible. We model $\bm{g}$ by a flow-based neural network as explained in \autoref{subsec-flow-intro}.
Then, the problem is how to learn the parameters of LatMM. 
\begin{figure}
  \begin{tikzpicture}
    \tikzstyle{enode} = [thick, draw=blue, ellipse, inner sep = 1pt,  align=center]
    \tikzstyle{nnode} = [thick, rectangle, rounded corners = 2pt,minimum size = 0.8cm,draw,inner sep = 2pt]
    \node[enode] (z1) at (0,1.8) {$\bm{z}_1\sim p_1(\bm{z})$};
    \node[enode] (z2) at (0,0.5){$\bm{z}_2\sim p_{2}(\bm{z})$};
    \node[enode] (zK) at (0,-1.8) {$\bm{z}_K\sim p_{K}(\bm{z})$}; 
    \node[enode] (x) at (5.5,0){$\bm{x}\sim p(\bm{x};\ubar{\bm{\Phi}})$};
    \node[nnode] (g) at (3.2,0) {$\bm{g}$};
    \draw[dotted,line width=2pt] (0,-0.3) -- (0,-1.2);
    \filldraw[->] (1.8, 0.5)circle (2pt) -- node[above=0.2]{$\bm{s}\sim \bm{\pi}$} (g) ;
    \draw[->] (z1) -- (1.6, 1.8);
    \draw[->] (z2) -- (1.6, 0.5);
    \draw[->] (zK) -- (1.6, -1.8);
    \draw[->] (g) to (x);
  \end{tikzpicture}
  \caption{Diagram of Latent Mixture Model (LatMM).}\label{dia-emgm-sm}
  \vspace{0.1cm}
\end{figure}
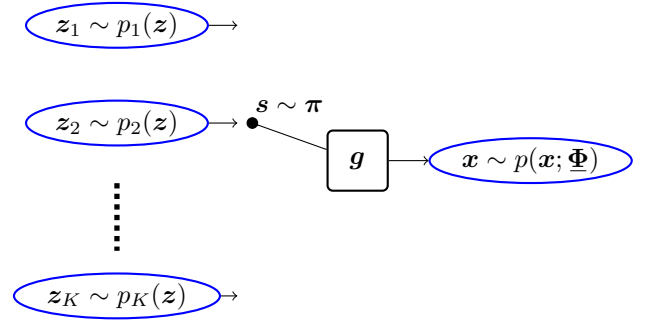

\subsubsection{EM Algorithm for LatMM}
\begin{algorithm}[t]
  \caption{EM for learning LatMM}\label{flow-algo-sem}
  \begin{algorithmic}[1]
    \State {\bfseries Input:} Empirical distribution $P_d(\bm{x})$ of dataset;
    \Statex Latent mixture distribution: 
    \Statex $\sum_{k=1}^{K}\pi_k \Nn\left(\bm{z}; \bm{\mu}_k, \mathrm{diag}(\bm{\sigma}_k^2)\right)$
    \State Set a total number of epochs $T$ of training, prior $\bm{\pi}$ update gap $t_{\pi}$, EM update gap $t_{\mathrm{EM}}$, a learning rate $\eta$; Set hyperparameter $a$ , $b$ for prior of
    $\bm{\sigma}_k^{-1}, \forall k$.
    \State Build two models with parameter sets:
    \Statex $\ubar{\bm{\Phi}}^{\mathrm{old}}=\{ \bm{\pi}^{\mathrm{old}},
    \bm{\mu}_1^{\mathrm{old}},\hdots, \bm{\mu}_K^{\mathrm{old}}, \bm{\sigma}_1^{\mathrm{old}},\hdots,
    \bm{\sigma}_K^{\mathrm{old}}, \bm{\theta}^{\mathrm{old}} \}$,
    \Statex $\ubar{\bm{\Phi}} = \{ \bm{\pi},
    \bm{\mu}_1,\hdots, \bm{\mu}_K, \bm{\sigma}_1,\hdots,
    \bm{\sigma}_K, \bm{\theta} \}$.
    \State Initialize the generator prior distribution $\pi_k = 1/K$ and initialize its $\bm{\theta}$ for $\bm{g}$, $\bm{\mu}_k$, $\bm{\sigma}_k$, $\forall k$ randomly.
    \State $\ubar{\bm{\Phi}}^{\mathrm{old}} \gets \ubar{\bm{\Phi}}$
    \For {epoch $t < T$}
    \For{the iteration in epoch $t$}
    \State Sample a batch of data $\left\{ \bm{x}^{(i)}
    \right\}_{i=1}^{n_b}$ from dataset
    \State Compute $\ubar{\gamma}_k(\ubar{\bm{\Phi}}^{\mathrm{old}})$ by \autoref{em-latmm-gamma}, $\forall \bm{x}^{(i)}$ and $k=1, 2, \cdots, K$ 
    \State Compute
    $\ubar{\Qq}\left(\ubar{\bm{\Phi}},\ubar{\bm{\Phi}}^{\mathrm{old}}\right)$ in \autoref{eq-latmm-obj}

    \State $\partial{\bm{\th}}, \partial{\bm{\mu}_k}, \partial{\bm{\sigma}_k}\gets 
    \nabla_{\bm{\theta}, \bm{\mu}_k, \bm{\sigma}_k} 
    \frac{1}{n^{b}}\ubar{\Qq}\left(\ubar{\bm{\Phi}},\ubar{\bm{\Phi}}^{\mathrm{old}}\right)
    +\frac{1}{K}
    \log\prod_{k=1}^K
    \Gamma(\bm{\sigma}_k^{-1};
    a, b)$ 
    \State $\bm{\th} \gets \bm{\th} + \eta \cdot \partial{\bm{\theta}}$
    \State $\bm{\mu}_k \gets \bm{\mu}_k + \eta \cdot \partial{\bm{\mu}_k}, \forall k$
    \State $\bm{\sigma}_k \gets \bm{\sigma}_k + \eta \cdot \partial{\bm{\sigma}_k},
    \forall k$
    \EndFor

    \If{$(t \mod t_{\mathrm{EM}} )=0$} 
    \State $\pi_k \gets \EE_{P_d}\left[ \ubar{\gamma}_k(\ubar{\bm{\Phi}}^{\mathrm{old}}) \right]$
    \State $\ubar{\bm{\Phi}}^{\mathrm{old}} \gets \ubar{\bm{\Phi}}$
    \EndIf
    
    
    \EndFor
  \end{algorithmic}
\end{algorithm}

Algorithm~\ref{flow-algo-sem} summarizes the EM algorithm for LatMM.
LatMM is used to
learn one generative model that gets input from a mixture latent source
distribution with one single generator $\bm{g}$. For simplicity, we set
the covariance matrix of each latent Gaussian source as a diagonal
matrix, $\bm{C}_k = \mathrm{diag}(\bm{\sigma}_k^2) $. Each component $p_k(\bm{z})$ of
the latent source $p(\bm{z})$ can be obtained by an affine transform from
the standard Gaussian, {i.e.}, $\bm{z}_k \sim p_k(\bm{z})$ can be
obtained by a linear layer of neural network with $\bm{z}_k = \bm{\mu}_k
+ \bm{\sigma}_k \bm{\epsilon}, \bm{\epsilon} \sim \Nn(\bm{0}, \bm{I})$.
According to \autoref{subsec-latmm}, the posterior and objective function in M-step of LatMM can be computed as
\begin{align}\label{em-latmm-gamma}
  &\ubar{\gamma}_k(\ubar{\bm{\Phi}}^{\mathrm{old}}) = \frac{\pi_k^{\mathrm{old}} p_k\left(\bm{z}\right)}{\sum_{j=1}^K\;\pi_j^{\mathrm{old}} p_j\left(\bm{z}\right)}\bigg|_{\bm{z}=\bm{f}(\bm{x})},
\end{align}\vspace{-0.42cm}
\begin{align}\label{eq-latmm-obj}
  &\hspace{-8pt}\ubar{\Qq}\left( \ubar{\bm{\Phi}},\ubar{\bm{\Phi}}^{\mathrm{old}}\right)
    =\sum_{i=1}^n \log\;\bigg|\det\left(
    \pd{\bm{f}(\bm{x}^{(i)})}{\bm{x}^{(i)}} \right)\bigg| \nonumber \\ 
  &\hspace{-10pt}+\hspace{-3pt} \sum_{k=1}^{K}\ubar{\gamma}_k^{(i)}(\ubar{\bm{\Phi}}^{\mathrm{old}})\bigg[ \log\hspace{-1pt}{\pi_k}
    +\log\hspace{-1pt}{p_k\hspace{-2pt}\left(\bm{f}(\bm{x}^{(i)}); \bm{\mu}_k, \bm{\sigma}_k^2\right)}\bigg], 
\end{align}
where $\bm{f}$ is the inverse of $\bm{g}$.
Similar to \autoref{sec-algo-genmm}, update of prior $\bm{\pi}$
follows $\pi_k \leftarrow \EE_{P_d}[ \ubar{\gamma}_k(\bm{x})
]$. However, we need to consider the following issue when learning the parameters of
Gaussian mixture source $p(\bm{z}) = \sum_{k=1}^K \pi p_k(\bm{z})$. If a component of the mixture source overfits and collapses onto a data sample, the likelihood can be large but the parameter learning can be problematic. This problem is known as the singularity problem of Gaussian mixture \cite{Bishop:2006:PRM:1162264}.
We avoid this problem by using the following alternatives:
\begin{itemize}[noitemsep, nolistsep]
\item Assume that for each $\forall k=1, 2, \cdots, K$, there is a parameter prior distribution for $\bm{C}_k=\mathrm{diag}(\bm{\sigma}_k^2)$.
   To be specific, assume that the parameter prior distribution of the precision $\bm{\sigma}_k^{-1}$ is
  $\Gamma(\bm{\sigma}_k^{-1};a, b)$, where $\Gamma(\cdot; a, b)$ is
  Gamma distribution with parameter $a$ and $b$. Then, the objective function of the optimization problem w.r.t. $\bm{\Phi}$ is reformulated as
  \begin{equation}\label{eq-latmm-obj1}
    \umax{ \ubar{\bm{\Phi}}}   \frac{1}{n}
    \ubar{\Qq}\left(\ubar{\bm{\Phi}},\ubar{\bm{\Phi}}^{\mathrm{old}} \right)
    +\frac{1}{K}
    \log\prod_{k=1}^K
    \Gamma(\bm{\sigma}_k^{-1};
    a, b).
  \end{equation}
\item Alternatively, we use an $l_2$ regularization on $\bm{\sigma}_k$,
  which formulates the optimization step as
  \begin{equation}
    \umax{\ubar{\bm{\Phi}}}   \frac{1}{n} \ubar{\Qq}\left(\ubar{\bm{\Phi}},\ubar{\bm{\Phi}}^{\mathrm{old}}\right)-\lambda \sum_{k=1}^K \frac{(1- \bm{\sigma}_k)^2}{K},
  \end{equation}
  where $\la$ is the regulation parameter.
\end{itemize}


\subsection{On complexity of models and new variant models}
We have proposed two models, GenMM and LatMM. GenMM has a high complexity whereas LatMM has a low complexity. Due to their difference in model complexity as well as training complexity, their usage efficiency is highly application-dependent. For example, when the training data is limited, it may be advisable to use LatMM.

It is possible to combine GenMM and LatMM to obtain new models. A simple way is to replace the latent source $p(\bm{z})$ of GenMM by a LatMM model. 
This new combined model has a higher complexity than both GenMM and LatMM. 

To get a less complex model than GenMM, another new model can be derived by modifying the architecture of LatMM. There are multiple latent sources $p_k(\bm{z})$, $k=1, 2, \cdots, K$, in LatMM. If we assume that each such latent source $p_k(\bm{z})$ is induced by a latent generator network, we can obtain a new model that has a common-and-individual architecture. Each latent generator of its corresponding latent source has its own parameters and acts as an individual part. The common part of the new model transforms signal between observable signal $\bm{x}$ and latent signal $\bm{z}$ (generated by the latent generator networks). The common-and-individual technique is prevalently used in machine learning systems\cite{sundman2016design, SUNDMAN2014298}.


Therefore several new models can be derived using our proposed models, GenMM and LatMM. In spite of the scope and potential, development of analytical methodology to derive new model architectures turns out to be challenging. Traditionally the development is trial-and-error driven. Development of new model architectures by combining GenMM and LatMM will be investigated in future, and not to be pursued in this article.



\section{Experiments Results}\label{sec:experiments}

In this section, we evaluate our proposed mixture models for generating samples and maximum likelihood classification. We will show encouraging results.


\begin{figure*}[!ht]
  \captionsetup[subfigure]{justification=centering}
  \centering
  \begin{subfigure}{.4\textwidth}
    \centering
    \includegraphics[width=1\linewidth]{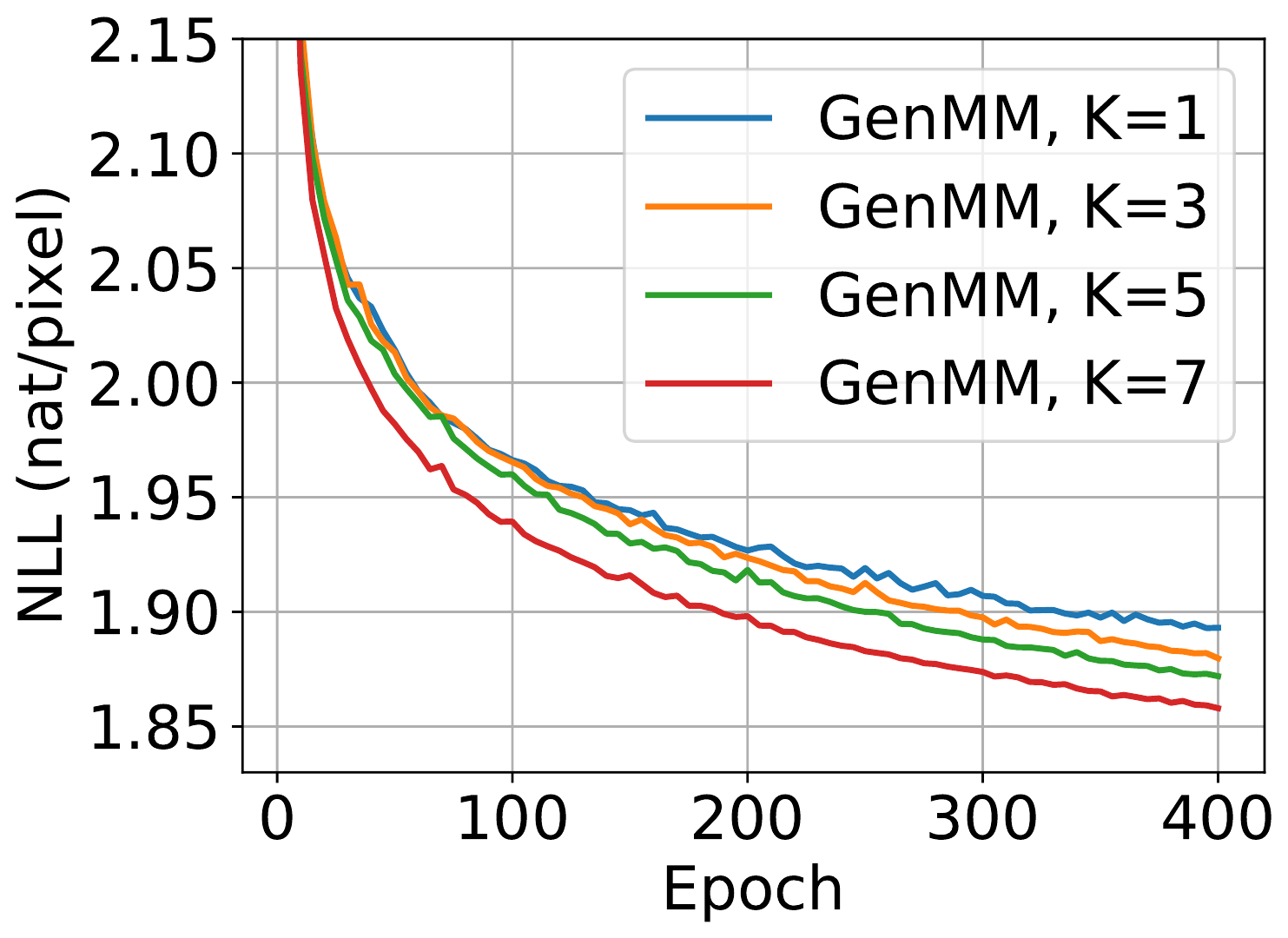}
    \vspace{-0.6cm}
    \caption{Dataset MNIST}
    \label{fig-genmm-mnist-nll-curve}
  \end{subfigure}\hspace{1cm}
  \begin{subfigure}{.4\textwidth}
    \centering
    \includegraphics[width=1\linewidth]{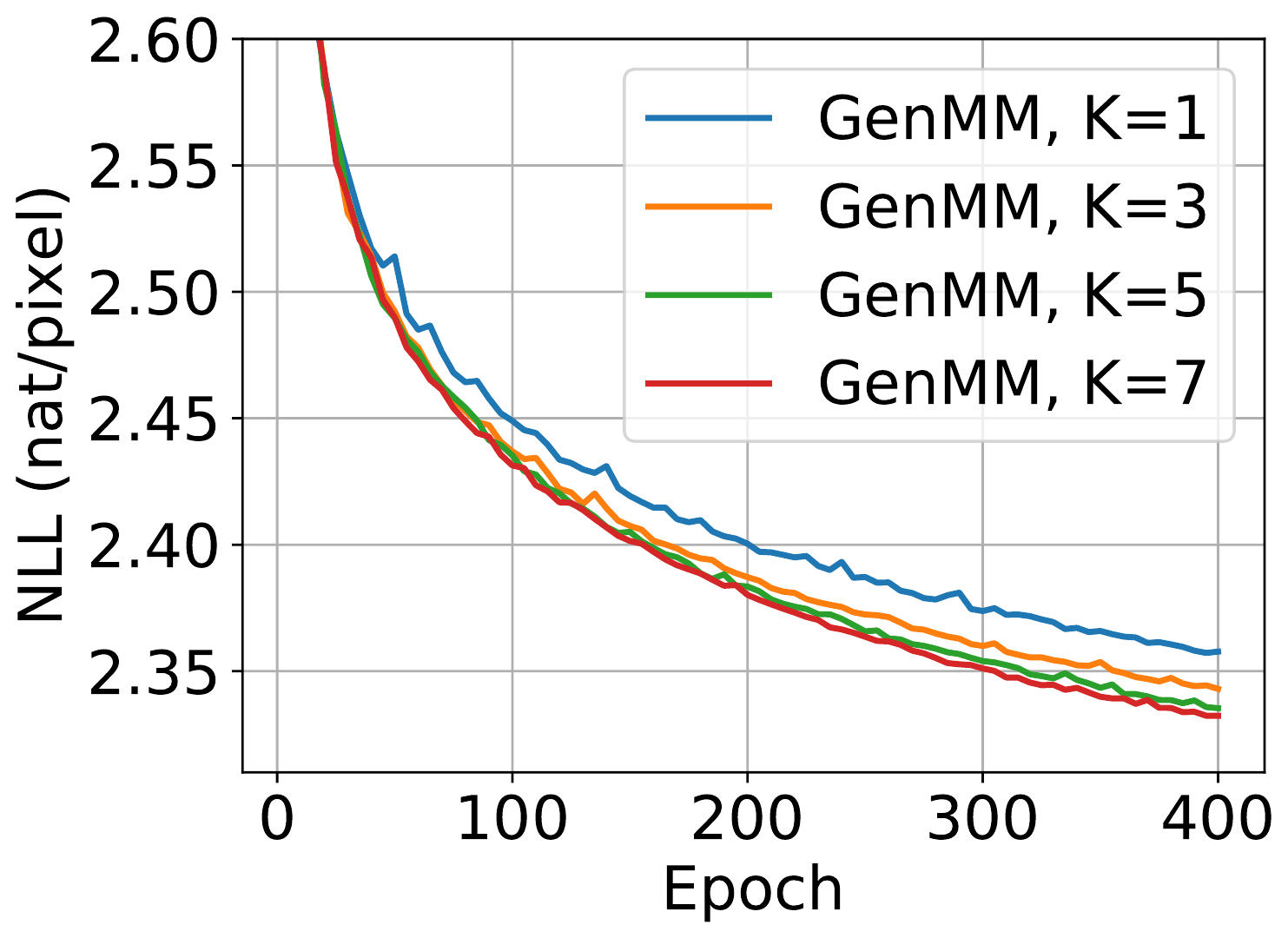}
    \vspace{-0.6cm}
    \caption{Dataset Fashion-MNIST}
    \label{fig-genmm-fsh-nll-curve}
  \end{subfigure}
  \vspace{-0.3cm}
  \caption{NLL (Unit: nat/pixel) of GenMM versus training epochs with different number of mixture component $k$. (a) $10000$ images from MNIST is used for training, (b) $10000$ images from Fashion-MNIST is used for training.}
  \label{fig:genmm-nll}
\end{figure*}

\begin{figure*}[!ht]
  \captionsetup[subfigure]{justification=centering}
  \centering
  \begin{subfigure}{.4\textwidth}
    \centering
    \includegraphics[width=1\linewidth]{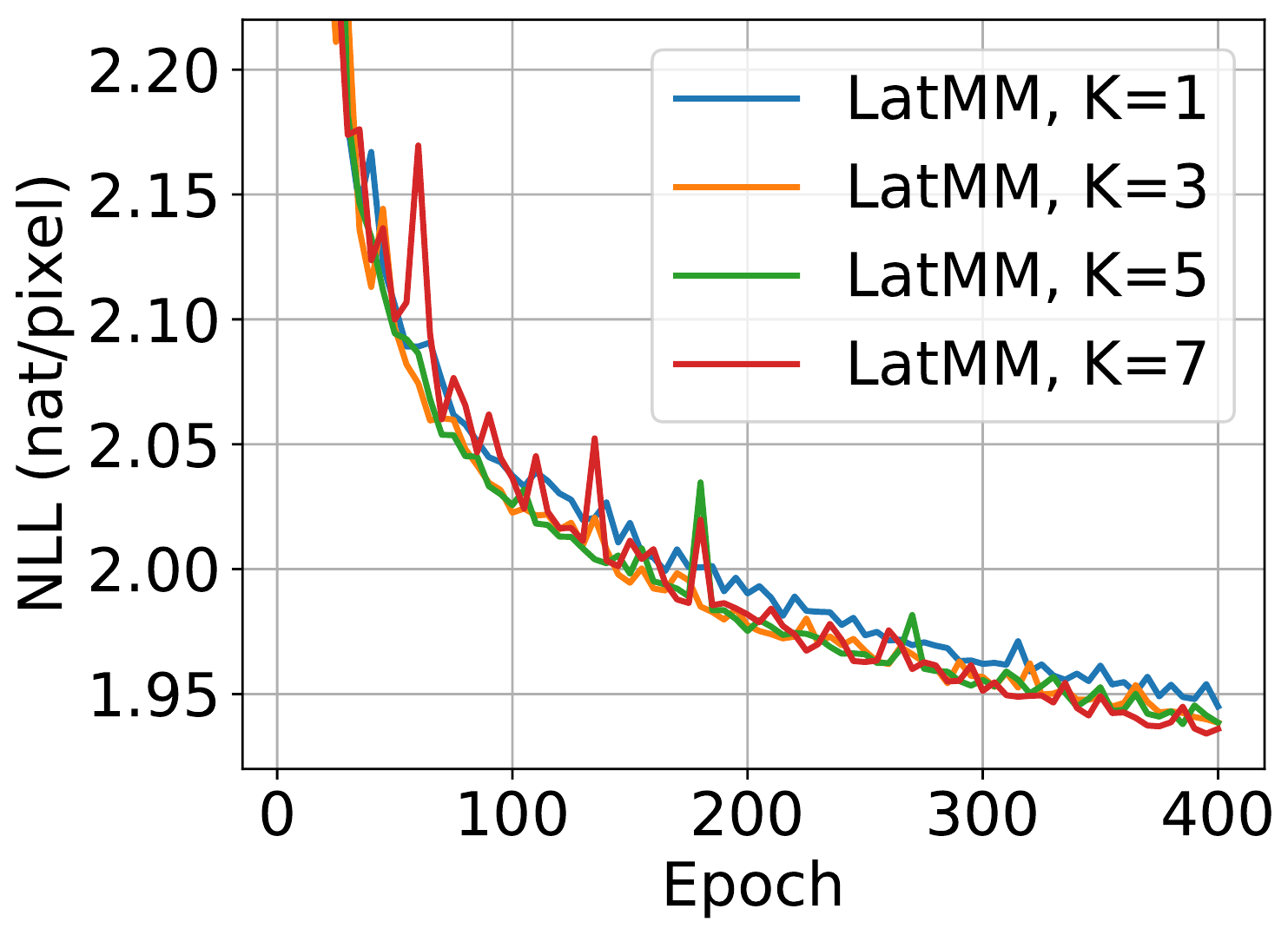}
    \vspace{-0.6cm}
    \caption{Dataset of MNIST}
    \label{fig-latmm-mnist-nll-curve}
  \end{subfigure}\hspace{1cm}
  \begin{subfigure}{.4\textwidth}
    \centering
    \includegraphics[width=1\linewidth]{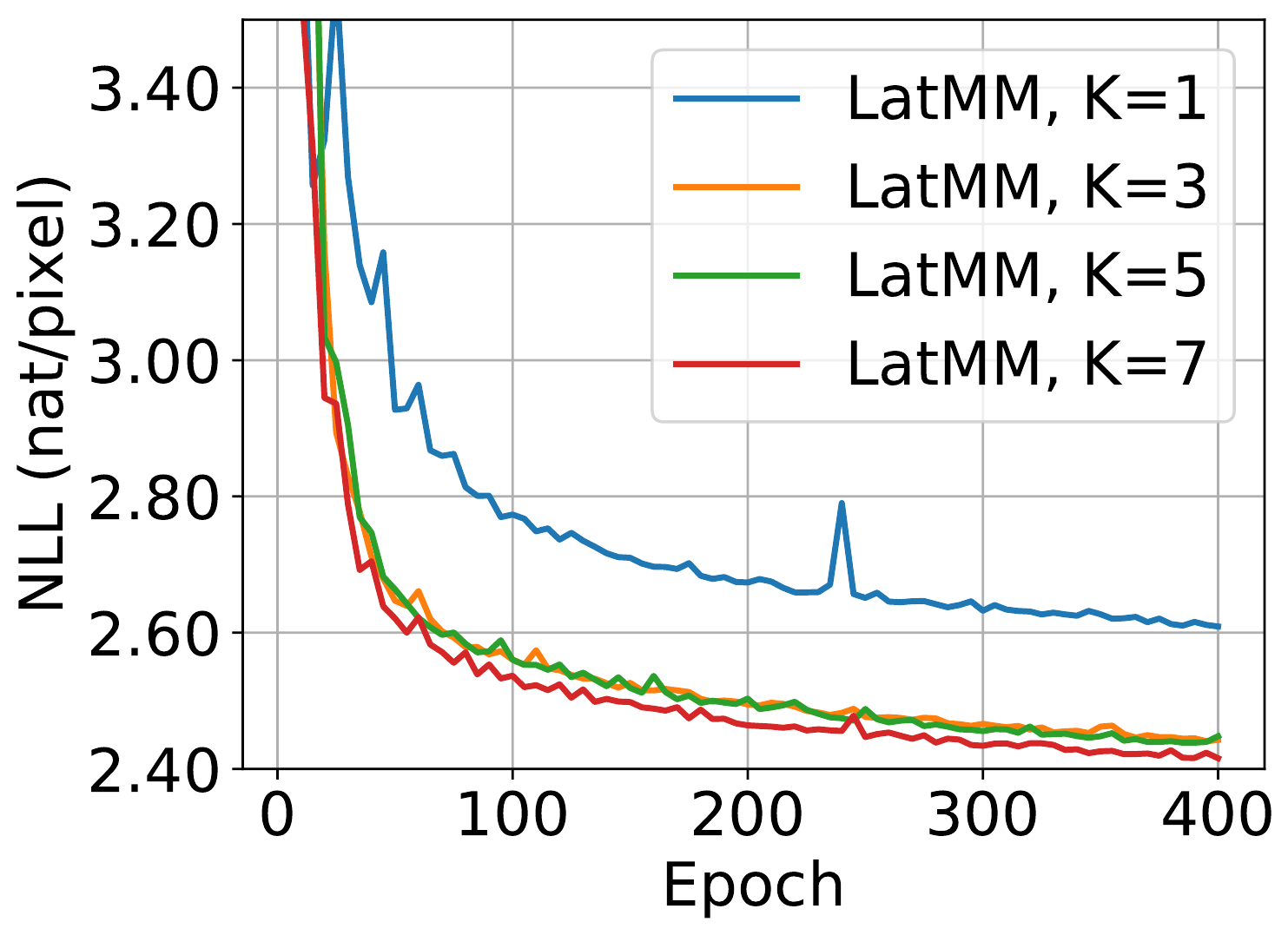}
    \vspace{-0.6cm}
    \caption{Dataset Fashion-MNIST}
    \label{fig-latmm-fsh-nll-curve}
  \end{subfigure}
  \vspace{-0.3cm}
  \caption{NLL (Unit: nat/pixel) of LatMM versus training epochs with different number of mixture component $k$. (a) $10000$ images from MNIST is used for training, (b) $10000$ images from Fashion-MNIST is used for training.}
  \label{fig:latmm-nll}
  \vspace{0.3cm}
\end{figure*}
\subsection{Experimental setup}\label{sub:exp-setup}

We use the flow-based neural network for implementing generators $\{\bm{g}_k\}_{k=1}^{K}$ in GenMM and $\bm{g}$ in LatMM. Specifically, we use the Glow structure \cite{2018arXiv180703039K} that is developed based on RealNVP \cite{2016arXiv160508803D} and NICE \cite{DBLP:journals/corr/DinhKB14}. As introduced in \autoref{subsec-flow-intro}, the operation in \autoref{eq-gl} is a coupling
layer. Since only a part of the input is mapped non-linearly after a coupling
layer and the rest part remains the same, permutation \cite{2016arXiv160508803D} or $1\times 1$ convolution operation \cite{2018arXiv180703039K} is used to alternate the part of signal that goes through identity mapping. In Glow structure, a basic \textit{flow step} is the concatenation of three layers: Actnorm (element-wise affine mapping)
$\rightarrow$ $1\times 1$ Convolution (for permutation purpose)
$\rightarrow$ Coupling layer. A \textit{flow block} consists of: a squeeze layer,
several flow steps, a split layer. A squeeze layer reshapes
signal. A split layer allows
flow model to split some elements of hidden layers out and model them
directly as standard Gaussian, which relieves computation burden. In our experiments, there are also split layers that make dimension of $\bm{z}$ one fourth of dimension $\bm{x}$, and split signal in hidden layers are modeled as standard Gaussian. 

All generators used in our experiments are randomly initialized before training. 
In addition, the prior distribution $\bm{\pi}$ update in both GenMM and LatMM is every $5$ epochs, {i.e.},  $t_{\pi} = 5$. For the training of LatMM, we adopt the Gamma distribution $\Gamma(\bm{\sigma}_k^{-1}; a, b)$ as the parameter prior for $\bm{\sigma}_k^{-1}, \forall k$, with shape parameter $a=2$ and rate parameter $b = 1$.
Our models are implemented using Pytorch and experiments are carried out on Tesla P100 GPU. Code is available at github repository\footnote{https://github.com/FirstHandScientist/EM-GM}.

\begin{figure*}[!tp]
  \captionsetup[subfigure]{justification=centering}
  \centering
  \begin{subfigure}{.24\textwidth}
    \centering
    \includegraphics[width=1\linewidth]{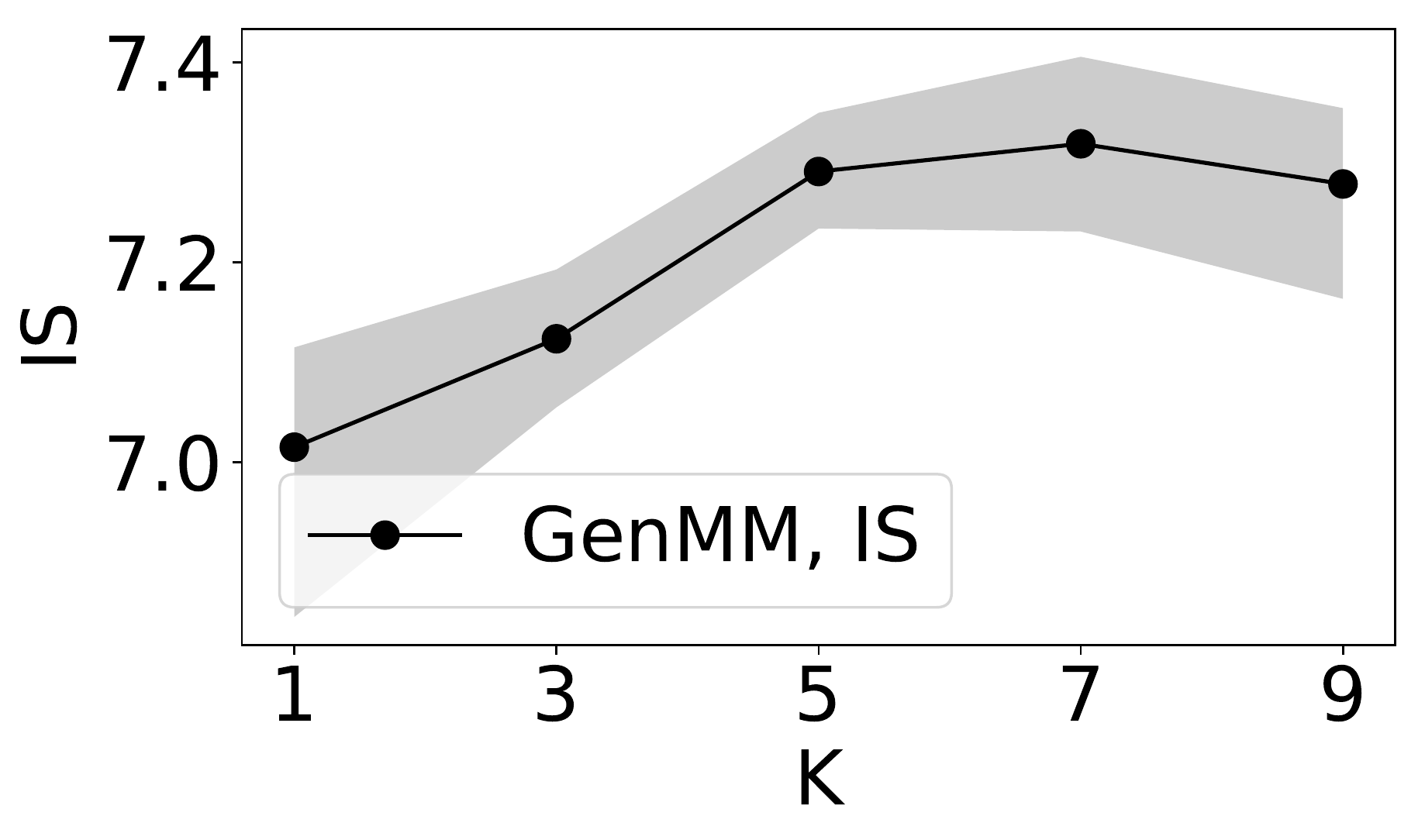}
  \end{subfigure}
  \vspace{-2pt}
  \begin{subfigure}{.24\textwidth}
    \centering
    \includegraphics[width=1\linewidth]{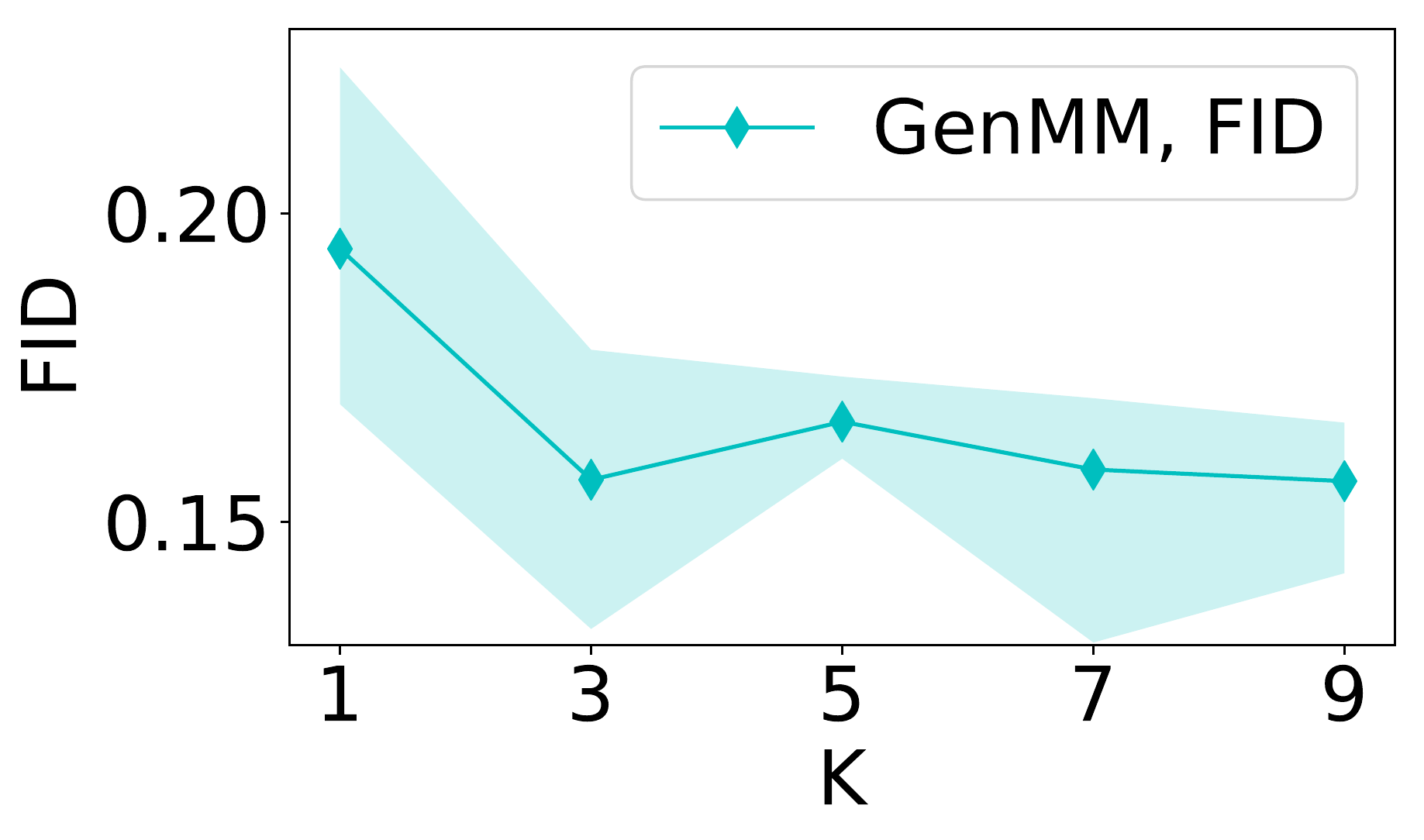}
  \end{subfigure}
  \centering
  \begin{subfigure}{.24\textwidth}
    \centering
    \includegraphics[width=1\linewidth]{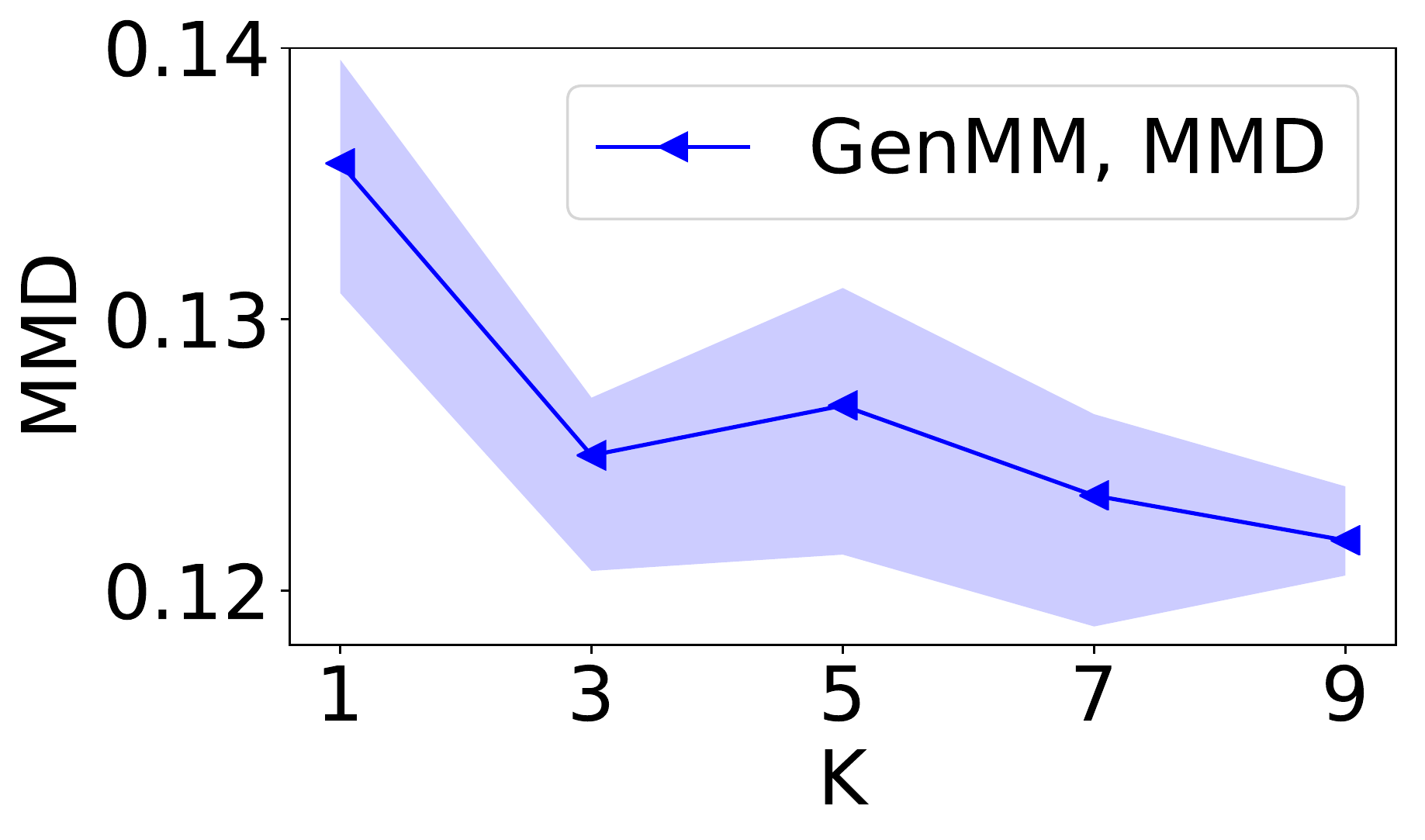}
  \end{subfigure}
  \centering
  \begin{subfigure}{0.24\textwidth}
    \centering
    \includegraphics[width=1\linewidth]{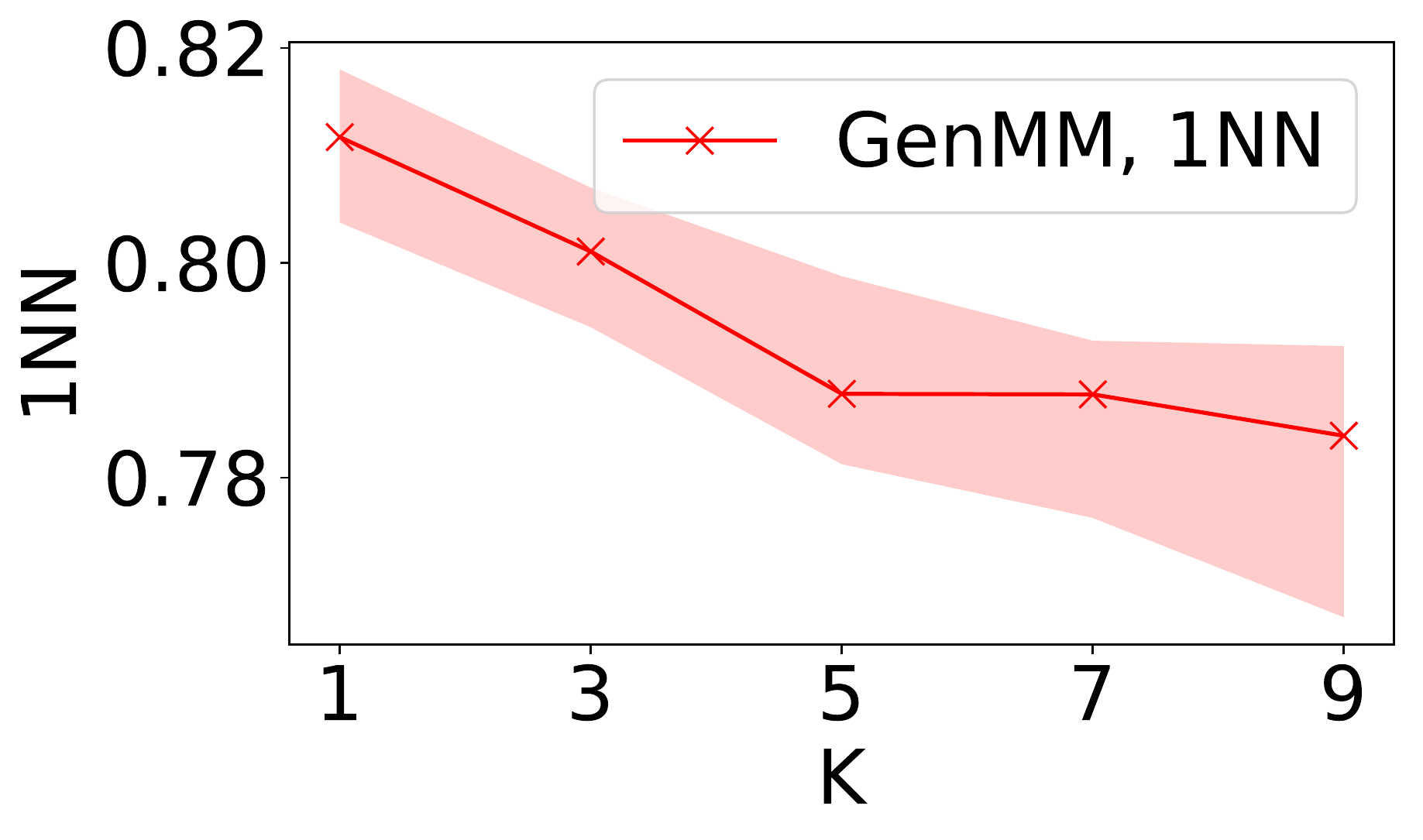}
  \end{subfigure}
  \centering
  \begin{subfigure}{.24\textwidth}
    \centering
    \includegraphics[width=1\linewidth]{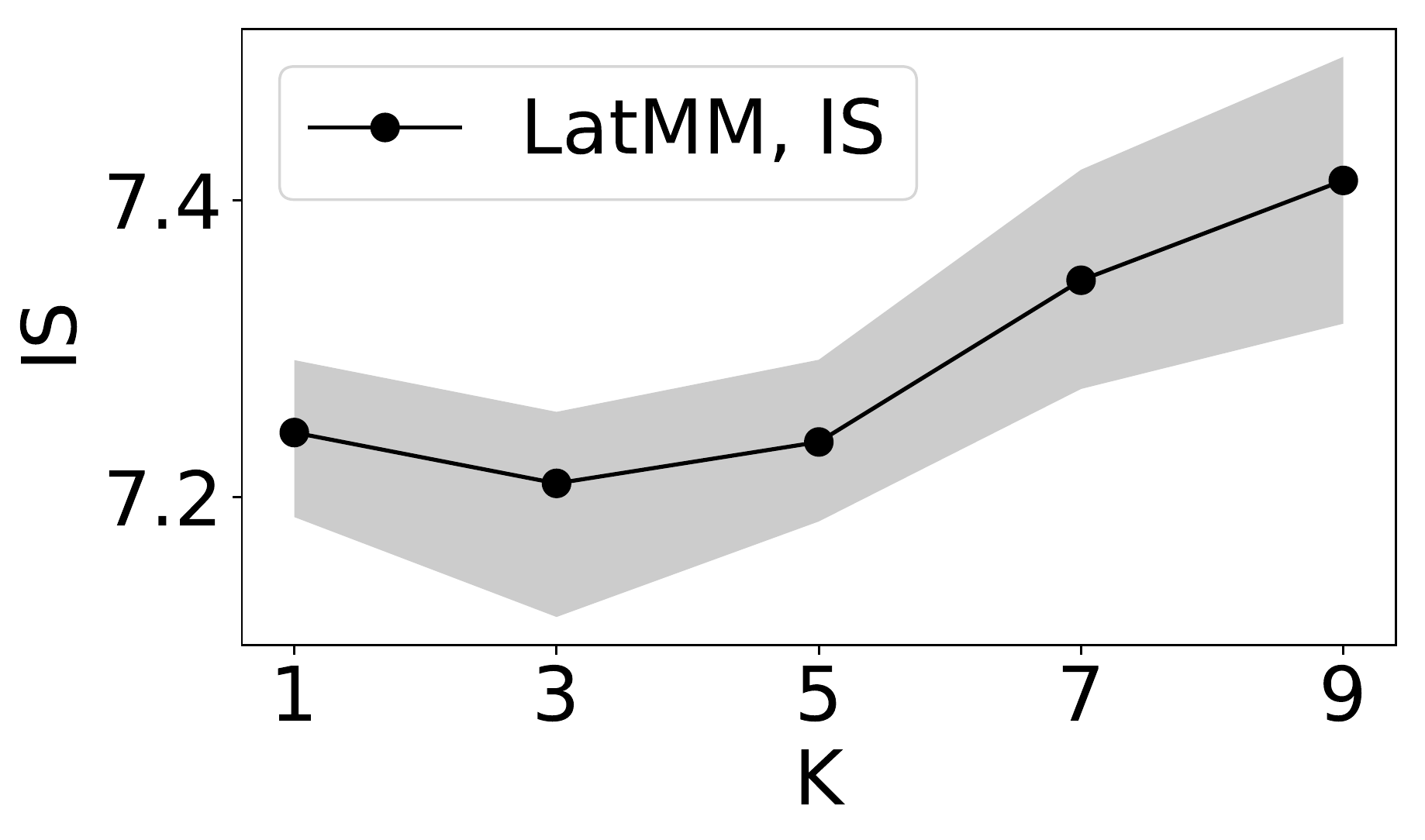}
  \end{subfigure}
  \centering
  \begin{subfigure}{.24\textwidth}
    \centering
    \includegraphics[width=1\linewidth]{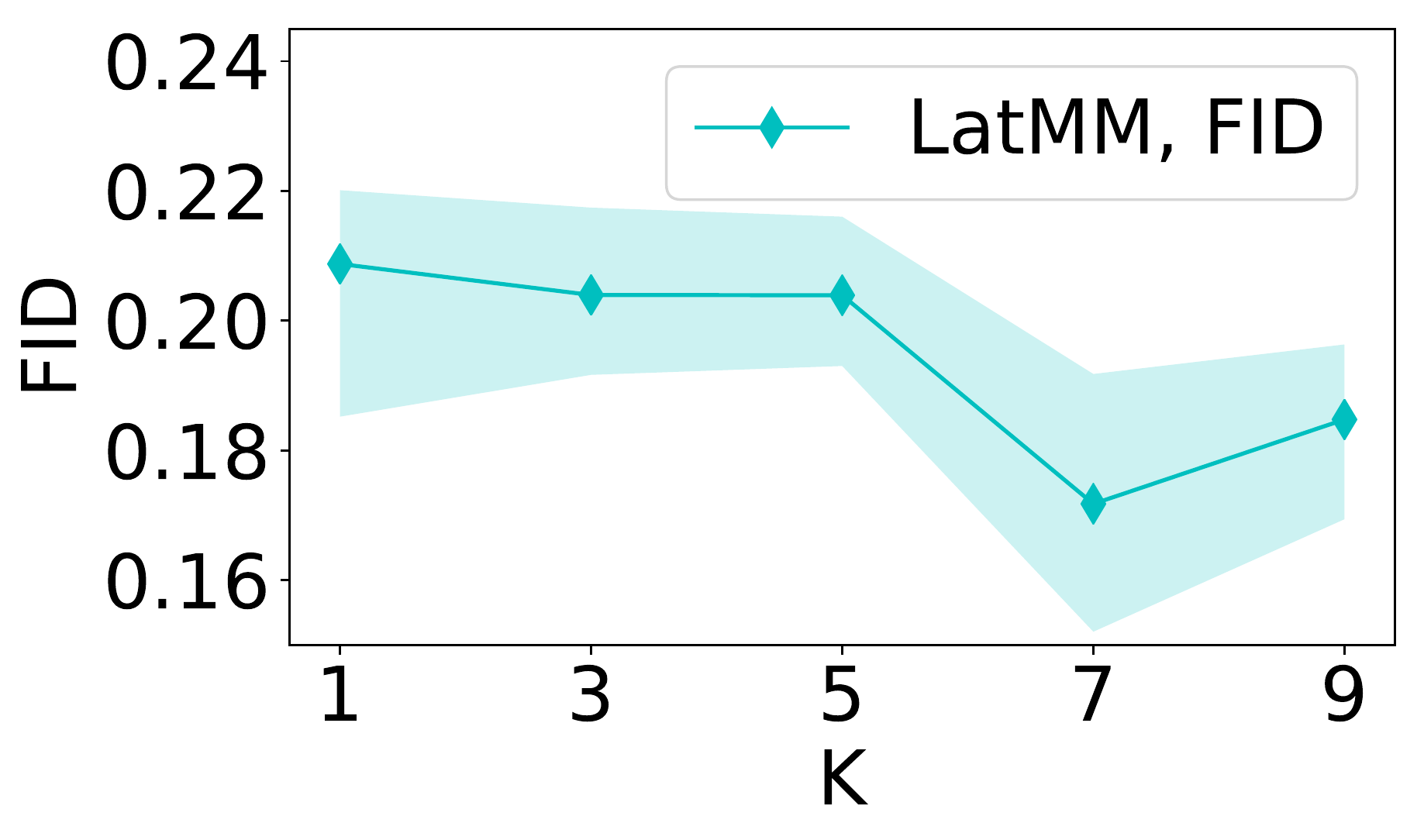}
  \end{subfigure}
  \centering
  \begin{subfigure}{.24\textwidth}
    \centering
    \includegraphics[width=1\linewidth]{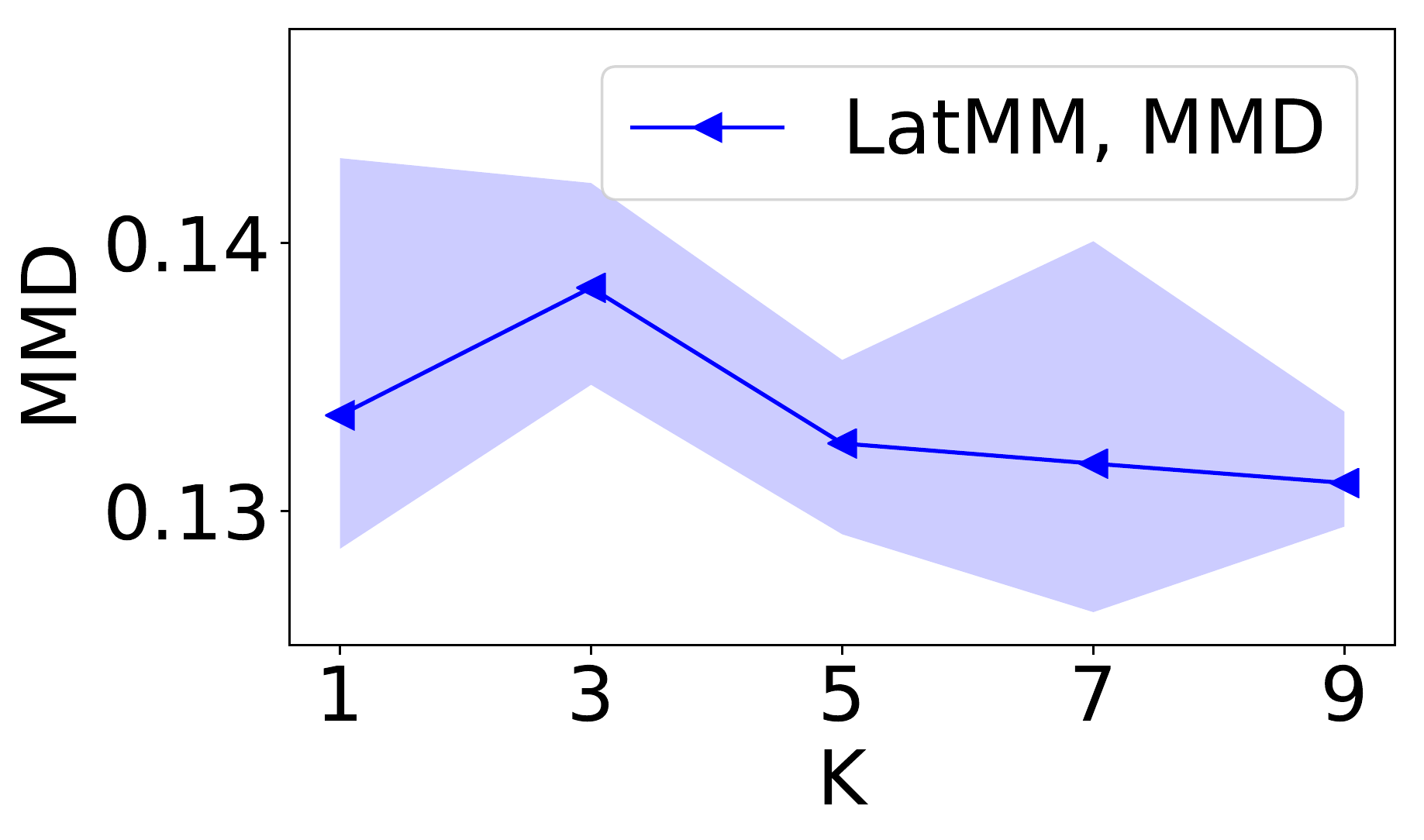}
  \end{subfigure}
  \begin{subfigure}{0.24\textwidth}
    \centering
    \includegraphics[width=1.\linewidth]{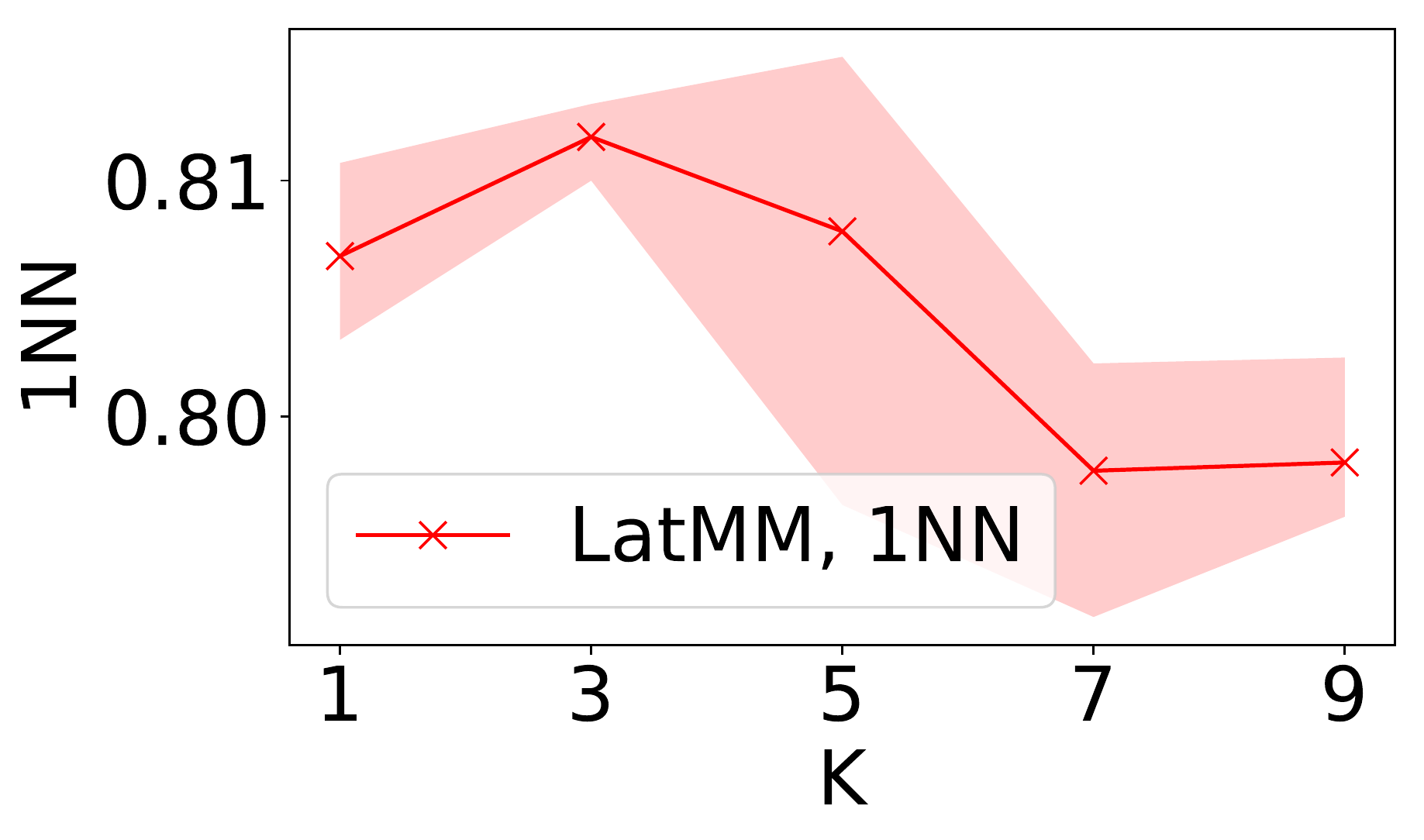}
  \end{subfigure}
  \vspace{-0.35cm}
  \caption{IS, FID, MMD and 1NN of GenMM and LatMM for MNIST dataset. GenMM and LatMM are trained on $60000$ images of MNIST. The results are evaluated on $2000$ samples per simulation point ($1000$ samples generated by GenMM or LatMM for corresponding $K$, $1000$ samples from MNIST). $5$ experiments are carried out for each assessed score at each setting of $K$. Curve with marker denotes mean score and shaded area denotes the range of corresponding score.}\label{fig-scores-k}
  \vspace{-0.15cm}
\end{figure*}

\begin{figure*}[!ht]
  \captionsetup[subfigure]{justification=centering}
  \centering
  \begin{subfigure}{.24\textwidth}
    \centering
    \includegraphics[width=1.0\linewidth]{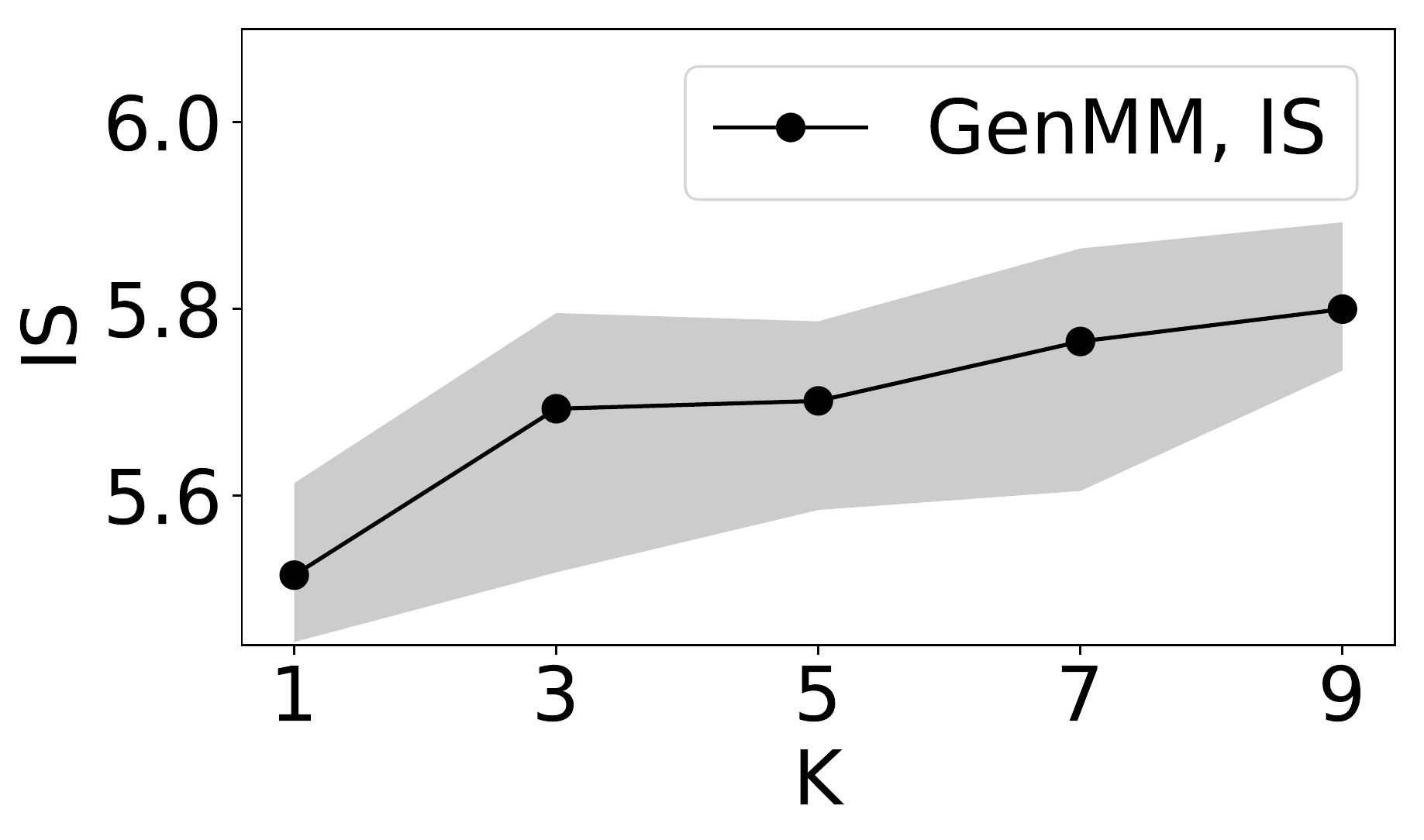}
  \end{subfigure}
  \vspace{-2pt}
  \begin{subfigure}{.24\textwidth}
    \centering
    \includegraphics[width=1.0\linewidth]{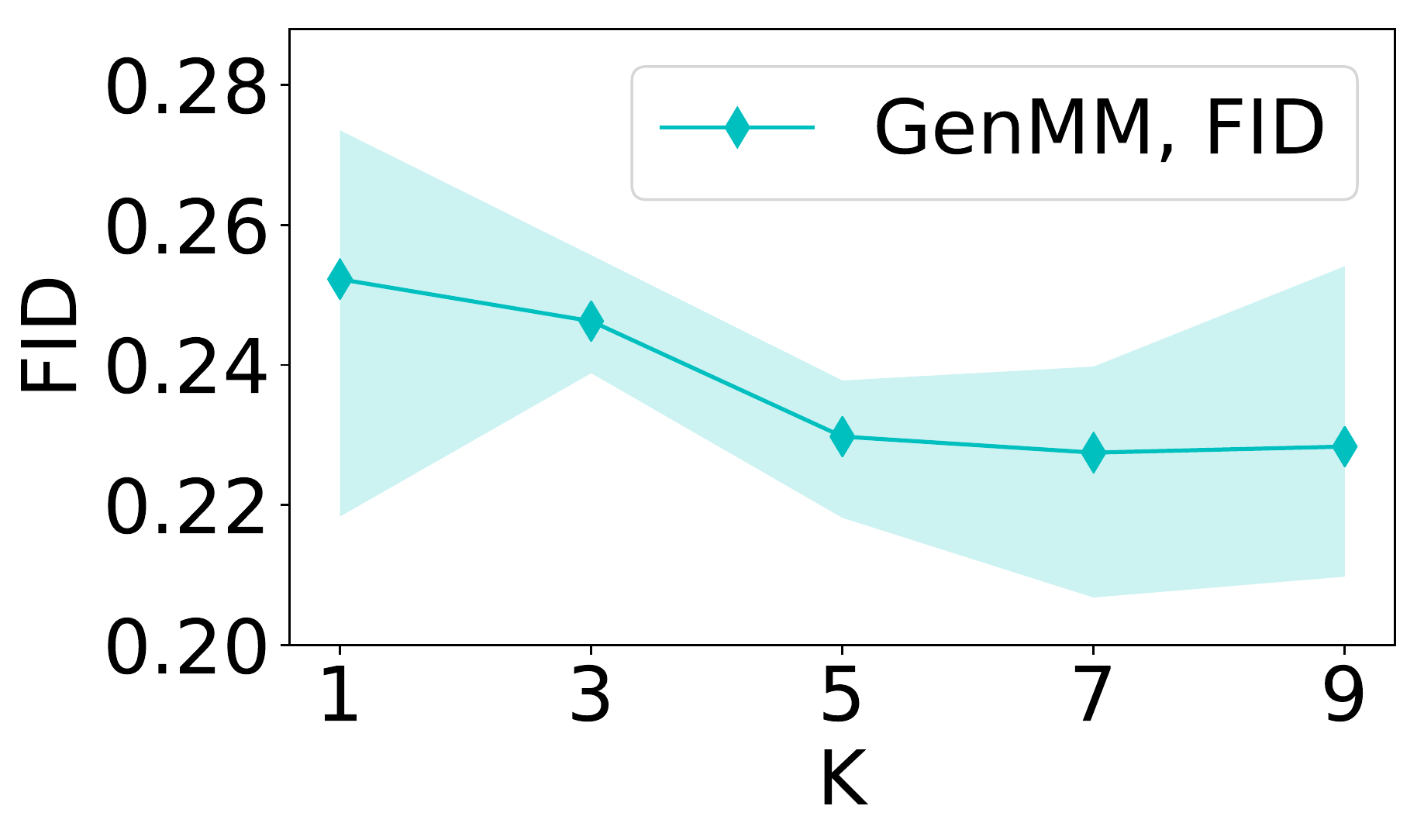}
  \end{subfigure}
  \centering
  \begin{subfigure}{.24\textwidth}
    \centering
    \includegraphics[width=1\linewidth]{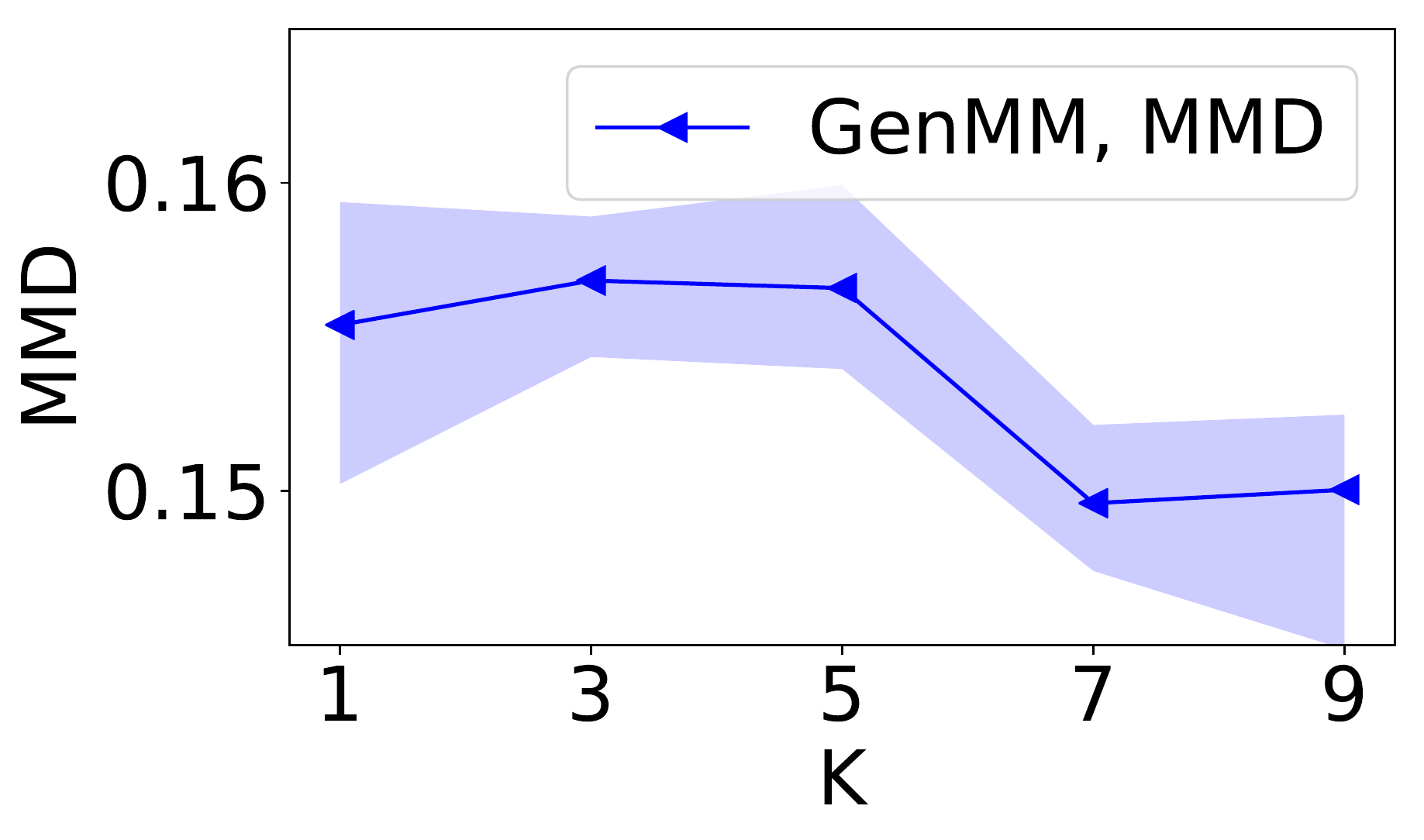}
  \end{subfigure}
  \centering
  \begin{subfigure}{0.24\textwidth}
    \centering
    \includegraphics[width=1\linewidth]{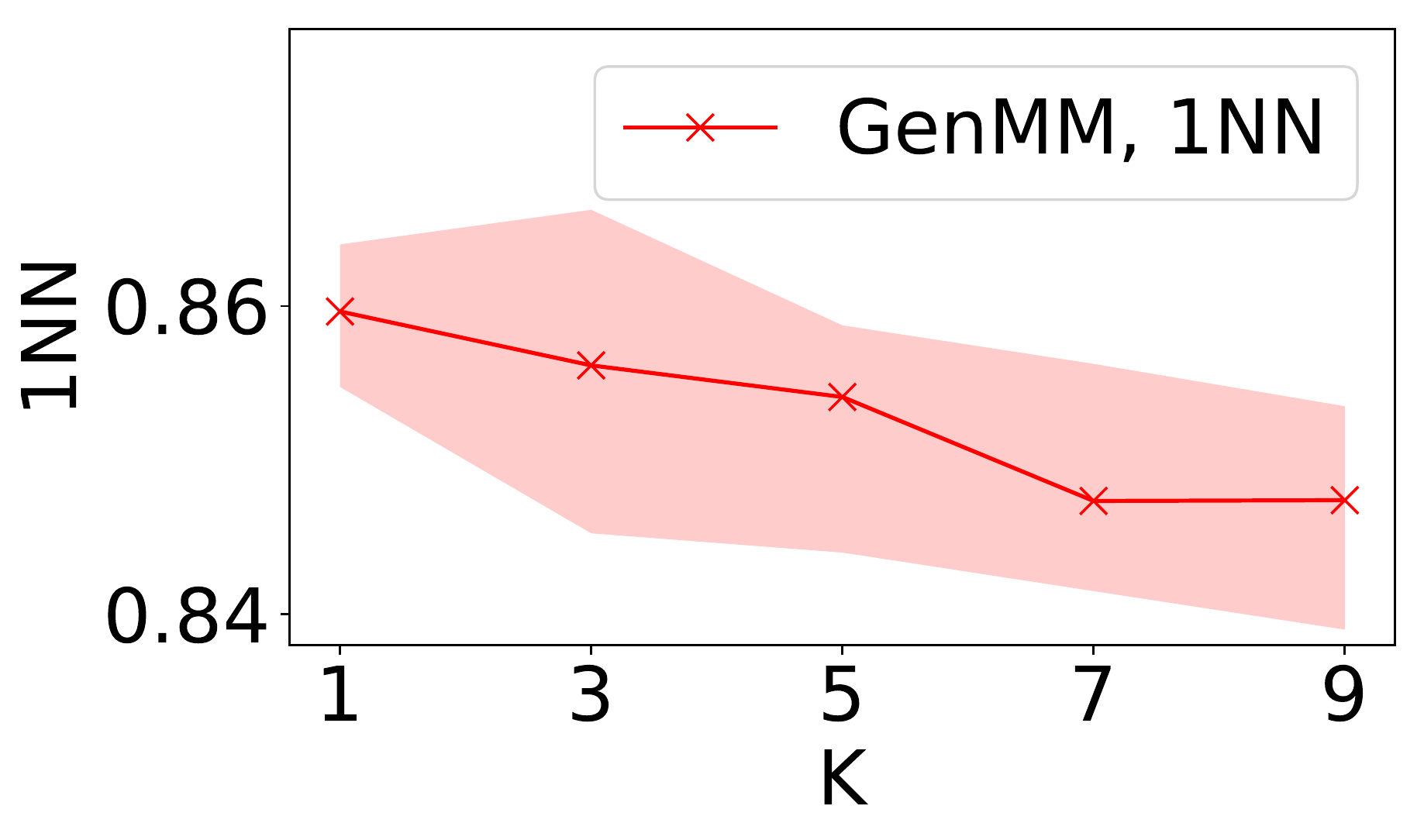}
  \end{subfigure}
  \centering
  \begin{subfigure}{.24\textwidth}
    \centering
    \includegraphics[width=1\linewidth]{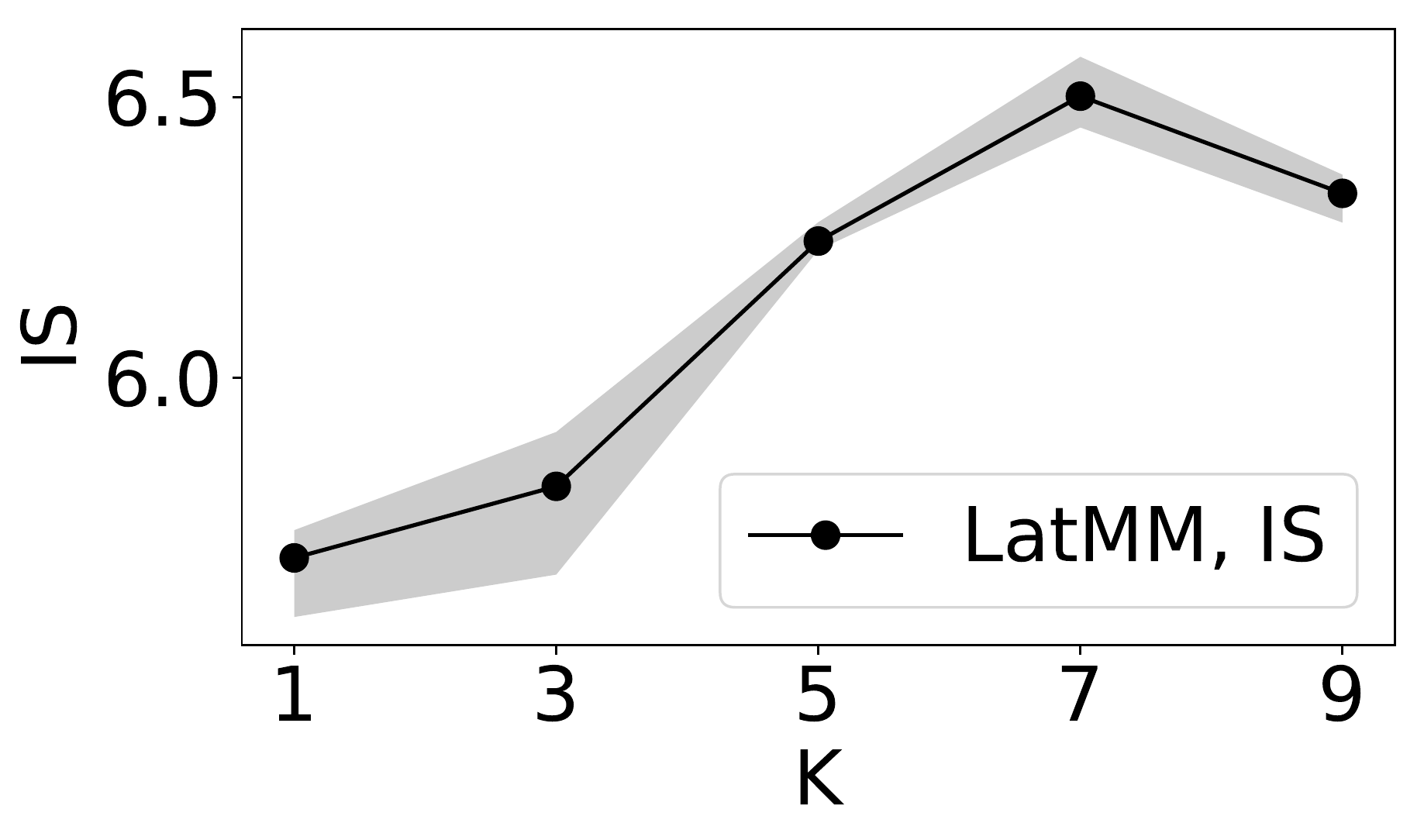}
  \end{subfigure}
  \centering
  \begin{subfigure}{.24\textwidth}
    \centering
    \includegraphics[width=1\linewidth]{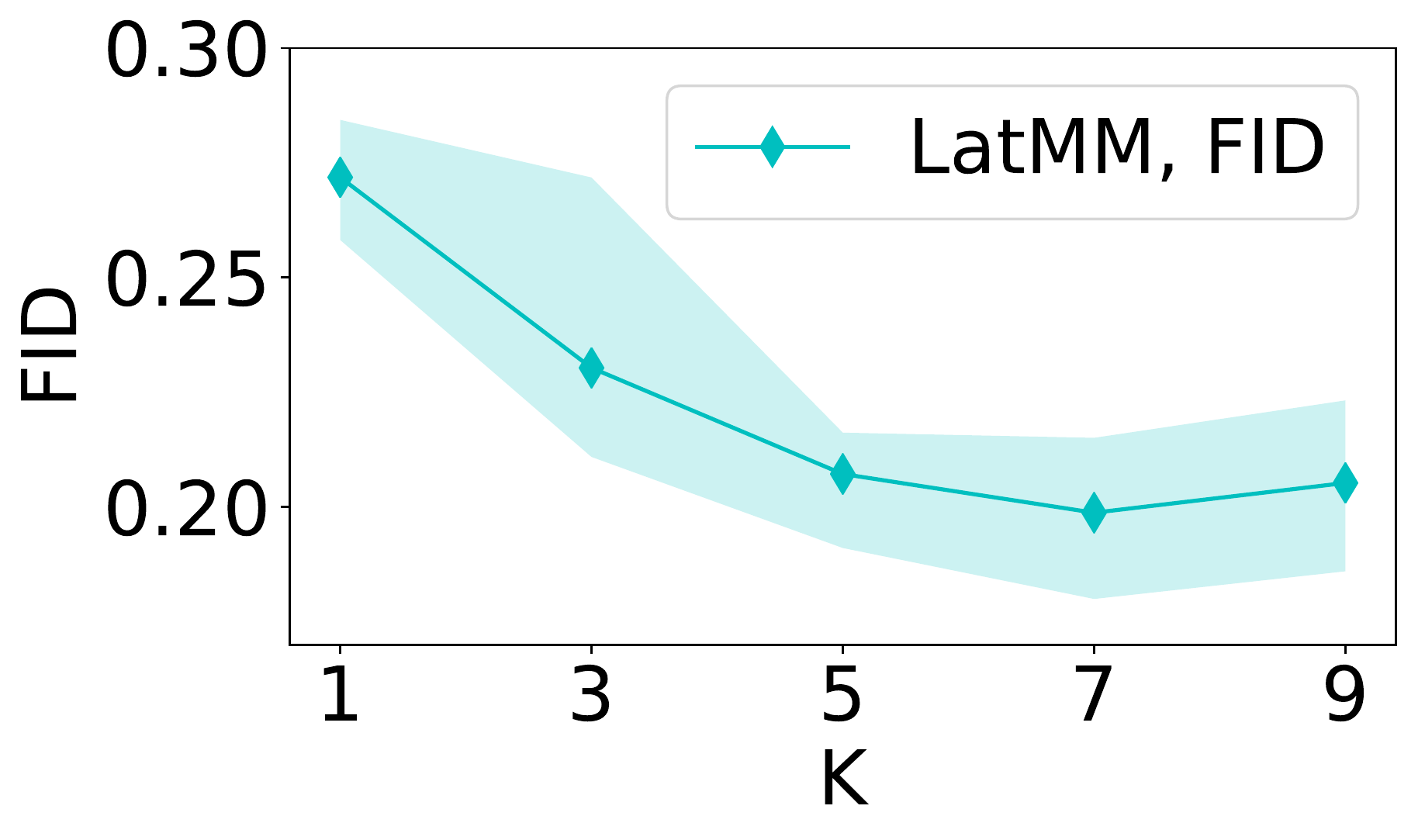}
  \end{subfigure}
  \centering
  \begin{subfigure}{.24\textwidth}
    \centering
    \includegraphics[width=1\linewidth]{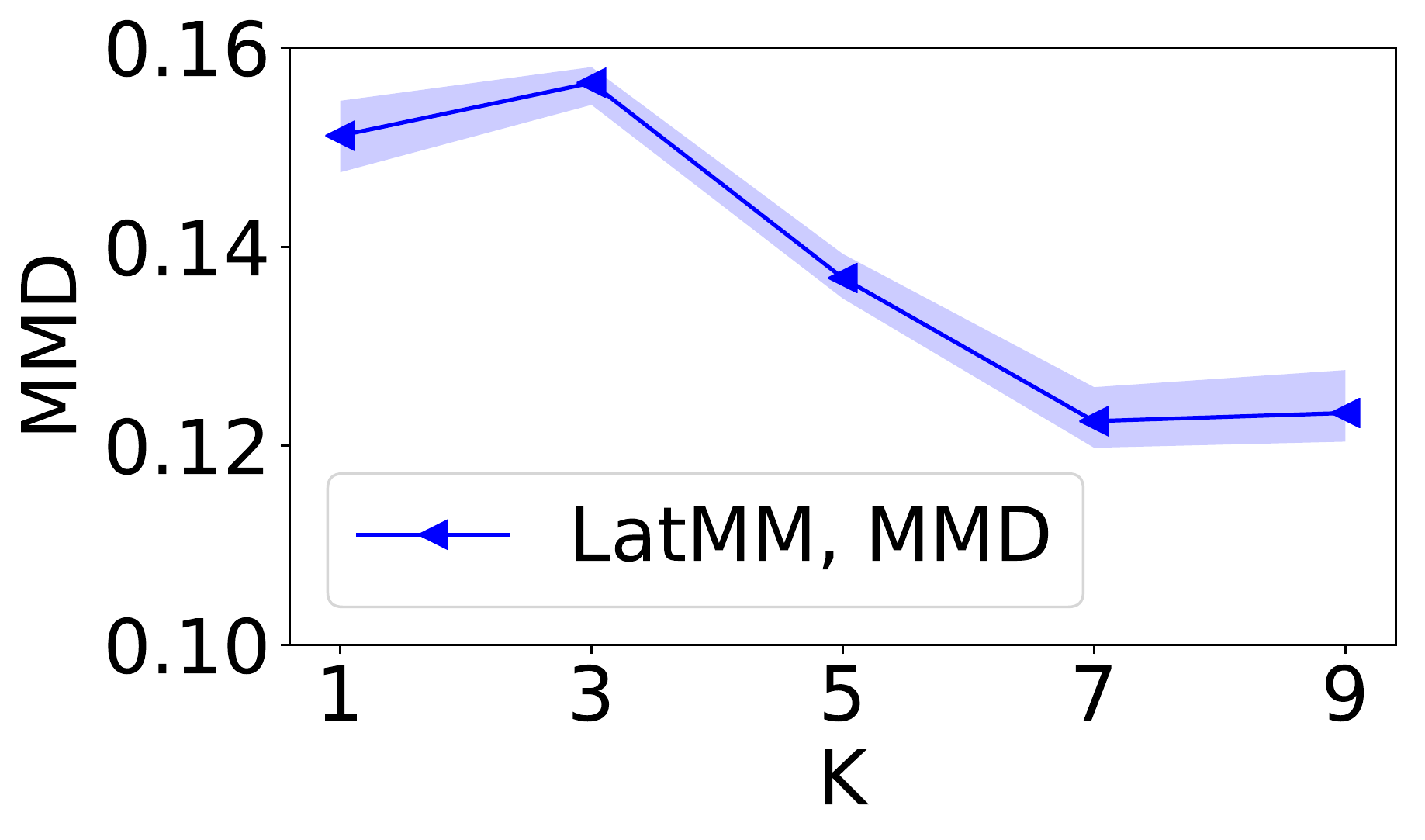}
  \end{subfigure}
  \begin{subfigure}{0.24\textwidth}
    \centering
    \includegraphics[width=1.\linewidth]{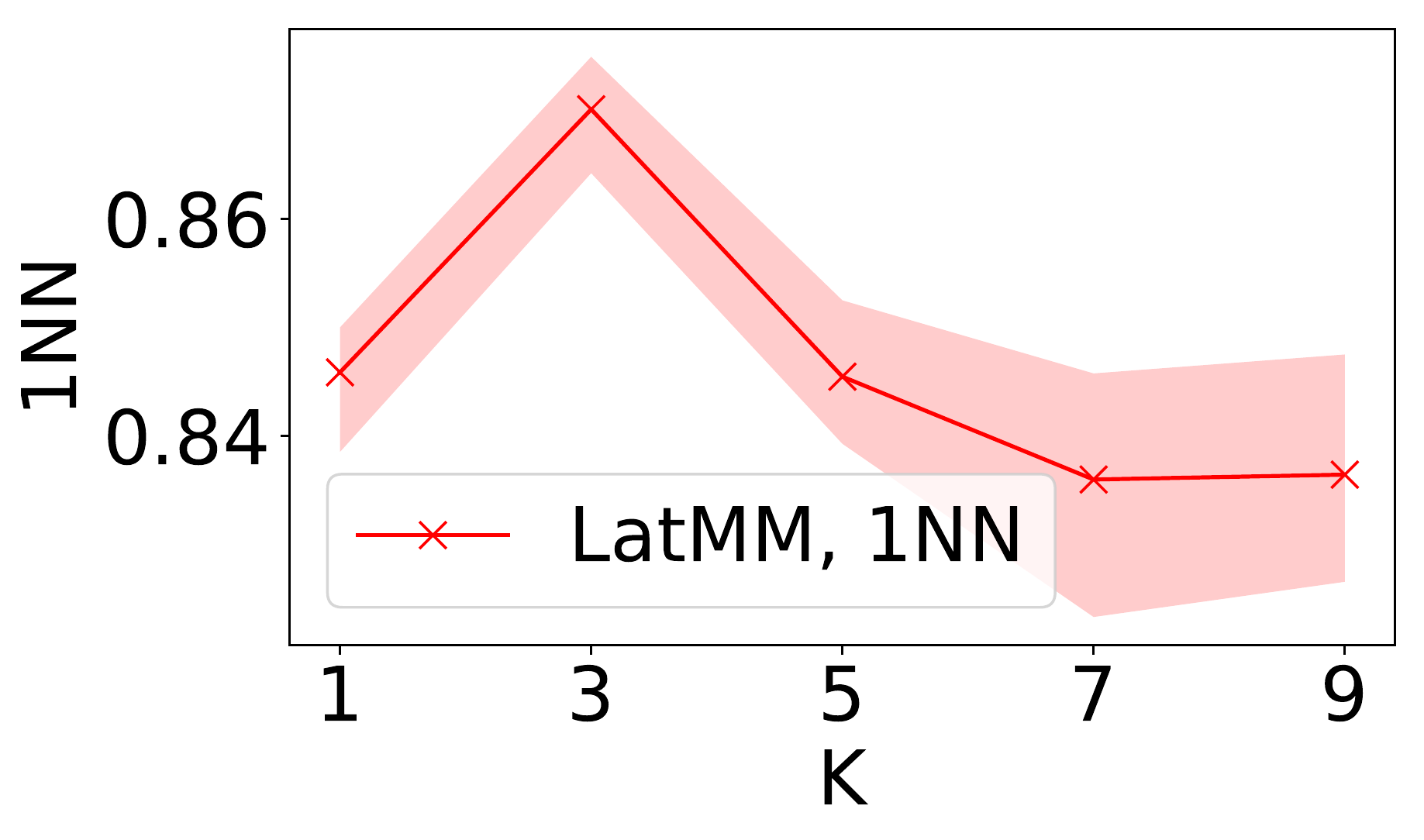}
  \end{subfigure}
  \vspace{-0.35cm}
  \caption{IS, FID, MMD and 1NN of GenMM and LatMM for
    Fashion-MNIST dataset. GenMM and LatMM are trained on $60000$ images of Fashion-MNIST. The results are evaluated on $2000$ samples
    per simulation point ($1000$ samples generated by GenMM or LatMM
    for corresponding $K$, $1000$ samples from Fashion-MNIST). $5$
    experiments are carried out for each assessed score at each
    setting of $K$. Curve with marker denotes mean score and shaded
    area denotes the range of corresponding
    score.}\label{fig-scores-k-FashionMNIST}
  \vspace{0.2cm}
\end{figure*}
\subsection{Evaluation of Proposed Models}
In order to see if the proposed algorithms of GenMM and LatMM help to improve probability distribution modeling capacity, we assess our proposed algorithms with varying number of mixtures ($K$). Since our models are explicit models, the negative log likelihood (NLL) is used for comparison of our models. Apart from NLL, another four different metrics are used in assessment of models.
The metrics are Inception Score (IS) \cite{NIPS2016_6125,2018arXiv180101973B,2018arXiv180607755X}, Frechet
Inception Distance (FID) \cite{2017arXiv170608500H}, Maximum Mean
Discrepancy (MMD) \cite{2018arXiv180607755X} and two-sample test based 1-Nearest
Neighbor (1NN) score \cite{2016arXiv161006545L}. IS measures statistically if a given sample can be recognized by a classifier with high confidence. A high IS stands for high quality for generated samples. FID measures a divergence between two distributions under testing by assuming these two distribution are both Gaussian. We also use MMD with Gaussian kernel to test how dissimilar two distributions are.
Small values of FID and MMD mean that the mixture distribution model
is close to the underlying distribution of dataset. 1NN score measures
if two given distributions are empirically close by computing 1NN accuracy
on samples from two distributions under testing. The closer 1NN score is to $0.5$, the more likely 
two distributions under testing are the same. Therefore, a high IS is good, low FID and MMD scores, and 1NN score close to 0.5 are good. We use the evaluation
framework of \cite{2018arXiv180607755X} to compute these metrics scores, where
we train a ResNet on datasets MNIST and Fashion-MNIST, respectively, as the feature extractor for evaluation of the four performance metrics.

\begin{table}
  \caption{The lowest NLL value of GenMM for curves in \autoref{fig:genmm-nll} (nat/pixel).}
  \label{tab:lowestNLLgenMM}
\begin{tabular}{l|c|c|c|c}
\toprule
    {Dataset} & K=1 &  K=3 &  K=5 &  K=7 \\                                         
\midrule                                                                                          MNIST &     1.8929 &    1.8797 &    1.8719 &    1.8579 \\
    FashionMNIST &   2.3571 &   2.3429 &   2.3353 &   2.3323 \\ 
  \end{tabular} 
\end{table}
The NLL curves of GenMM and LatMM models during model training phase are shown in \autoref{fig:genmm-nll} and \autoref{fig:latmm-nll}, respectively. Subsets of MNIST and Fashion-MNIST are used to train our mixture models in order to assess their performance w.r.t. NLL when different number of mixture components $K$ is used. All the curves in \autoref{fig:genmm-nll} and \autoref{fig:latmm-nll} show that NLL decreases as training epoch number increases in general. There is fluctuation of these decreasing NLL curves due to: (a) the iteration of E-step and M-step of EM, and (b) the use of batch-size gradient in optimization at M-step. In each figure of \autoref{fig:genmm-nll} and \autoref{fig:latmm-nll}, NLL curve corresponding to larger total number of mixture components, $K$, reaches smaller NLL value after traning for same number of epochs. The results are consistent since as $K$ increases, both GenMM and LatMM have smaller NLL. These results confirm our hypothesis that mixture models fit real data better. The lowest NLL values of curves in \autoref{fig:genmm-nll} in training GenMM models are reported in \autoref{tab:lowestNLLgenMM}.

\begin{figure*}[!ht]
  \captionsetup[subfigure]{justification=centering}
  \centering
  \begin{subfigure}[b]{0.24\textwidth}
    \centering
    \includegraphics[width=1\linewidth]{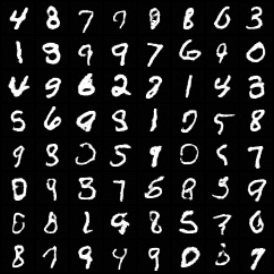}
    \caption{Generated Samples. (GenMM, K=7)}
  \end{subfigure}
  \centering
  \begin{subfigure}[b]{0.24\textwidth}
    \centering
    \includegraphics[width=1\linewidth]{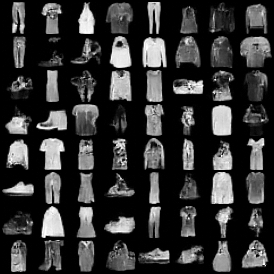}
    \caption{Generated samples. \\(GenMM, K=3)}
  \end{subfigure}
  \begin{subfigure}[b]{0.24\textwidth}
    \centering
    \includegraphics[width=1\linewidth]{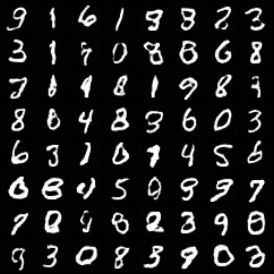}
    \caption{Generated samples. \\(LatMM, K=3)}
  \end{subfigure}
  \begin{subfigure}[b]{0.24\textwidth}
    \centering
    \includegraphics[width=1\linewidth]{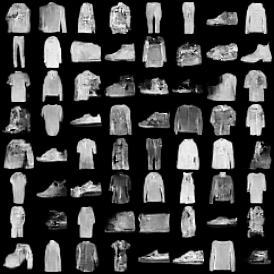}
    \caption{Generated samples. \\(LatMM, K=7)}
  \end{subfigure}
  \vspace{-0.3cm}
  \caption{Generated samples by GenMM and LatMM for MNIST and Fashion-MNIST datasets.}\label{fig-demo-samples}
\end{figure*}

\begin{figure*}[!ht]
  \centering
  \captionsetup[subfigure]{justification=centering}
  \begin{subfigure}[b]{0.3\textwidth}
    \centering
    \includegraphics[width=1\linewidth]{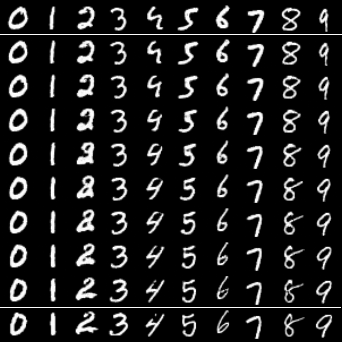}
    \caption{Interpolation by GenMM, K=7. Identity of $\bm{g}_k$ is chosen by $\argmax_{k}\; \gamma_k$.}\label{fig-interpo-genmm1}
  \end{subfigure}
  \hspace{10pt}
  \begin{subfigure}[b]{0.3\textwidth}
    \centering
    \includegraphics[width=1\linewidth]{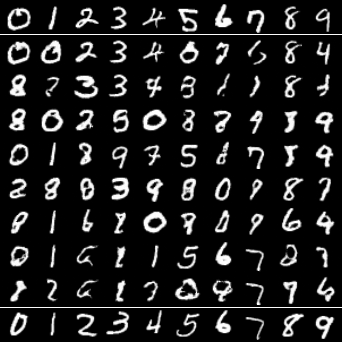}
    \caption{Interpolation by GenMM, K=7. Identity of $\bm{g}_k$ is randomly chosen.\\~ }\label{fig-interpo-genmm2}
  \end{subfigure}
  \hspace{10pt}
  \begin{subfigure}[b]{0.3\textwidth}
    \centering
    \includegraphics[width=1\linewidth]{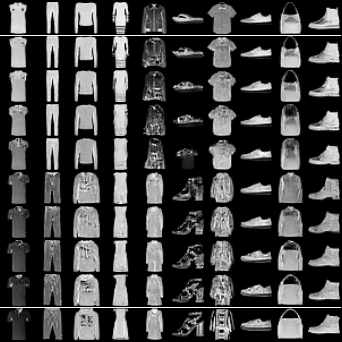}
    \caption{Interpolation by GenMM, K=9. Identity of $\bm{g}_k$ is chosen by $\argmax_{k}\; \gamma_k$.}\label{fig-interpo-genmm3}
  \end{subfigure}
  \vspace{0.22cm}
  \centering
  \captionsetup[subfigure]{justification=centering}
  \begin{subfigure}[b]{0.3\textwidth}
    \centering
    \includegraphics[width=1\linewidth]{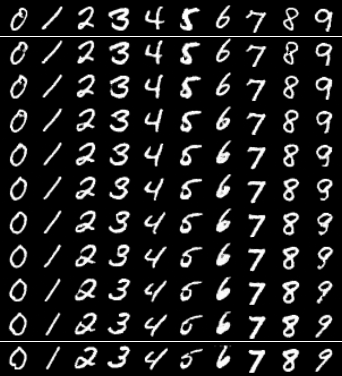}
    \caption{Interpolation by LatMM, K=9.}\label{fig-interpo-latmm1}
  \end{subfigure}
  \hspace{10pt}
  \begin{subfigure}[b]{0.3\textwidth}
    \centering
    \includegraphics[width=1\linewidth]{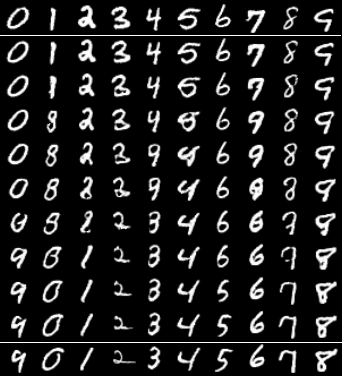}
    \caption{Interpolation by LatMM, K=9.}\label{fig-interpo-latmm2}
  \end{subfigure}
  \hspace{10pt}
  \begin{subfigure}[b]{0.3\textwidth}
    \centering
    \includegraphics[width=1\linewidth]{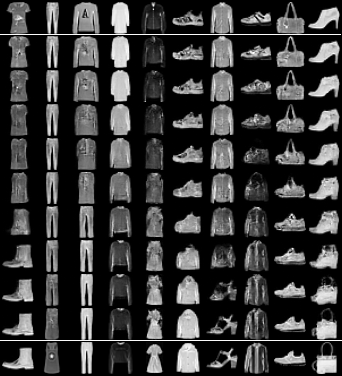}
    \caption{Interpolation by LatMM, K=9.}\label{fig-interpo-latmm3}
  \end{subfigure}\vspace{-0.5cm}
  \caption{Interpolation in latent space to generate samples . First
    and last rows are real samples from MNIST. For each row, images
    are generated by interpolating latent variables of empirical
    images in first and last rows.}\label{fig-interpo}
  \vspace{0.2cm}
  \label{fig-app-interpolation}
\end{figure*}

As for the scores of IS, FID, MMD, and 1NN, we increase $K$ for the proposed models and check
how the four metrics vary. We do several trials of evaluation and
report the results. The results are shown in
\autoref{fig-scores-k} for MNIST datset and
\autoref{fig-scores-k-FashionMNIST} for Fashio-MNIST dataset. Let us
first address the results in \autoref{fig-scores-k}. It can be
observed that IS increases with number of mixtures $K$. The IS
improvement shows a saturation and decreasing trend for GenMM when
$K=9$. The FID, MMD and 1NN scores show a decreasing trend with
increase in $K$. Their trends also saturate with increase in $K$. The
trends obey a statistical knowledge that performance improves with
increase in the model complexity, and then deteriorates if the model
complexity continues to increase. As that in \autoref{fig-scores-k}, similar trends are also observed in \autoref{fig-scores-k-FashionMNIST}. In some cases, performance for $K=3$ is poorer than $K=1$. We assume that the random initialization of parameters in mixture models has a high influence in this regard. 
Considering the trends in all the scores for both the figures, we can conclude that GenMM and LatMM can model the underlying distributions of data and the mixture models are good.

\subsection{Sample Generating and Interpolation}

\begin{table*}
  \caption{Test Accuracy Table of GenMM for Classification Task}\label{tab:acc-classification}
  \begin{tabular}{l|c|c|c|c|c|c|c} \toprule
    {Dataset} &  K=1 &  K=2 &  K=3 &  K=4 & K=10 & K=20 & State Of Art \\ \midrule
    Letter & 0.9459 &  0.9513 & 0.9578  & 0.9581 & 0.9657 & \textbf{0.9674} & {0.9582} \cite{tang2016extreme} \\ \midrule
    Satimage & 0.8900 & 0.8975 & 0.9045 & 0.9085 & 0.9105 & \textbf{0.9160} & 0.9090 \cite{jiang2013k-svd}   \\ \midrule
    Norb & 0.9184 & 0.9257 & 0.9406 & 0.9459 & 0.9538 & \textbf{0.9542} & 0.8920 \cite{pmlr-v5-salakhutdinov09a}  \\
  \end{tabular}
\end{table*}
\begin{figure*}[!ht]
  \captionsetup[subfigure]{justification=centering}
  \centering
  \begin{subfigure}{.33\textwidth}
    \centering
    \includegraphics[width=1\linewidth]{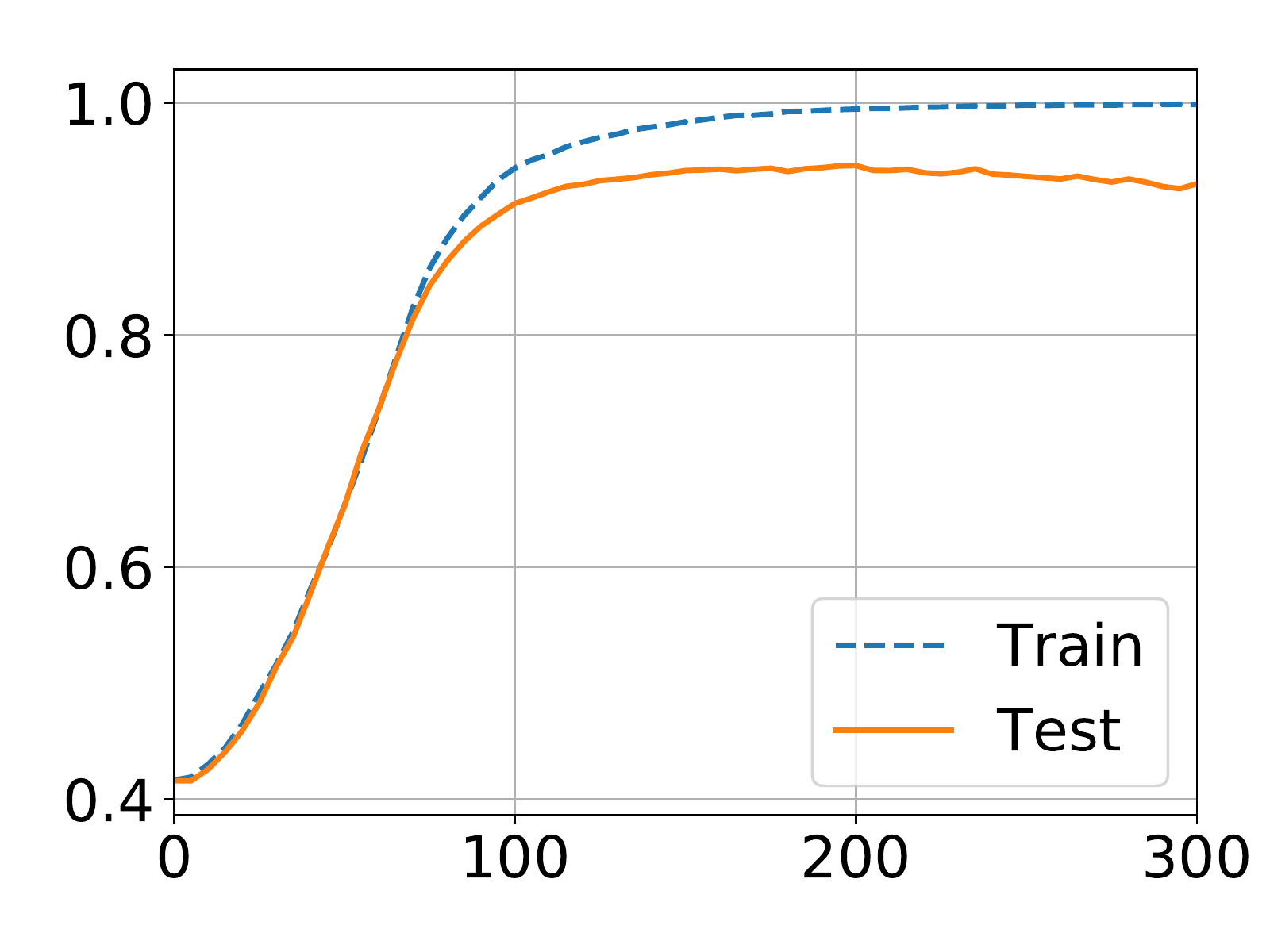}
    \vspace{-0.8cm}
    \caption{K=1}
  \end{subfigure}
  \begin{subfigure}{.33\textwidth}
    \centering
    \includegraphics[width=1\linewidth]{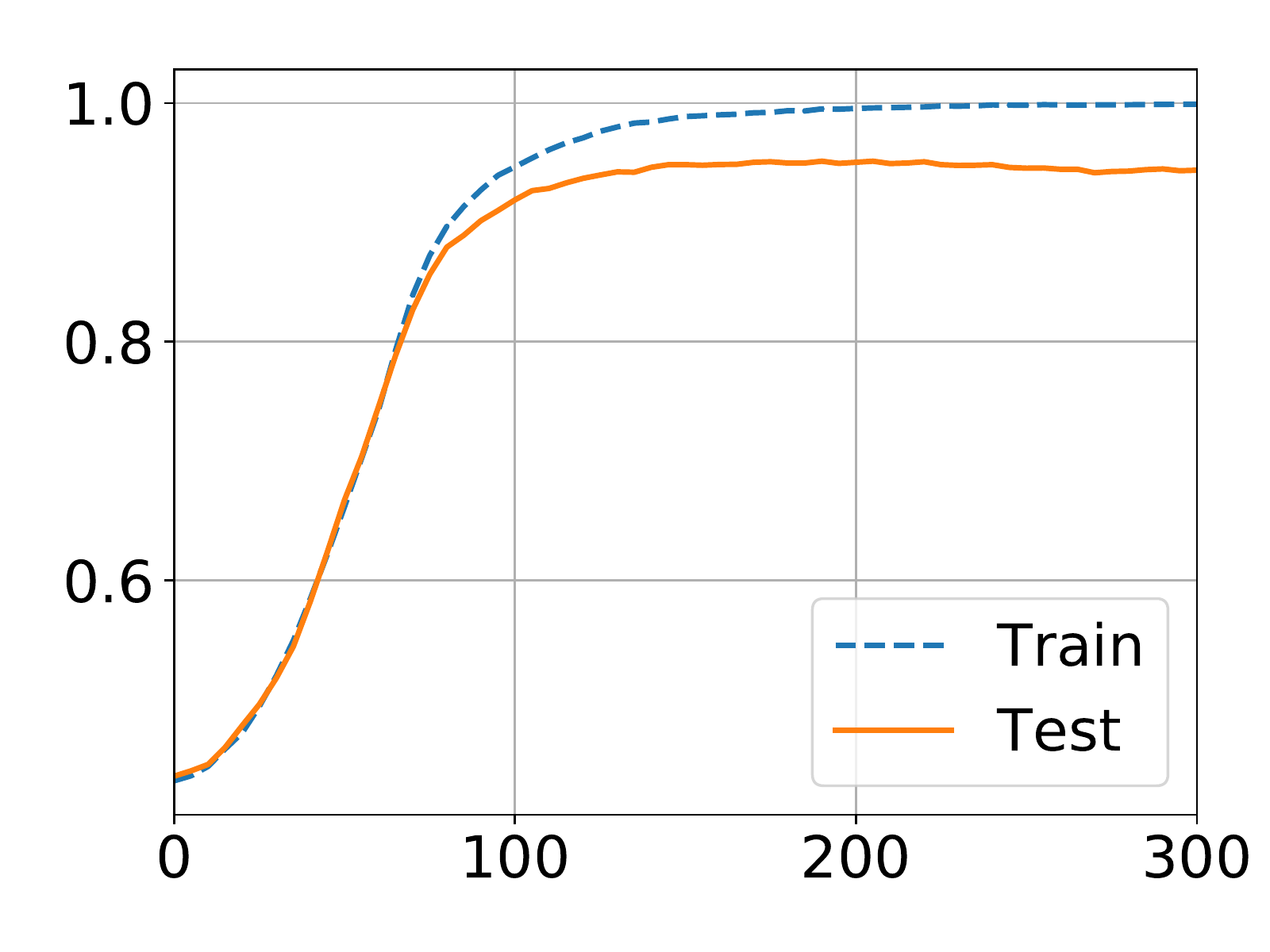}
    \vspace{-0.8cm}
    \caption{K=2}
  \end{subfigure}
  \centering
  \begin{subfigure}{.33\textwidth}
    \centering
    \includegraphics[width=1\linewidth]{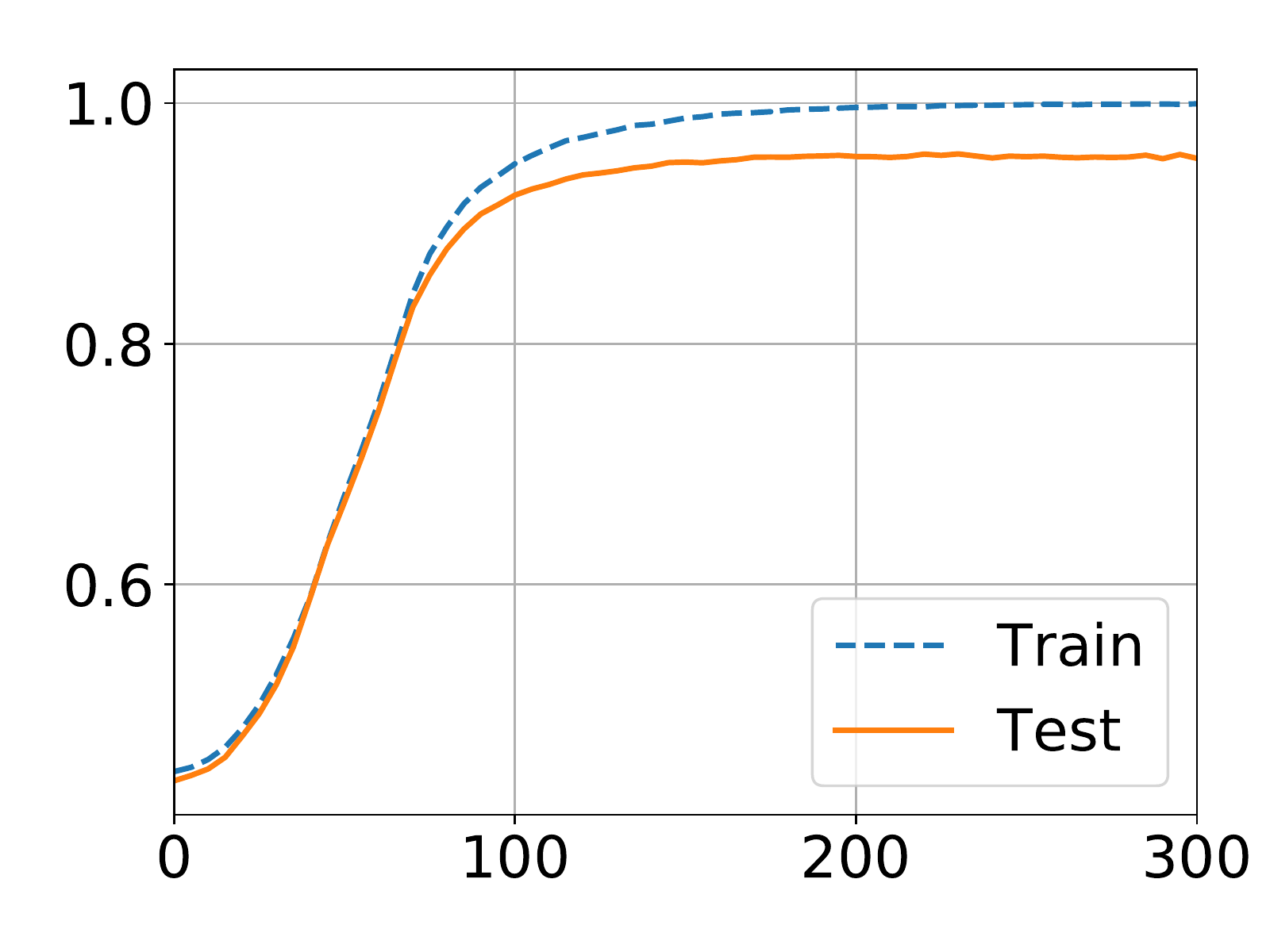}
    \vspace{-0.8cm}
    \caption{K=3}
  \end{subfigure}
  \centering
  \begin{subfigure}{0.33\textwidth}
    \centering
    \includegraphics[width=1\linewidth]{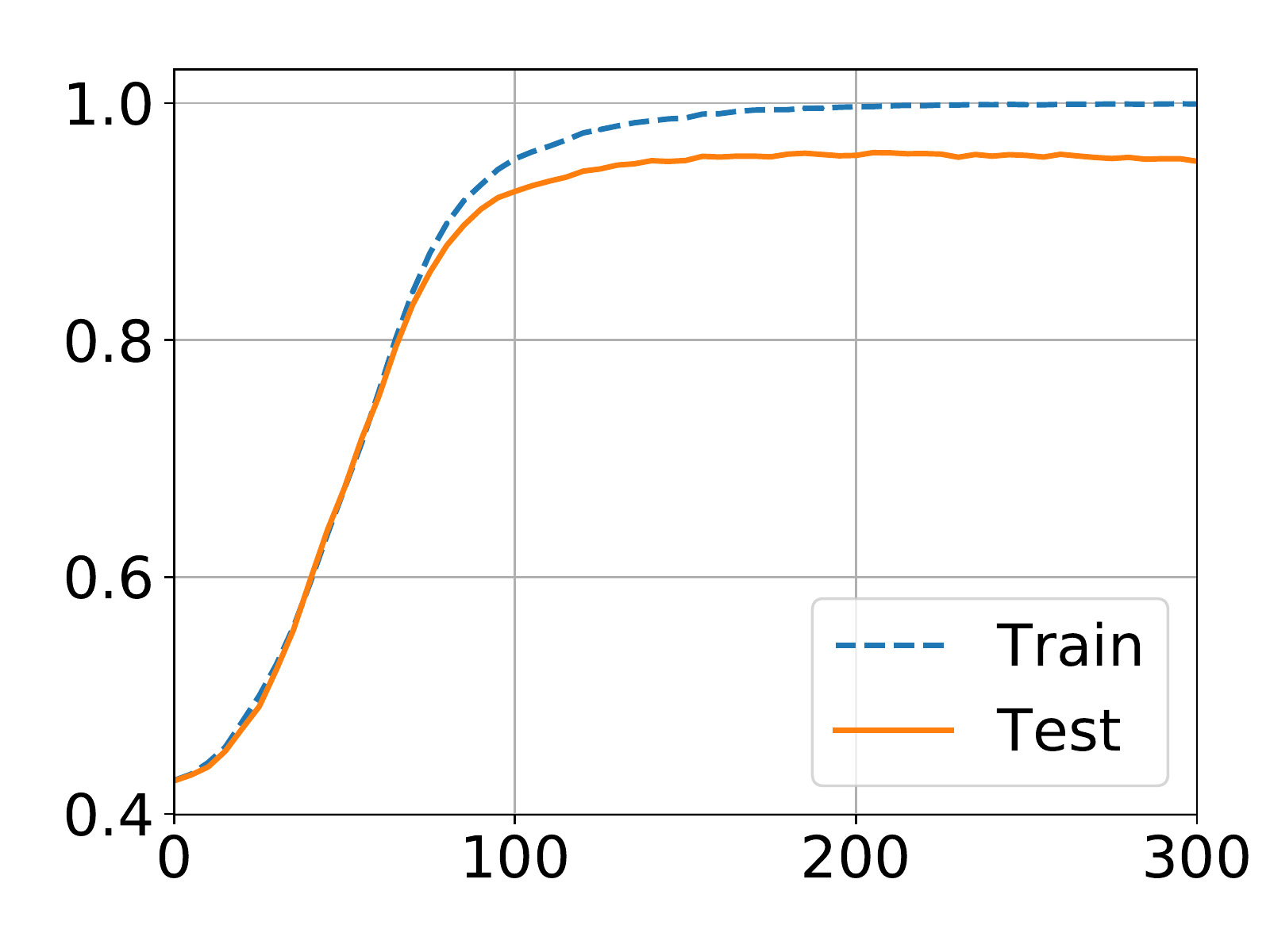}
    \vspace{-0.8cm}
    \caption{K=4}
  \end{subfigure}
  \centering
  \begin{subfigure}{.33\textwidth}
    \centering
    \includegraphics[width=1\linewidth]{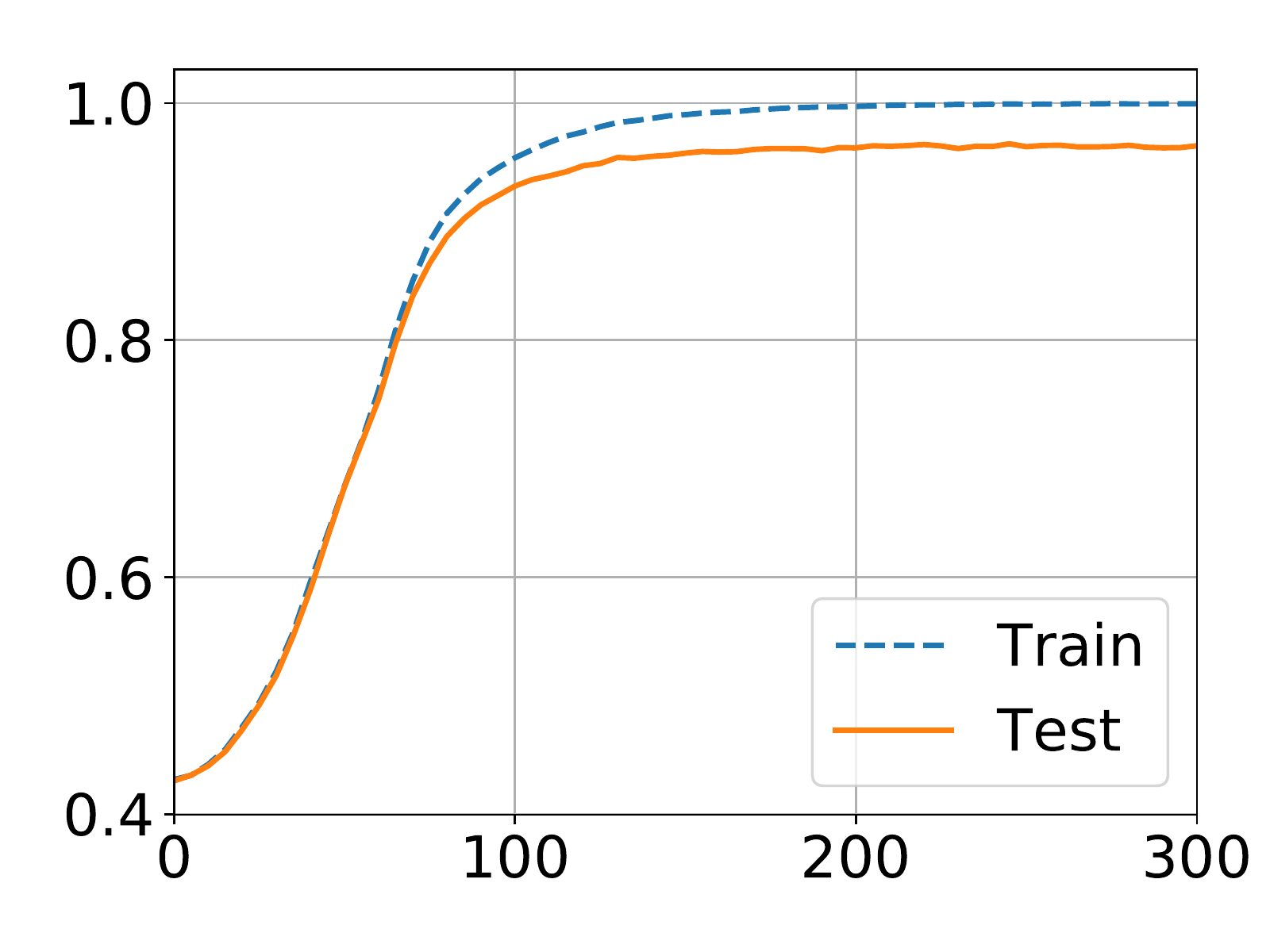}
    \vspace{-0.8cm}
    \caption{K=10}
  \end{subfigure}
  \centering
  \begin{subfigure}{.33\textwidth}
    \centering
    \includegraphics[width=1\linewidth]{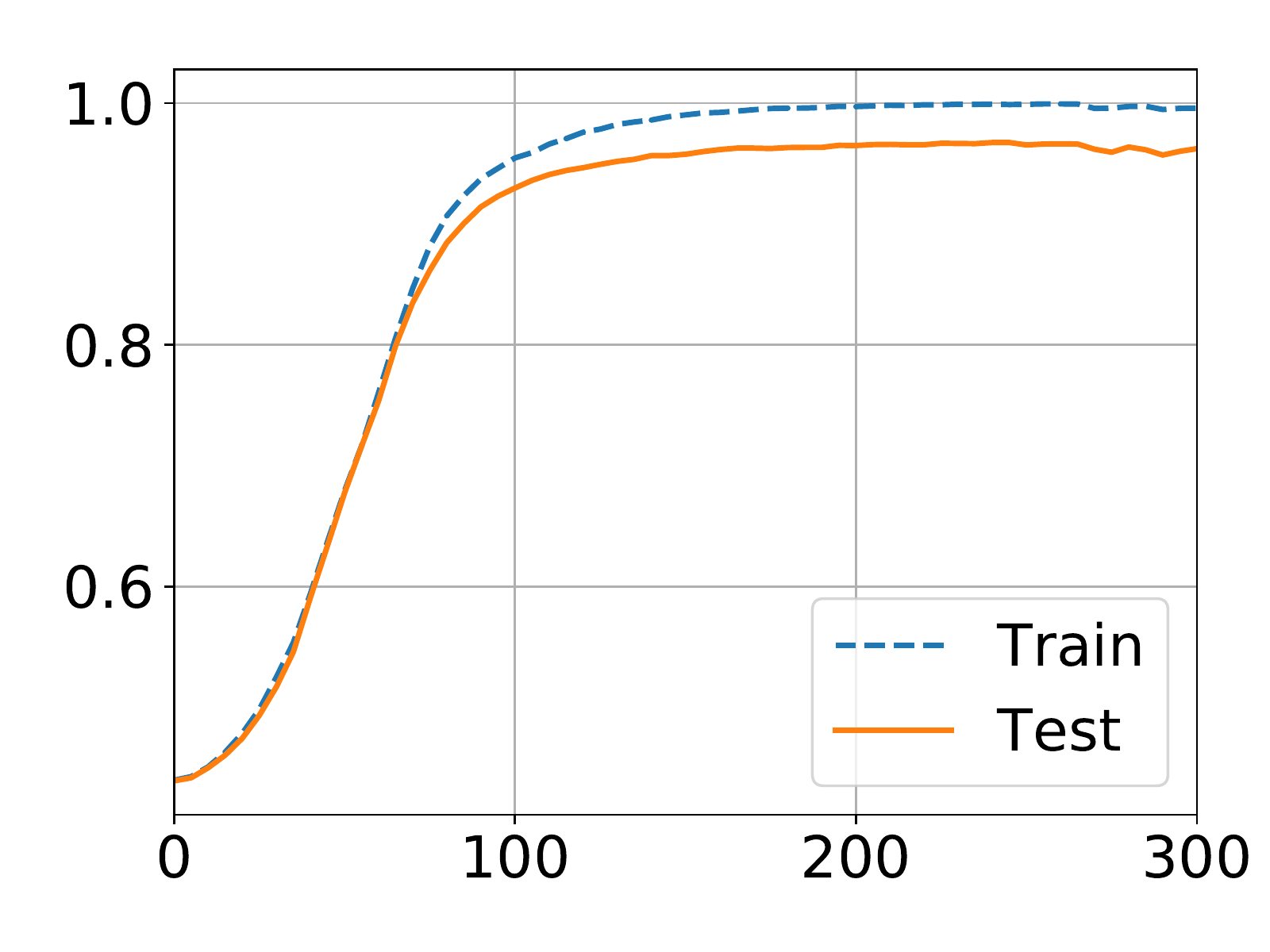}
    \vspace{-0.8cm}
    \caption{K=20}
  \end{subfigure}
  \vspace{-0.3cm}
  \caption{Train and Test Accuracy Curves versus Epochs on Dataset Letter.}
  \label{fig:class-letter}
  \vspace{0.2cm}
\end{figure*}

Next we show generated samples from the proposed models trained with MNIST and Fashion-MNIST in \autoref{fig-demo-samples}. In the figure, we show generated samples from GenMM and LatMM for MNIST and Fashion-MNIST datasets. We use different value of $K$ to generate images. It can be observed that LatMM is able to produce good quality image samples as GenMM. While we argue that LatMM has a lower level of complexity than GenMM, it is seen that LatMM works good in practice.    

In the second experiment, we explore power of invertibility for interpolation in the latent domain. We use samples from MNIST and Fashion-MNIST datasets for this `interpolation' experiment. In \autoref{fig-interpo}, we have six subfigures. For each subfigure, the first row and the last row are comprised of the real (true) data samples from MNIST and Fashion-MNIST dataset. In each column, we find latent codes corresponding to the real samples of the first row and the last row, $\bm{z}_1, \bm{z}_2$. This is possible as the neural networks are invertible. Then, we perform a convex combination of the two latent codes as $\alpha \bm{z}_1 + (1- \alpha)\bm{z}_2$, where $0 < \alpha <1$. The latent code produced by the convex combination is used to generate a new sample using the trained models. All other rows except the first and the last rows of the figure are the generated samples by varying $\alpha$. In \autoref{fig-interpo}, we observe the change visually from the first row to last row - how the first row slowly changes to the last row. We use GenMM for \autoref{fig-interpo-genmm1}, \autoref{fig-interpo-genmm2}, \autoref{fig-interpo-genmm3}, and LatMM for \autoref{fig-interpo-latmm1}, \autoref{fig-interpo-latmm2}, \autoref{fig-interpo-latmm3}. Interpolation experiment for LatMM is easier than GenMM. GenMM has a set of neural network generators $\{ \bm{g}_k(\bm{z}) \}_{k=1}^K$ and a fixed Gaussian distribution for latent variable $\bm{z}$. We compute $\gamma_k$ for a real image $\bm{x}$, and then find the latent code $\bm{z}$ of $\bm{x}$ using $\bm{g}_{k^{*}}^{-1}(\bm{x})=\bm{f}_{k^{*}}(\bm{x})$, where $k^{*} = \arg \max_{k} \gamma_k$. For two real images (one image is in the first row and the second image in the last row), we find the corresponding latent codes, compute their convex combination as interpolation, and then pass the computed latent code through a generator $\bm{g}_k(\bm{z})$ to produce a generated sample $\bm{x}$. Identity of the generator of GenMM is chosen as $k^{*}$ corresponding to the image of the first row if $\alpha < 0.5$, or to the image of the last row if $\alpha \geq 0.5$.

\begin{figure*}[!ht]
  \captionsetup[subfigure]{justification=centering}
  \centering
  \begin{subfigure}{.33\textwidth}
    \centering
    \includegraphics[width=1\linewidth]{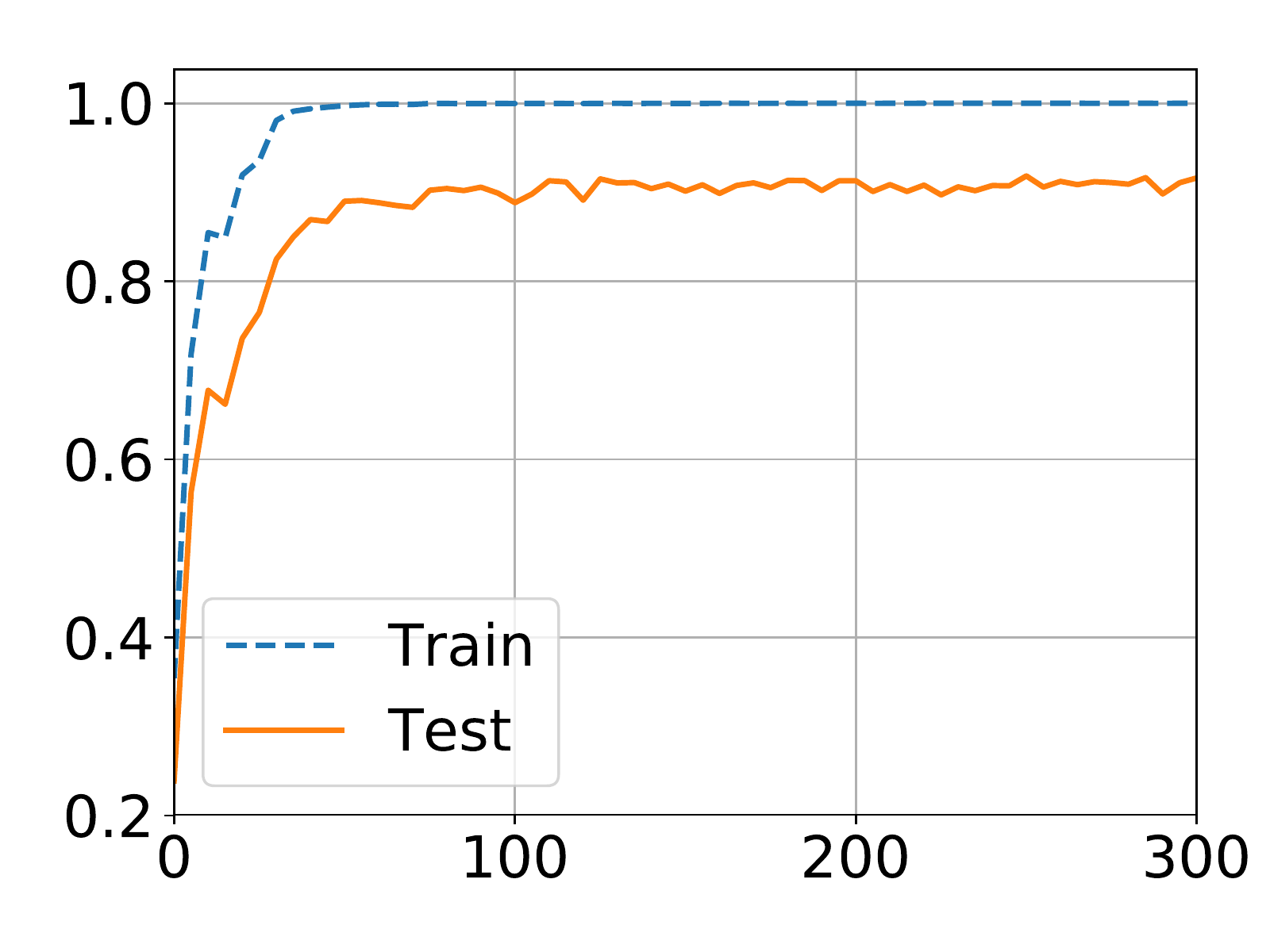}
    \vspace{-0.8cm}
    \caption{K=1}
  \end{subfigure}
  \vspace{-2pt}
  \begin{subfigure}{.33\textwidth}
    \centering
    \includegraphics[width=1\linewidth]{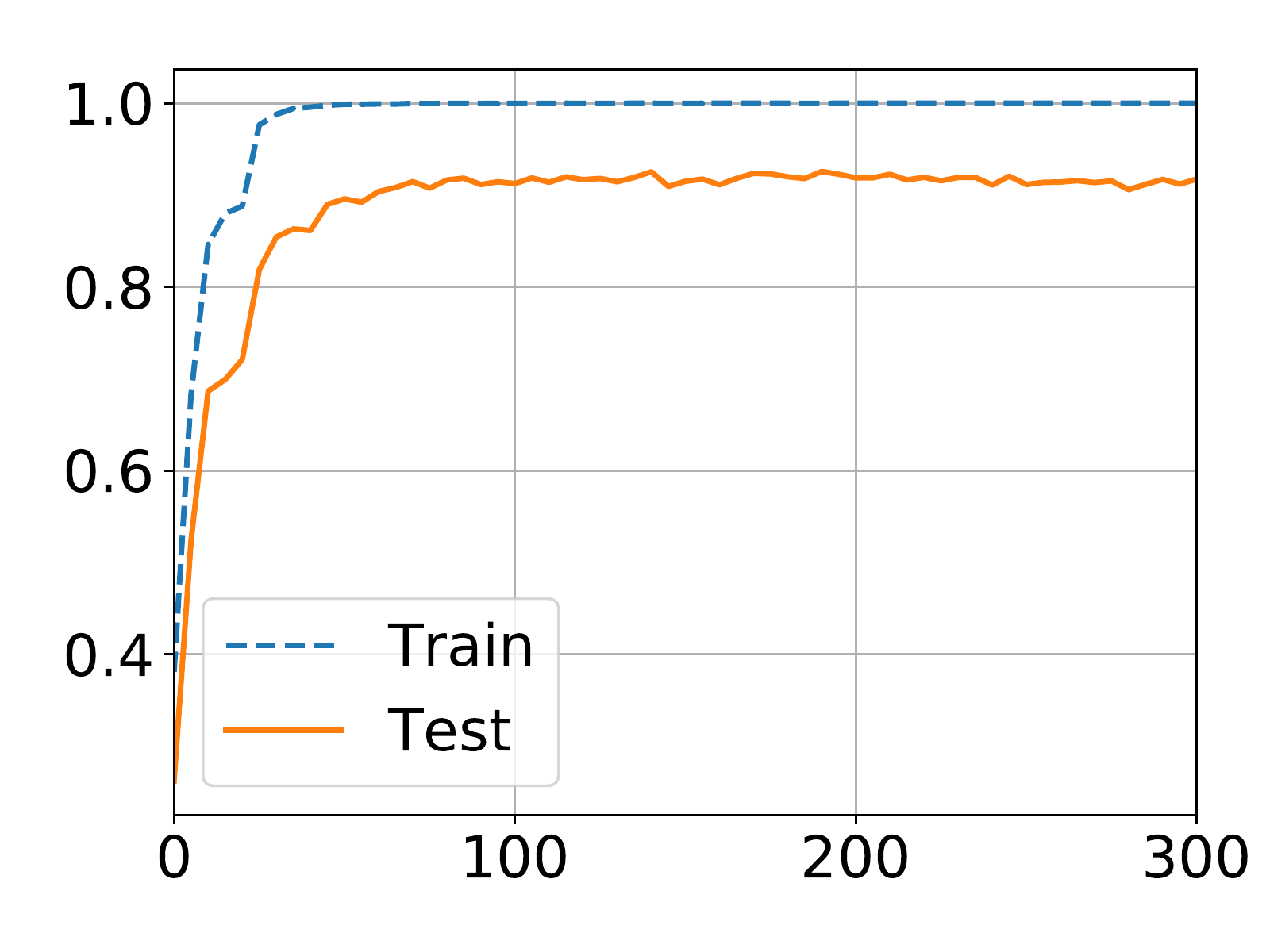}
    \vspace{-0.8cm}
    \caption{K=2}
  \end{subfigure}
  \centering
  \begin{subfigure}{.33\textwidth}
    \centering
    \includegraphics[width=1\linewidth]{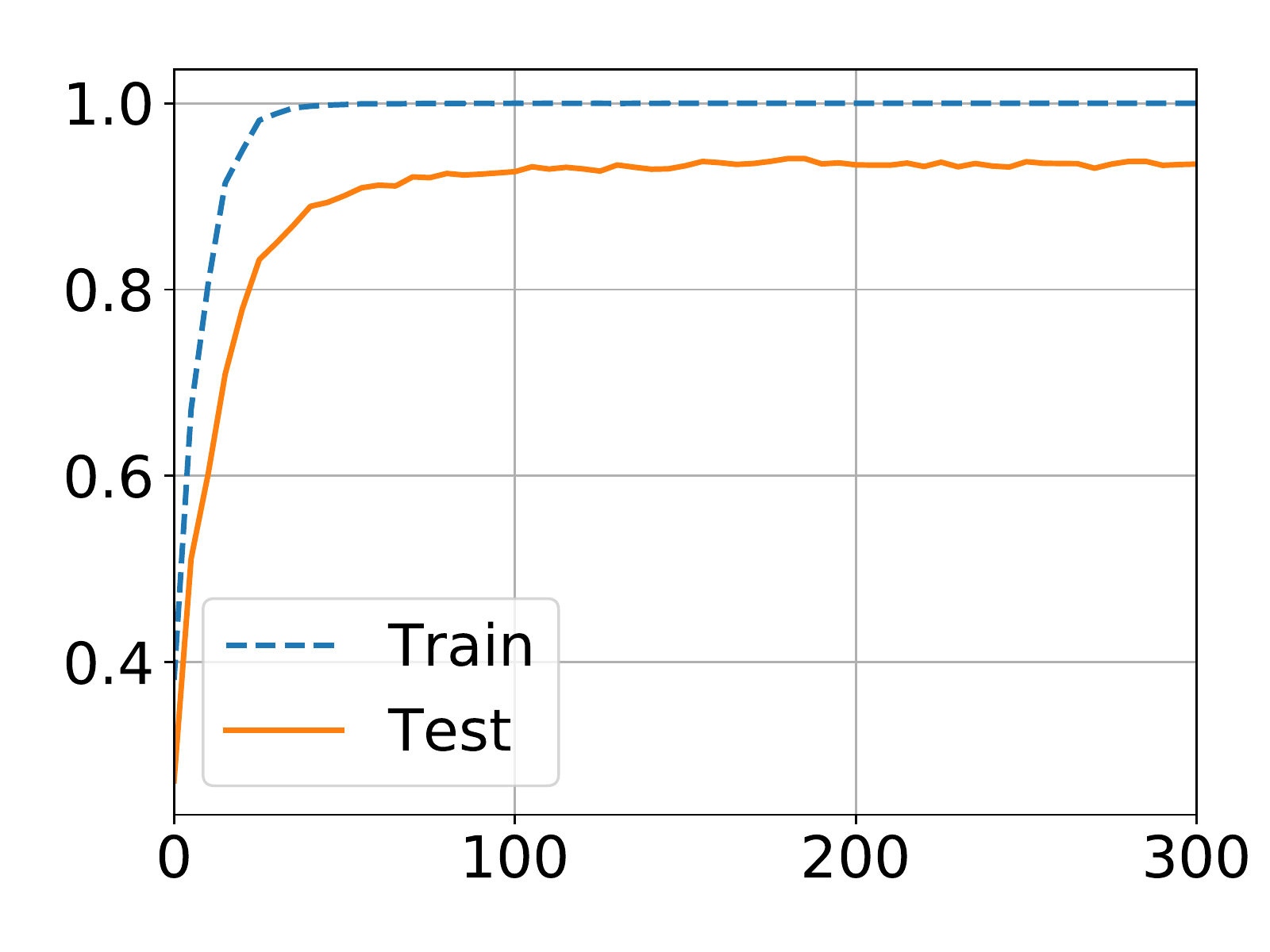}
    \vspace{-0.8cm}
    \caption{K=3}
  \end{subfigure}
  \centering
  \begin{subfigure}{0.33\textwidth}
    \centering
    \includegraphics[width=1\linewidth]{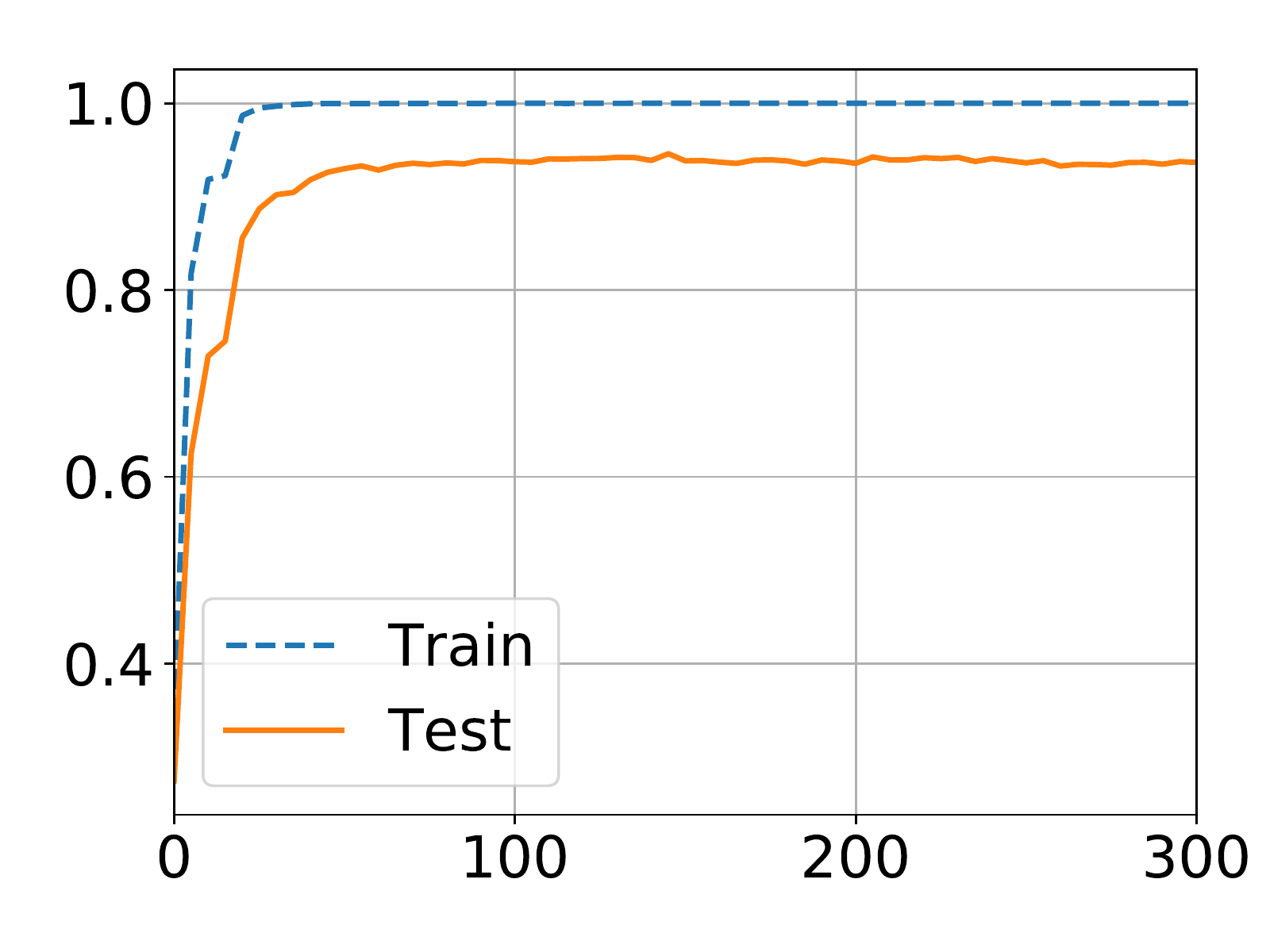}
    \vspace{-0.8cm}
    \caption{K=4}
  \end{subfigure}
  \centering
  \begin{subfigure}{.33\textwidth}
    \centering
    \includegraphics[width=1\linewidth]{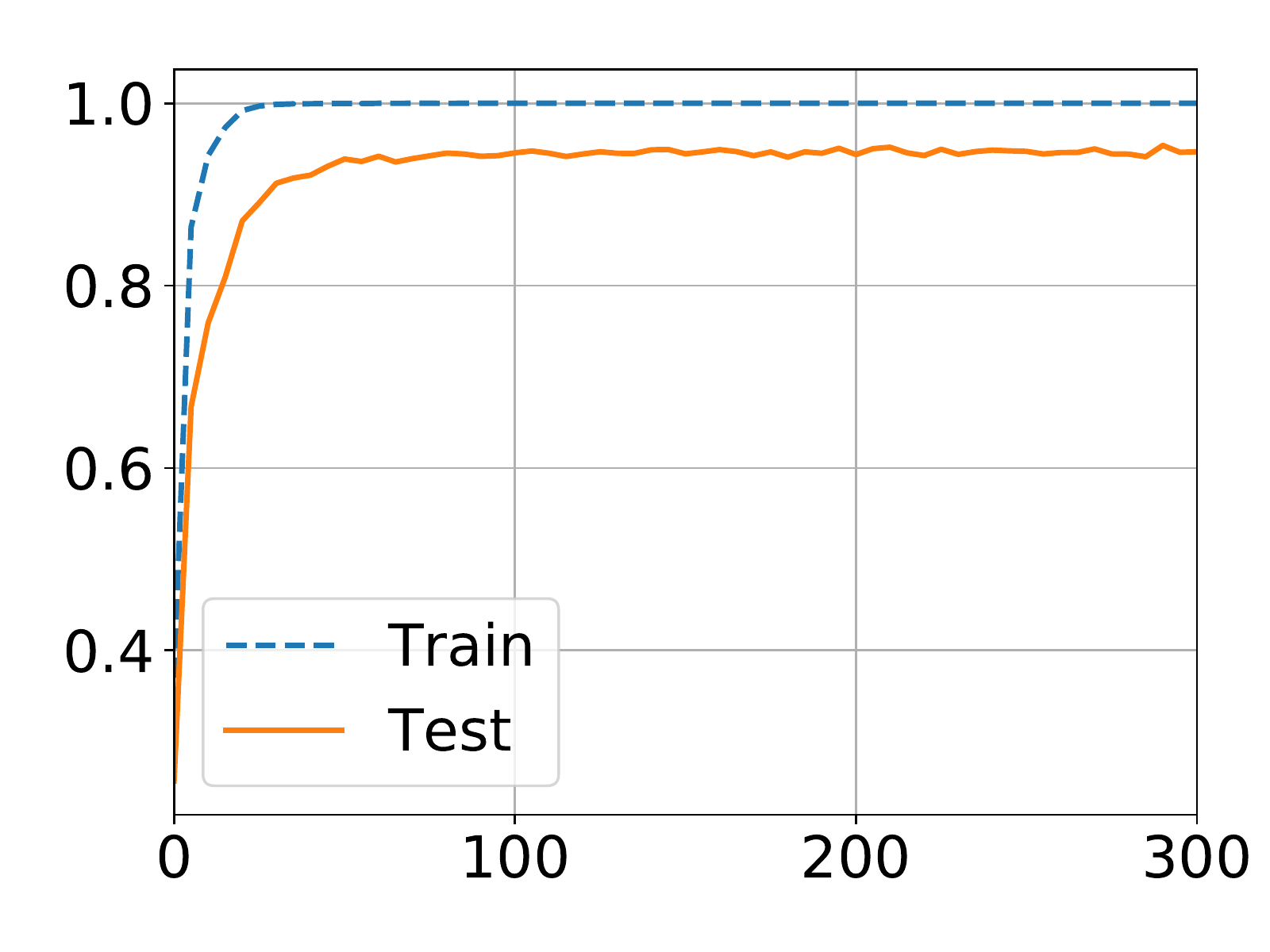}
    \vspace{-0.8cm}
    \caption{K=10}
  \end{subfigure}
  \centering
  \begin{subfigure}{.33\textwidth}
    \centering
    \includegraphics[width=1\linewidth]{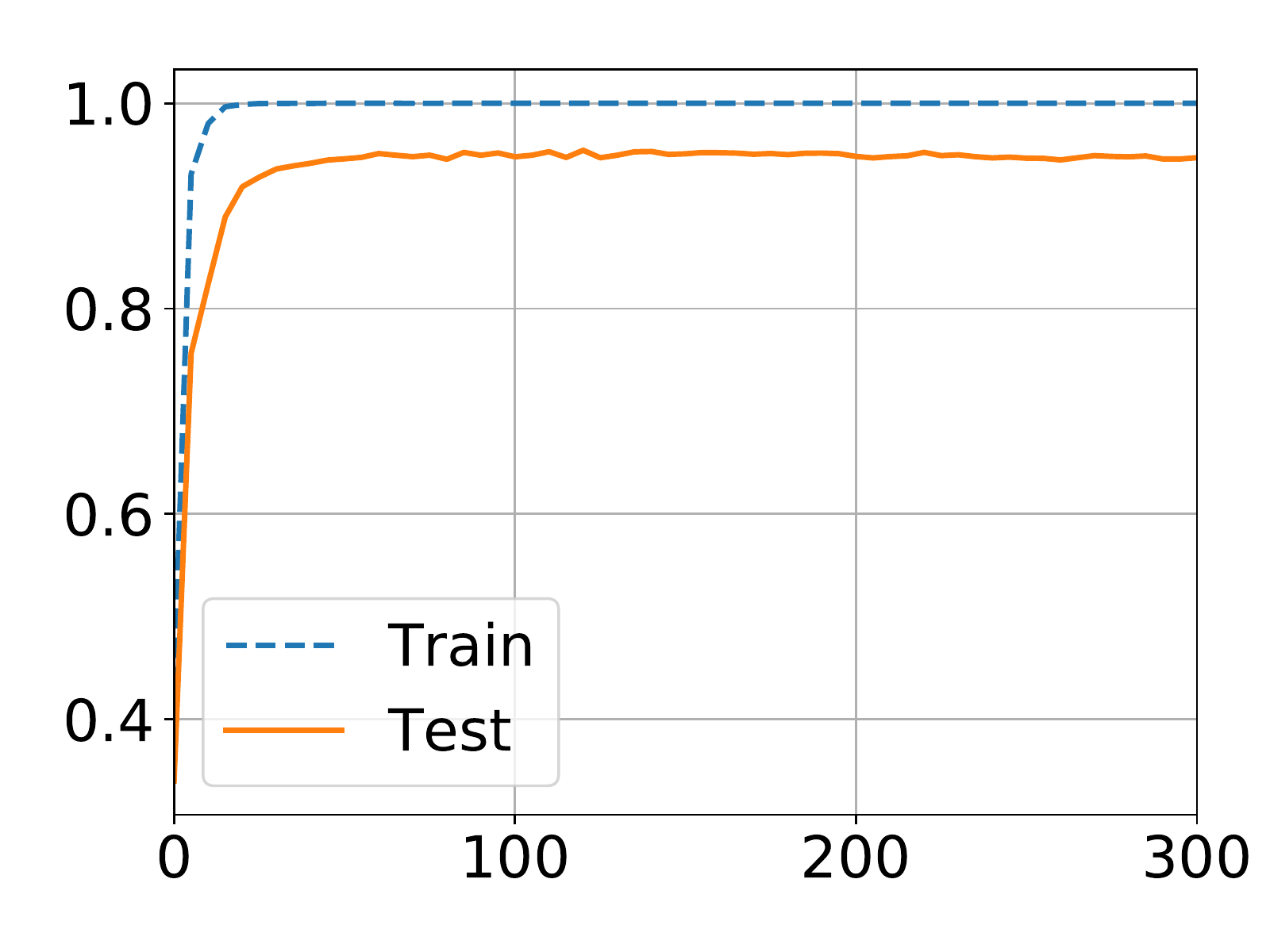}
    \vspace{-0.8cm}
    \caption{K=20}
  \end{subfigure}
  \vspace{-0.3cm}
  \caption{Train and Test Accuracy Curves versus Epochs on Dataset Norb}
  \label{fig:class-norb}
\end{figure*}

The second experiment on interpolation shows interesting result for
modeling multi-modal data. The distribution of ten digits together in
MNIST dataset is expected to be multi-modal. The aspect of multi-modal
distribution is addressed using the experimental result shown in
\autoref{fig-interpo-genmm2}. We use similar experimental steps
as that in \autoref{fig-interpo-genmm1} but with modifications. It
is evident that the generated digit images do not correspond well to
the real images of the first row and the last row. For example, in the
first column of \autoref{fig-interpo-genmm2}, we observe
presence of digits two and eight, while we expect that the
column should be comprised of only images of digit zero. Natural
question is why interpolation leads to generation of digits that are
unexpected. The answer lies in the procedure of performing our
experiment. The key difference for this experiment compared to the
experiment in \autoref{fig-interpo-genmm1} is that a sample is
produced by a randomly selected generator $\bm{g}_k(\bm{z})$ from $K$
possible choices. We compute interpolated latent code using the same
procedure as that in \autoref{fig-interpo-genmm1}, but use the generator where its identity $k$ is randomly sampled from the prior $\bm{\pi}$ directly. The generated images in this interpolation experiment reveals a clue that each generator models a subset of the whole training dataset. We can qualitatively argue that use of multiple generators helps for modeling the multi-modal distribution.  

\subsection{Application to Classification Task}
\begin{figure*}[!ht]
  \captionsetup[subfigure]{justification=centering}
  \centering
  \begin{subfigure}{.33\textwidth}
    \centering
    \includegraphics[width=1\linewidth]{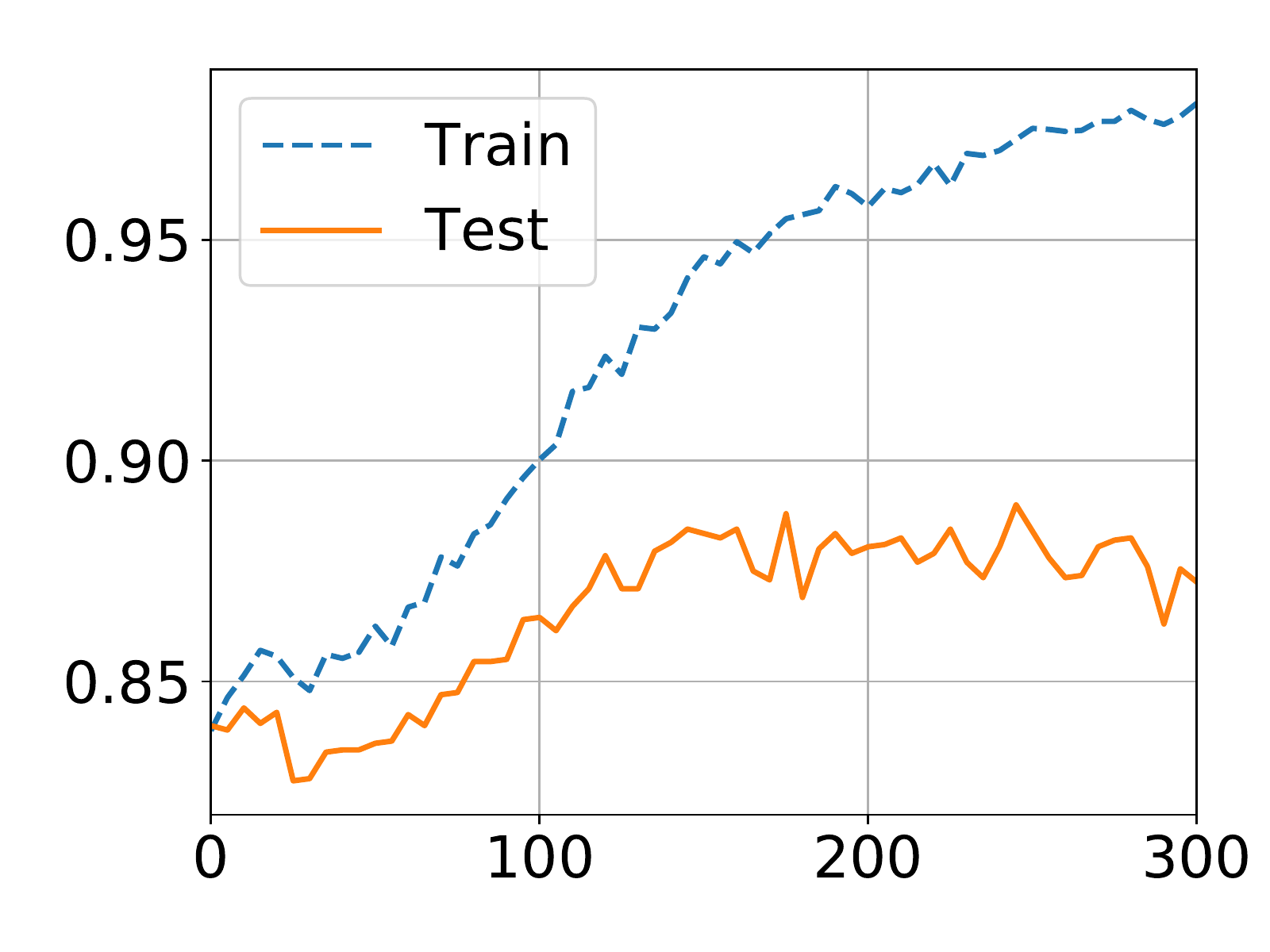}
    \vspace{-0.8cm}
    \caption{K=1}
  \end{subfigure}
  \vspace{-2pt}
  \begin{subfigure}{.33\textwidth}
    \centering
    \includegraphics[width=1\linewidth]{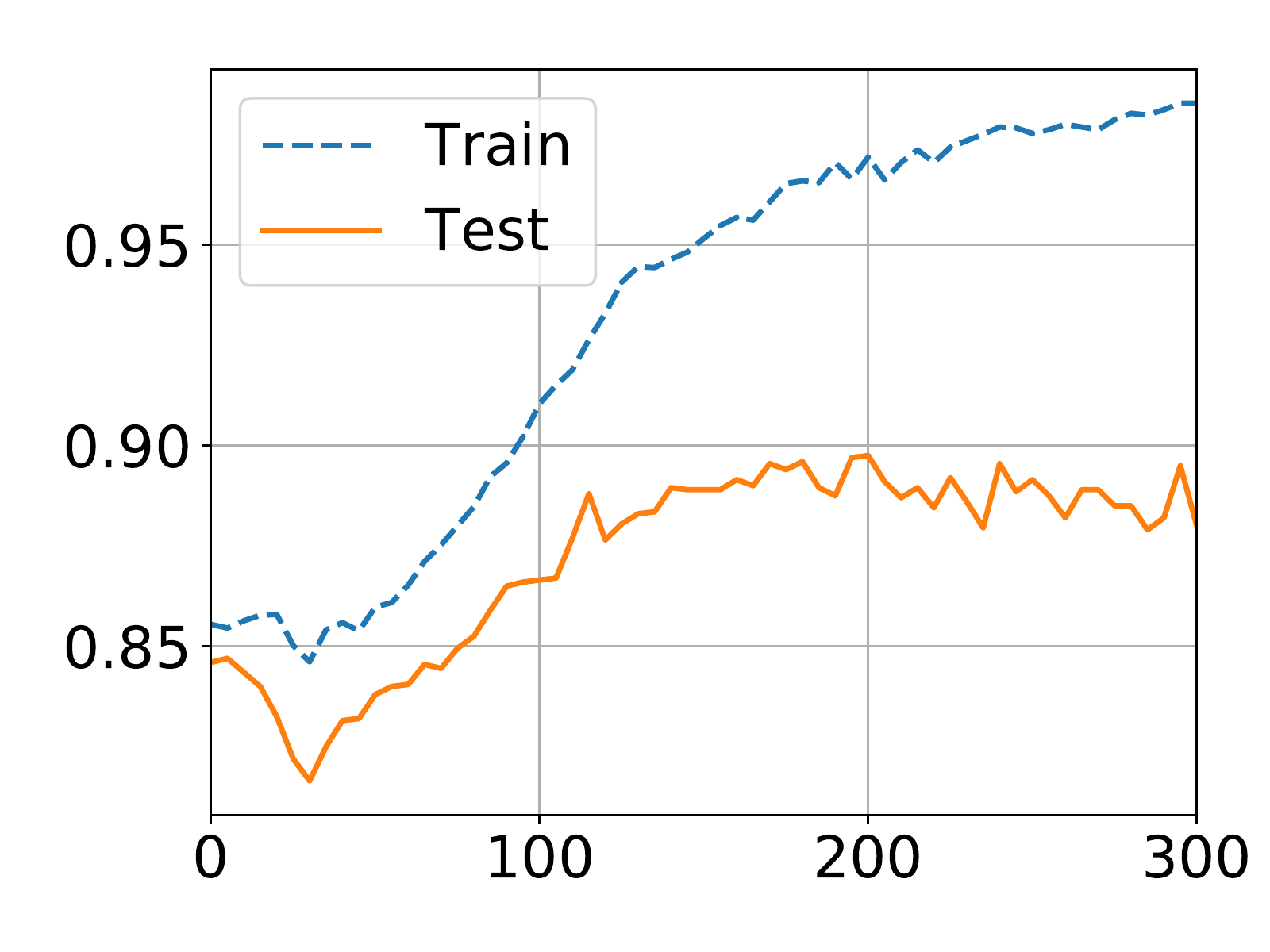}
    \vspace{-0.8cm}
    \caption{K=2}
  \end{subfigure}
  \centering
  \begin{subfigure}{.33\textwidth}
    \centering
    \includegraphics[width=1\linewidth]{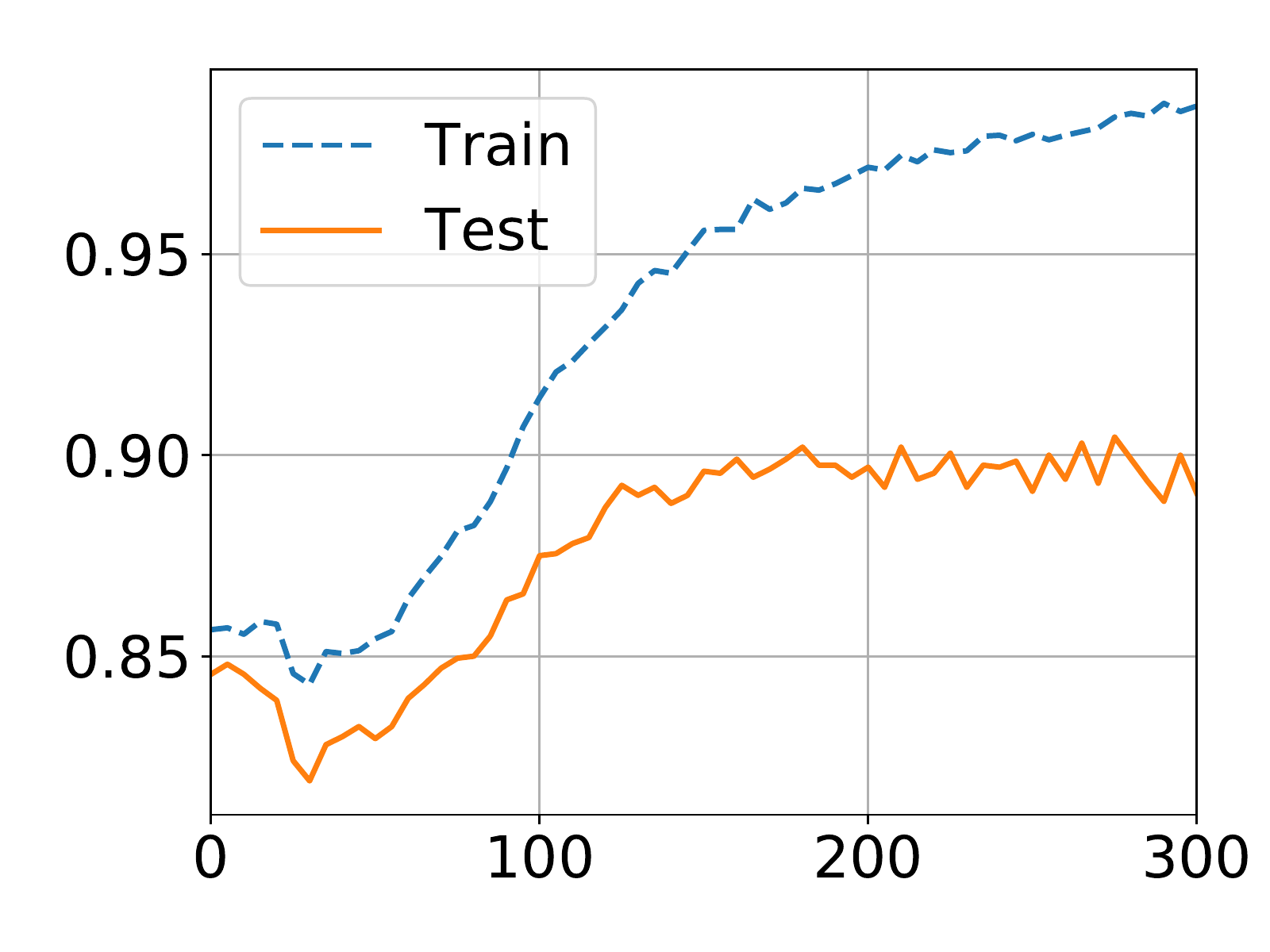}
    \vspace{-0.8cm}
    \caption{K=3}
  \end{subfigure}
  \centering
  \begin{subfigure}{0.33\textwidth}
    \centering
    \includegraphics[width=1\linewidth]{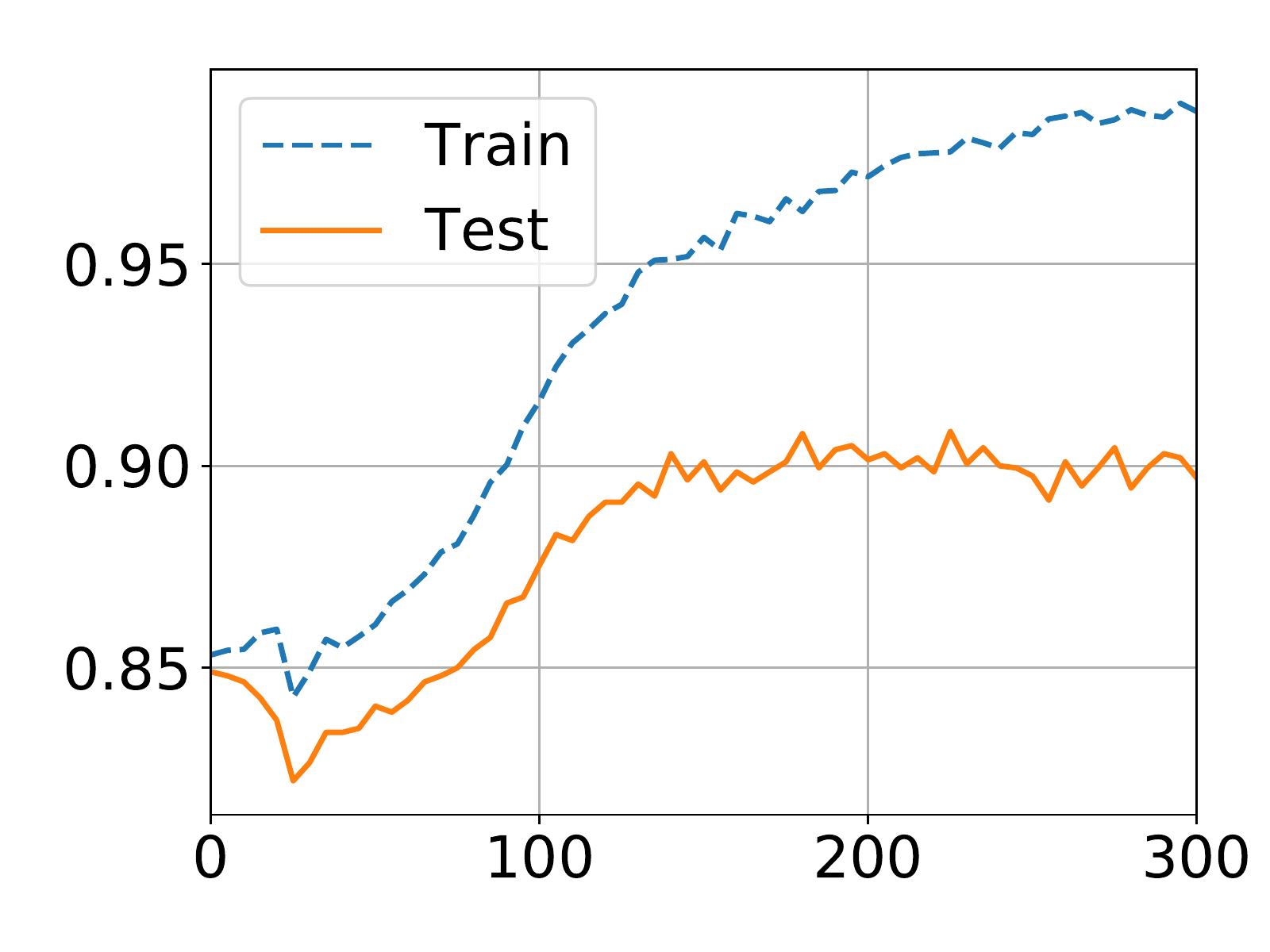}
    \vspace{-0.8cm}
    \caption{K=4}
  \end{subfigure}
  \centering
  \begin{subfigure}{.33\textwidth}
    \centering
    \includegraphics[width=1\linewidth]{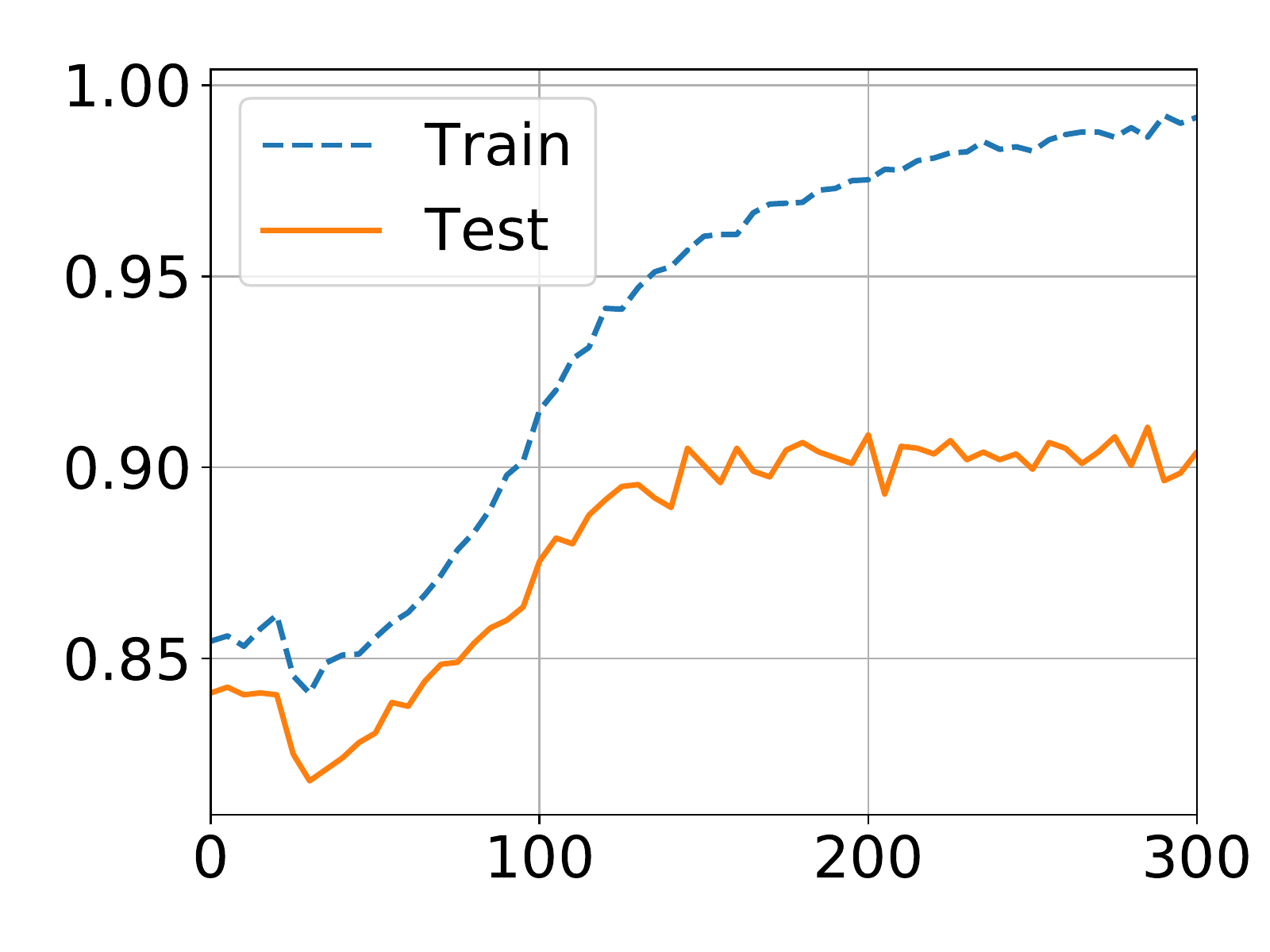}
    \vspace{-0.8cm}
    \caption{K=10}
  \end{subfigure}
  \centering
  \begin{subfigure}{.33\textwidth}
    \centering
    \includegraphics[width=1\linewidth]{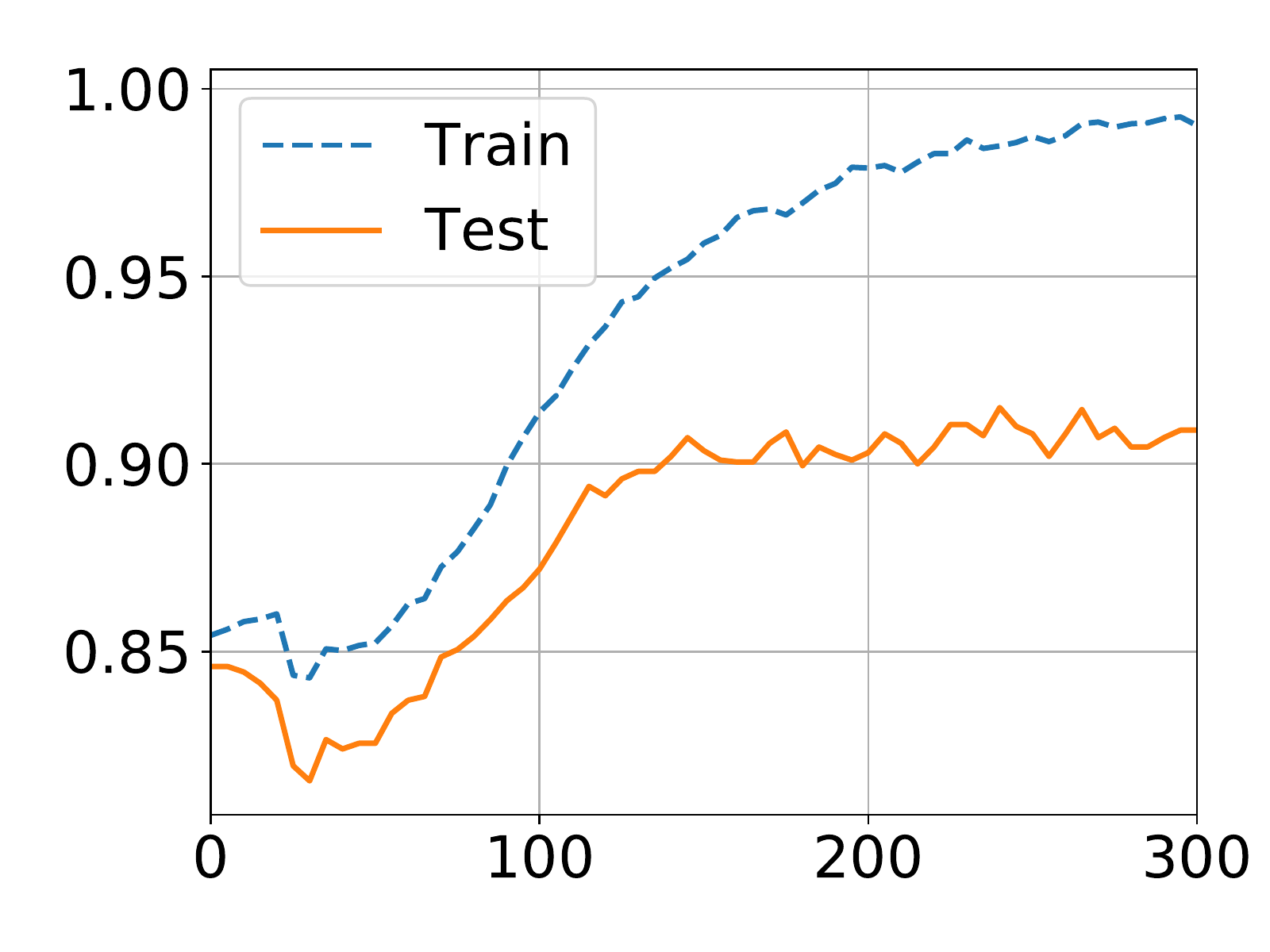}
    \vspace{-0.8cm}
    \caption{K=20}
  \end{subfigure}
  \vspace{-0.3cm}
  \caption{Train and Test Accuracy Curves versus Epochs on Dataset Satimage.}
  \vspace{0.3cm}
  \label{fig:class-satimage}
\end{figure*}
In this subsection, we apply our proposed mixture models to classification tasks using the maximum likelihood criterion. We compare classification performance with the state-of-art results. The state-of-art results are produced by discriminative learning approaches. The major advantage of maximum likelihood based classification is that any new class can be accommodated on-the-fly. On the contrary a discriminative learning approach requires retraining whenever new classes appear. 

For a given dataset with $Y$ classes, we divide the dataset by sample labels and each subset has the same label $y$. Then we train one GenMM model per class of data, i.e. $p(\bm{x};\bm{\Phi}_{y})$ is trained with the $y$-th class's data. After we have all $p(\bm{x};\bm{\Phi}_y)$, $\forall y = 1, 2, \cdots, Y$ trained, a new sample $\bm{x}$ is predicted by $\argmax_{y} p(\bm{x};\bm{\Phi}_y)$.

The maximum likelihood based classification experiment as described above is carried out in three different datasets: Letter, Satimage, and Norb. For each dataset, we train our models for $300$ epoches on the training data of the corresponding dataset, and the test accuracy is reported in \autoref{tab:acc-classification}. The state-of-art accuracy of each dataset in literature is also listed in this table for comparison. For each dataset, we increase the total number of mixture components $K$ and the neural network generators have the same structure. The table shows that the classification accuracy on each dataset is increased as we increase the number of generators in GenMM. When $K$ is $10$ or $20$, maximum likelihood based classification by GenMM outperforms the state-of-art accuracy. The state-of-art accuracy results are obtained by using discriminative learning approaches. For dataset Norb, more significant performance gain is observed. Our classification accuracy is boosted from $0.9184$ to $0.9542$ when $K$ is increased from $1$ to $20$ and a large improvement margin is obtained over reference accuracy. We also test LatMM on classification task, but its accuracy is more or less around the accuracy of GenMM with $K=1$. Note that LatMM is a relatively low-complexity model than GenMM. 

\autoref{fig:class-letter} \autoref{fig:class-satimage} and \autoref{fig:class-norb} show the train and test accuracy changing along with the training epoch on dataset Letter and Satimage, respectively. For each dataset, the accuracy curves versus epoch trained with GenMM at different value of $K$ are shown. In these sets of figures, all accuracy curves climbs and flattens around some value, as training epoch increases. Train accuracy is either coincident with, or above test accuracy curve at different training phases. For each set of figures on a given dataset, the gap between train and test curve is smaller as a larger number of mixture components is used. As $K$ increases, test curve flattens at a larger accuracy value. This again speaks for validation of our proposed models and also the advantage of using our mixture models for practical tasks.

\section{Conclusion}
We conclude that the principal of expectation maximization can be used for neural network based probability distribution modeling. Our approach leads to explicit distribution modeling and the experimental results show an important aspect that the normal statistical behaviour of modeling performance versus model complexity remains valid. The proposed models are able to generate images which have good visual quality. This is also supported by several metric scores. Practical applications of our models for classification tasks are also carried out. The results confirm that our approach is good for modeling multi-modal distributions.
Further extensions using variational inference for learning
parameters of mixture models will be studied in the future.

\section{Acknowledgments}
The computations were enabled by resources provided by the Swedish National Infrastructure for Computing (SNIC) at HPC2N partially funded by the Swedish Research Council through grant agreement no. 2016-07213.

\bibliographystyle{IEEEtran}

\bibliography{myref}

\end{document}